\documentclass[twoside]{article}

\usepackage[accepted]{aistats2023}
\usepackage[round]{natbib}

\renewcommand{\epsilon}{\varepsilon}
\newcommand{\card}[1]{\left\lvert#1\right\rvert}        % cardinality of a set
\newcommand{\defeq}{\vcentcolon =}                      % define equals
\newcommand{\indic}[1]{\mathbf{1}_{#1}}                 % indicator function
\newcommand{\Reals}{\mathbb{R}}                         % real numbers
\DeclareMathOperator*{\Minimize}{\mathrm{Minimize}}     % minimization problem
\newcommand{\binomial}[2]{\mathrm{Bin}(#1,#2)}          % binomial random variable
\DeclareMathOperator*{\argmax}{arg\max}
\DeclareMathOperator*{\argmin}{arg\min}
\newcommand{\abs}[1]{\left\lvert#1\right\rvert}         % absolute value
\newcommand{\exps}[1]{\mathrm{e}^{#1}}                  % small exp 
\newcommand{\floor}[1]{\left\lfloor #1 \right\rfloor}   % floor function
\newcommand{\gaussian}[2]{\mathcal{N}(#1,#2)}           % Gaussian distribution
\newcommand{\proba}[1]{\mathbb{P}\left (#1\right )}             % probability of an event
\newcommand{\probaunder}[2]{\mathbb{P}_{#2}\left (#1\right )}   % probability of an event
\newcommand{\condproba}[2]{\mathbb{P}\left(#1\middle|#2\right)} % conditional probability
\newcommand{\expec}[1]{\mathbb{E}\left[#1\right]}               % expected value
\newcommand{\var}[1]{\mathrm{Var}\left(#1\right)}               % variance of a random variable
\newcommand{\bigo}[1]{\mathcal{O}\left(#1\right)}               % big O notation
\newcommand{\littleo}[1]{o\left(#1\right)}                      % little o of something
\newcommand{\norm}[1]{\left\lVert #1\right\rVert}               % norm of something
\newcommand{\expl}[1]{\mathrm{exp}\left(#1\right)}              % exp of something

\newcommand{\doc}{z}                                            % generic document
\newcommand{\word}{w}                                           % generic word
\newcommand{\dictionary}{\mathcal{D}}                           % global dictionary
\newcommand{\localdic}{\mathcal{D}}                             % local dictionary
\newcommand{\mult}{m}                                           % multiplicity
\newcommand{\Mult}{M}                                           % random multiplicity
\newcommand{\corpus}{\mathcal{C}}                               % corpus: the set training texts for tf-idf 
\newcommand{\texts}{\mathcal{T}}                                % the set of all texts
\newcommand{\Vectorizer}{\varphi}                               % generic vectorizer
\newcommand{\vectorizer}[1]{\Vectorizer (#1)}                   % vectorizer applied to smthg
\newcommand{\length}[1]{\left\lvert#1\right\rvert}              % length of an anchor
\newcommand{\Anchors}{\mathcal{A}}                              % set of candidates anchors
\newcommand{\Precision}{\mathrm{Prec}}                          % symbol for precision
\newcommand{\precision}[1]{\Precision(#1)}                      % precision of an anchor
\newcommand{\Empprec}{\widehat{\mathrm{Prec}}}                  % empirical precision
\newcommand{\expecunder}[2]{\mathbb{E}_{#2}\left[#1\right]}     % expectation under some distribution
\newcommand{\anchor}{a}                                         % some anchor index
\newcommand{\Astar}{A^\star}                                    % best anchor
\newcommand{\idf}{v}                                            % inverse document frequency of a word
\newcommand{\Tfidf}{\varphi}
\newcommand{\tfidf}[1]{\Tfidf(#1)}                              % non-normalized tf-idf 
\newcommand{\Approxprec}{L}                                     % approx precision for logistic model
\newcommand{\approxprec}[1]{\Approxprec\left(#1\right)}         % approx precision for logistic model
\newcommand{\Phibar}{\overline{\Phi}}                           % 1- the cdf of a standard Gaussian
\newcommand{\phibar}[1]{\Phibar\left(#1\right)}                 % 1- the cdf of a standard Gaussian
\newcommand{\cberryesseen}{C}                                   % constant in the Berry-Esseen theorem
\newcommand{\Ztilde}{\tilde{Z}}                                 % modified version 
\newcommand{\vmax}{\idf_{\text{max}}}                           % max idf weight
\newcommand{\vmin}{\idf_{\text{min}}}                           % min idf weight
\newcommand{\maxmult}{M}                                        % maximum multiplicity
\newcommand{\minmult}{m}                                        % minimum multiplicity
\newcommand{\denom}{D}                                          % denominator for the normalized TF-IDF
\newcommand{\num}{N}                                            % numerator for the normalized tf-idf
\newcommand{\normtfidf}[1]{\phi(#1)}% normalized tfidf

\usepackage{url}            % url
\usepackage{float}          % float environment, for figures in minipage
\usepackage{graphicx}       % figures
\usepackage{tikz}
\usepackage{wrapfig}

\usepackage{amssymb}        % math fonts
\usepackage{mathtools}      % \vcentcolomn
\usepackage{amsthm}         % \qed
\usepackage{stmaryrd}       % bracket

\usepackage{colortbl}       % table
\usepackage{xcolor}
\definecolor{anchors_color}{HTML}{2B7BBA}
\usepackage{array,multirow,graphicx}        % vertical text

\usepackage{dirtytalk}      % say

\usepackage[pdfencoding=auto]{hyperref}     % texorpdfstring

\usepackage{algorithm}
\usepackage{algpseudocode}

 \usepackage[export]{adjustbox}     
\begin{document}

%%%%%%%%%%%%%%%%%%%%%%%%%%%%%%%%%%%%%%%%%%%%%%%%%%%%%%%%%%%%%%%%%%%%%%%%%%%%%%%%%%%
% Bold font in theorem headings
\makeatletter
\def\th@plain{%
  \thm@notefont{}% same as heading font
  \itshape % body font
}
\def\th@definition{%
  \thm@notefont{}% same as heading font
  \normalfont % body font
}
\makeatother

\newtheorem{theorem}{Theorem}
\newtheorem{proposition}{Proposition}

\newtheorem{definition}{Definition}
\newtheorem{claim}{Claim}
\newtheorem{lemma}{Lemma}
\newtheorem{remark}{Remark}

%%%%%%%%%%%%%%%%%%%%%%%%%%%%%%%%%%%%%%%%%%%%%%%%%%%%%%%%%%%%%%%%%%%%%%%%%%%%%%%%%%%



\twocolumn[

\aistatstitle{A Sea of Words: An In-Depth Analysis of Anchors for Text Data}

\aistatsauthor{Gianluigi Lopardo \And Frédéric Precioso \And Damien Garreau}

\aistatsaddress{Université Côte d'Azur, Inria, \\ CNRS, LJAD, France \And 
                Université Côte d'Azur, Inria, \\ CNRS, I3S, France \And 
                Université Côte d'Azur, Inria, \\ CNRS, LJAD, France} ]  

\begin{abstract}
Anchors \citep{ribeiro2018anchors} is a post-hoc, rule-based interpretability method. For text data, it proposes to explain a decision by highlighting a small set of words (an anchor) such that the model to explain has similar outputs when they are present in a document. 
In this paper, we present the first theoretical analysis of Anchors, considering that the search for the best anchor is exhaustive.
After formalizing the algorithm for text classification, we present explicit results on different classes of models when the vectorization step is TF-IDF, and words are replaced by a fixed out-of-dictionary token when removed. 
Our inquiry covers models such as elementary if-then rules and linear classifiers. 
We then leverage this analysis to gain insights on the behavior of Anchors for any differentiable classifiers. 
For neural networks, we empirically show that the words corresponding to the highest partial derivatives of the model with respect to the input, reweighted by the inverse document frequencies, are selected by Anchors. 
\end{abstract}

%%%%%%%%%%%%%%%%%%%%%%%%%%%%%%%%%%%%%%%%%%%%%%%%%%%%%%%%%%%%%%%%%%%%%%%%%%%%%%%%%%%%%%%%%%%%%%%%%%%%%%%%%%%%%%%%%%%%%%%%%%%%%%%

\section{INTRODUCTION}
\label{sec:introduction}
As with other areas of machine learning, the state-of-the-art in natural language processing consists of very complex models based on hundreds of millions or even billions of parameters \citep{devlin2019bert,brown_et_al_2020}. 
This complexity is a huge limitation to the use of machine learning algorithms in critical or sensitive contexts: users and domain experts are reluctant to adopt decisions they cannot understand. 
In recent years, several methods have been proposed to meet the growing demand for \emph{interpretability}. 
In particular, \emph{local} \emph{model-agnostic} techniques have emerged to explain the individual predictions of any classifier for a specific instance. 
The unique assumption is that the model can be queried as often as necessary. 
However, the process generating the explanations can be, for a user, as mysterious as the prediction to be explained. 
Moreover, interpretability methods often lack solid theoretical guarantees. 
In particular, their behavior on simple, already interpretable models is often unknown. 
Instead of helping, applying a poorly understood explainer to a complex model can lead to misleading results. 

\begin{figure}[t]
\centering
\begin{minipage}{0.21\textwidth}\centering
    \includegraphics[scale=0.8]{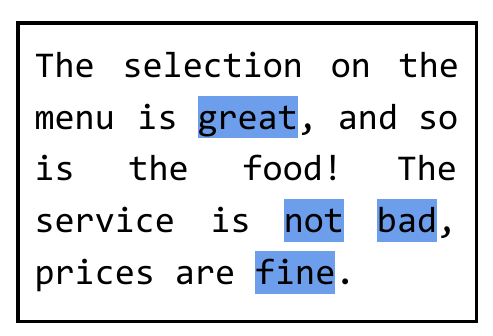}
\end{minipage}
%\hfill
\begin{minipage}{0.21\textwidth}\centering
    $\texttt{precision}: \; 0.97$ \\
    $\texttt{coverage}: \; 0.12$
\end{minipage}
    \caption{\label{fig:example}Anchors explaining the positive prediction of a black-box model~$f$ on an example $\xi$ from the Restaurant review dataset. The anchor $A = \{\textit{great, not, bad, fine}\}$ (in blue), having length $\length{A} = 4$ is selected. Intuitively,  these four words together ensure a positive prediction by $f$ with high probability ($\texttt{precision}: \; 0.97$), while being not too uncommon ($\texttt{coverage}: \; 0.12$).}
\end{figure}
In this work, we focus on Anchors \citep{ribeiro2018anchors}, an increasingly popular, local model-agnostic method more and more included in explainability toolboxes such as Alibi\footnote{\url{https://github.com/SeldonIO/alibi}}, and more precisely its implementation for text data. 
For a given prediction, Anchors' general idea is to provide a simple rule yielding the same prediction with high probability if it is satisfied. 
These rules can be formulated as the presence of a list of words in the document to be explained, and are presented as such to the user (see Figure~\ref{fig:example}). 
If the model to explain is intrinsically interpretable, in particular if we know with absolute certainty which words are important for the prediction, will Anchors highlight these words in the explanation? 

% After explaining how Anchors for text data works in Section~\ref{sec:anchors}, we introduce \emph{exhaustive Anchors} in Section~\ref{sec:exhaustive-anchors}. 
% We present our theoretical analysis on different classifiers, namely if-then rules and linear models in Section~\ref{sec:analysis}, and our empirical results for neural networks in Section~\ref{sec:neural-nets}.
% Finally, we draw our conclusions in Section~\ref{sec:conclusion}. 

 First, in Section~\ref{sec:anchors}, we explain the basic concepts of Anchors and formalize its mechanism for text classification. 
 We then delve into the definition of a more tractable, exhaustive version of the algorithm in Section~\ref{sec:exhaustive-anchors}, which constitutes the central object of our study. 
 Next, in order to understand the efficacy of Anchors for text data, we perform a theoretical and empirical analysis of its behavior on explainable classifiers in Section~\ref{sec:analysis}. 
This allows us to gain new insights that can be extended to broader classes of models. 
In particular, in Section~\ref{sec:neural-nets}, we empirically show a surprising result on neural networks. 
This section provides valuable results into the real-world applications of Anchors that can be applied to explain document classifiers. 
Finally, in Section~\ref{sec:conclusion}, we draw our conclusions and summarize the findings of our study. 
Some unexpected results from our research highlight the importance of theoretical analysis for explainers. 
We believe that the insights presented in this article will be useful for researchers and practitioners in the field of natural language processing to accurately interpret the explanations provided by Anchors. 
On the other hand, the framework we have designed for this analysis can be of great use to the explainability community, both in designing new methods with solid theoretical foundations and in analyzing existing ones.

\paragraph{Contributions.} 
In this paper, we present the first theoretical analysis of Anchors for text data, based on the default implementation available on Github\footnote{\url{https://github.com/marcotcr/anchor}} (as of February 2023). 
The main restrictions of our analysis are the simplification of the combinatorial optimization procedure (therefore considering an \emph{exhaustive} version of Anchors), the use of an out of dictionary token when removing words, and the assumption that a TF-IDF vectorization is used as a preprocessing step. 
Specifically, 
\begin{itemize}
\itemsep0em 
\item we dissect Anchors' algorithm for text classification, showing that the sampling procedure can be described simply as an i.i.d. Bernoulli's removal of words not belonging to the anchor (Proposition~\ref{proposition:equivalent-sampling}); 
\item we show that the exhaustive version is stable with respect to perturbation of the precision function, justifying our study of the exhaustive Anchors algorithm (Proposition~\ref{prop:exhaustive-anchors-stability} and \ref{prop:empirical-precision-concentration});
\item if the classifier ignores some words, they will not appear in the anchor selected by the exhaustive Anchors (Proposition~\ref{prop:dummy-features});
\item exhaustive Anchors for simple if-then rules provably outputs meaningful explanations, though words can be ignored from the explanation if their multiplicity is too high (Proposition~\ref{prop:product});
\item exhaustive Anchors picks the words associated to the most positive coefficients reweighted by the inverse document frequency for all linear classifiers (Proposition~\ref{prop:precision-logistic} and~\ref{prop:approx-prec-maximization}):
\item we empirically show that exhaustive Anchors picks the words associated to the most positive partial derivatives scaled by the inverse document frequency for neural networks (Section~\ref{sec:neural-nets}). 
\end{itemize}
All our theoretical claims are supported by mathematical proofs, available in Section \ref{sec:all-proofs} of the Appendix, and numerical experiments, detailed in Section \ref{sec:additional-experiments}, whose code is available at \url{https://github.com/gianluigilopardo/anchors_text_theory}.
Unless otherwise specified, experiments use the official implementation of Anchors with all default options.

\paragraph{Related work.}
Among several methods for machine learning interpretability proposed in recent years \citep{guidotti2018local, adadi2018peeking,linardatos2021explainable}, rule-based methods are increasingly popular contenders. 
One reason is that users prefer rule-based explanations rather than alternatives \citep{lim2009and, stumpf2007toward}. 
Hierarchical decision lists \citep{wang2015falling} are useful for understanding the global behavior of a model, prioritizing the most interesting cases.
%However, small and disjoint rules are easier to understand: 
\citet{lakkaraju2016interpretable} compromises between accuracy and interpretability to extract small and disjoint rules, simpler to interpret, introducing the concept of \emph{coverage}. 
Alternatively, \citet{barbiero2022entropy} proposes to learn simple logical rules along with the parameters of the model itself, so as not to sacrifice accuracy.  

Many other approaches focus on \emph{local interpretability}, based on the idea that any black-box model can be approximated accurately by a simpler---easier to understand---model around a specific instance to explain. 
As an example, LORE \citep{guidotti2018local} uses a decision tree: the explanation is the list of logical conditions satisfied by the instance within the tree. 
A central point of perturbation-based methods is the sampling scheme. 
\citet{delaunay2020improving} modifies Anchors' sampling for tabular data, implementing the Minimum Description Length Principle discretization \citep{fayyad1993multi} to learn the minimal number of intervals needed to separate instances from distinct classes. 
\citet{amoukou2021consistent} proposes Minimal Sufficient Rules, similar to anchors for tabular data, extended to regression models, that can directly deal with continuous features, with no need for a discretization.  

There are few local, post-hoc explainability techniques for text data  \citep{danilevsky2020survey}.  
Among them, LIME \citep{ribeiro2016should} and SHAP \citep{lundberg2017unified} provide explanations using a linear model as local surrogate, trained on perturbed samples of the instance to explain. 
As we will see, while LIME and SHAP assign a weight to each word of the example, Anchors extracts the \emph{minimal} subset of words that is \emph{sufficient} to have, in \emph{high probability}, the same prediction as the example. 
\citet{delaunay2020improving} proposes to extend Anchors by also exploiting the absence of words.

In this work, our main concern is to provide some theoretical guarantees for interpretability methods. 
For feature importance methods, \citet{lundberg2017unified} provides insights  in the case of linear models for kernel SHAP, while \citet{mardaoui_garreau_2021} looks specifically into LIME for text data, extending \citet{garreau2020explaining}. 
These last two papers also consider simple if-then rules and linear models in their analysis. 
Another related work is \citet{agarwal2022probing}: graph neural networks explainers are compared in terms of faithfulness, stability, and fairness. 
As for rule-based methods, the question has not been addressed yet, to the best of our knowledge.

%%%%%%%%%%%%%%%%%%%%%%%%%%%%%%%%%%%%%%%%%%%%%%%%%%%%%%%%%%%%%%%%%%%%%%%%%%%%%%%%%%%%%%%%%%%%%%%%%%%%%%%%%%%%%%%%%%%%%%%%%%%%%%%

\section{ANCHORS FOR TEXT DATA}
\label{sec:anchors}
In this section, we present the operating procedure of Anchors for text data, as introduced by \citet{ribeiro2018anchors}. 
After specifying our setting and notation in Section~\ref{sec:setting-notation}, we present the key notions of \emph{precision} and \emph{coverage} in Section~\ref{sec:precision-coverage}. 
The algorithm is described in Section~\ref{sec:algoritgm}.
We give further details on the sampling scheme in Section~\ref{sec:sampling}.

\subsection{Setting and notation}
\label{sec:setting-notation}
Throughout this paper, we consider the problem of explaining the decision of a classifier $f$ taking documents as input. 
We will denote by $\doc$ a generic document, and by $\xi$ the particular example being explained by Anchors. 
Let us define~$\dictionary = \{\word_1,\ldots,\word_D\}$ the \emph{global dictionary} with cardinality $D = \card{\dictionary}$. 
We see any document as a finite sequence of elements of $\dictionary$. 
For a given document $\xi = \{\xi_1,\ldots,\xi_b\}$ of $b$ (not necessarily distinct) words, up to a re-ordering of $\dictionary$, we can set $\localdic_\xi = \{\word_1,\ldots,\word_d\} \in \dictionary$ the \emph{local dictionary}, containing the distinct words in $\xi$, with $d\leq b$. 
We denote by $\mult_j(\doc)$ the multiplicity of word $\word_j$ in $\doc$, \emph{i.e.},~$\mult_j(\doc) \defeq \card{\{k \in [b], \, \doc_k = \word_j\}}$. 
When the context is clear, we write $\mult_j$ short for $\mult_j(\xi)$. 
Finally, for any integer $k$, we write $[k]\defeq \{1,\ldots,k\}$. 

We make two restrictive assumptions on the class of models that we take into account. 
First, we restrict our analysis to \emph{binary classification} and write $f(\doc) = \indic{g(\doc) \in R}$, where $g : \texts \to \Reals^p$ is a given measurable function taking document as input, and $R$ is a collection of (potentially unbounded) intervals of $\Reals^p$. 
Second, we assume that $g$ relies on a \emph{vectorization} of the documents. 
More precisely, we assume that $g = h \circ \Vectorizer$, where $\Vectorizer$ is a deterministic mapping $\texts \to \Reals^D$ (detailed in Section~\ref{sec:vectorizers}) and $h : \Reals^D \to \Reals^p$. 
Without loss of generality, \textbf{we will always assume that the example $\xi$ is classified as positive}, \emph{i.e.}, $f(\xi) = 1$. 
In definitive, we consider models of the form 
\begin{equation}
\label{eq:def-model}
    f(\doc) = \indic{h(\vectorizer{\doc}) \in R} \,.
\end{equation}

We call an \emph{anchor} any non-empty subset of $[b]$, corresponding to a preserved set of words of $\xi$. 
We set $\Anchors$ the set of all candidate anchors for the example $\xi$. 
For any anchor $A \in \Anchors$, we set $\length{A}$ the length of the anchor, defined as the number of (not necessarily distinct) words contained in the anchor. 
In practice, \textbf{an anchor $A$ for a document $\xi$ is represented as a non-empty sublist of the words present in the document}, and this is the output of Anchors (illustrated in Figure~\ref{fig:example}).

\subsection{Precision and coverage}
\label{sec:precision-coverage}
The \emph{precision} of an anchor $A \in \Anchors$ is defined by \cite{ribeiro2018anchors} as the probability for a local perturbation of~$\xi$ to be classified as~$1$. 
Since we assume $f(\xi) = 1$, the precision can be written as
\begin{equation}
\label{eq:precision}
\precision{A} = \expecunder{f(x) = 1}{A} = \probaunder{g(\tfidf{x}) = 1}{A} \,,
\end{equation}
where the expectation is taken with respect to $x$, a random perturbation of $\xi$ still containing all the words included in the anchor~$A$.
We detail the sampling of $x$ further in Section~\ref{sec:sampling}. 
For the anchor containing all the words of $\xi$, the precision is exactly $1$, while smaller anchors have, in general, smaller precision. 

Of course, large anchors with size comparable to $b$ are not very interesting from the point of view of interpretability (the text in Figure~\ref{fig:example} would be completely highlighted). 
To quantify this idea, one can use the notion of \emph{coverage}, defined in our case as the proportion of documents in the corpus (\emph{i.e}, the dataset of documents on which the vectorizer is fitted) that contain the anchor. 
For instance, the coverage of the anchor in Figure~\ref{fig:example} is $0.12$, meaning that $12\%$ of the reviews contain it.  
The notions of precision and coverage are paramount to the Anchors algorithm: in a nutshell, \textbf{Anchors will look for an anchor of maximal coverage with prescribed precision.} 
We detail this in the next section.

\subsection{The algorithm}
\label{sec:algoritgm}
In practice, the coverage can be costly to compute, and in many cases a corpus is not available when the prediction is explained. 
Since anchors with smaller length tend to have larger coverage, a natural solution, used in the default implementation, is to minimize the length instead of maximizing the coverage, leading to:
\begin{equation}
\label{eq:anchors-optimization}
\Minimize_{\substack{A \in \Anchors }} \length{A} \,, \; \text{such that} \; \precision{A} \geq 1 - \epsilon 
\, ,
\end{equation}
where $\epsilon>0$ is a pre-determined tolerance threshold (set to $0.05$ in practice). 
The lower $\epsilon$ is, the harder it is to find an anchor satisfying Eq.~\eqref{eq:anchors-optimization}. 

Of course, the exact precision of a specific anchor $A \in \Anchors$ is unknown, since we cannot compute the expectation appearing in Eq.~\eqref{eq:precision} in general. 
The strategy used by \citet{ribeiro2018anchors} is to approximate $\precision{A}$ by $\Empprec_n(A)$, an empirical approximation, defined is Section~\ref{sec:exhaustive-anchors}. 
Let us note that the optimization problem in Eq.~\eqref{eq:anchors-optimization} is generally intractable, whatever the selection function may be. 
The cardinality of $\Anchors$ is simply too large in all practical scenarios. 
As a consequence, the default implementation applies the KL-LUCB \citep{kaufmann2013information} algorithm to identify a subset of rules with high precision: at the next step, this subset is used as representative of all candidate anchors, finding an approximate solution to Eq.~\eqref{eq:anchors-optimization}. 
In this paper, we do not consider this optimization procedure and consider an exhaustive version of Anchors, described in Section~\ref{sec:exhaustive-anchors}. 

\subsection{The sampling}
\label{sec:sampling}
We now detail the sampling procedure performed to compute the precision of an anchor (see Eq.~\eqref{eq:precision}). 
The idea is to look at the behavior of the model $f$ in a local neighborhood of $\xi$, while fixing the set of words in $A$. 
In the official implementation, this amounts to setting \texttt{use\_unk\_distribution=True} (default choice). 
Formally, for a given example $\xi$ and for a candidate anchor $A \in \Anchors$, Anchors generates perturbed samples $x_1,\ldots,x_n$ in the following way: 
\begin{enumerate}
\itemsep0em 
    \item{\textit{copies generation}}: create $n$ identical copies $x_1,\ldots,x_n$ of the example to explain $\xi$; 
    \item{\textit{random selection}}: for each word with index $k \in [b]$ not belonging to the anchor, draw $B_k\sim\binomial{n}{1/2}$ a number of copies to be perturbed (words in the candidate anchor are the blue columns of Table~\ref{tab:sampling});
    \item{\textit{word replacement}}: for each word not in the anchor, draw independently uniformly at random a set $S_k$ of cardinality $B_k$ of copies to be perturbed. Replace the words belonging to copies whose indices are in $S_k$ by the token ``UNK.''
\end{enumerate}
Note that the perturbation distribution described in this section is different from what is described in \cite{ribeiro2018anchors}, \emph{i.e.}, replacing selected words with others having the same \emph{part-of-speech} tag \say{with probability proportional to their similarity in an embedding space.} 
In fact, replacing words with a predefined token generates meaningless sentences which can fool a classifier, as it produces unrealistic samples \citep{hase2021out}. 
However, this option is not implemented in the official release, while it is possible to replace the selected words in step $3$ using BERT \citep{devlin2019bert}. 
In this work, we nevertheless consider the \emph{UNK}-replacement because
(i) we believe the default choices to be the most used by Anchors' users, thus the ones most needing interpretation and theoretical guarantees, 
and (ii) as we detail in Section \ref{sec:vectorizers}, in the case of TF-IDF vectorization, the \emph{UNK}-replacement exactly replicates word removals. 
Nevertheless, experiments show that our results still hold when the BERT-replacement is applied (see Section \ref{sec-sup:bert-replacement} of the Appendix). 
\begin{figure}[t]\resizebox{0.5\textwidth}{!}{
    \centering
    \begin{tabular}
    {c | c c >{\columncolor{anchors_color!50}} c c >{\columncolor{anchors_color!50}} c c >{\columncolor{anchors_color!50}} c >{\columncolor{anchors_color!50}} c c}
        & $\xi_1$ & $\xi_2$ & $\xi_3$ & $\xi_4$ & $\xi_5$ & $\xi_6$ & $\xi_7$ & $\xi_8$ & $\xi_9$ \\ 
        \hline\hline
        $\xi$ & the & quick & brown & fox & jumps & over & the & lazy & dog \\ 
        \hline
        $x_1$ & \textbf{UNK} & \textbf{UNK} & brown & fox & jumps & over & the & lazy & dog \\ 
        $x_2$ & the & quick & brown & \textbf{UNK} & jumps & \textbf{UNK} & the & lazy & dog \\
        \vdots & \vdots & \vdots & \vdots & \vdots & \vdots & \vdots & \vdots & \vdots & \vdots \\
        $x_n$ & the & quick & brown & \textbf{UNK} & jumps & over & the & lazy & \textbf{UNK}        
        \end{tabular}
        }
\caption{\label{tab:sampling}Anchors' sampling is performed in three main steps: copies generation, random selection, word replacing. Here, for instance,  for the fourth word (``fox''), $B_4=2$ and $S_4=\{2,n\}$, so the second and the $n$-th copies are considered for replacement. }
\vspace{-0.3cm}
\end{figure}

We remark that Anchors' sampling procedure is similar to that of LIME for text data \citep{ribeiro2016should}, with the crucial difference that LIME removes \emph{all occurrences} of a given word when it is selected for removal. 
We refer to \citet{mardaoui_garreau_2021} for more details. 
We show in Appendix~\ref{sec:proof-eq-sampling} 
that the sampling procedure can be described more simply: 

\begin{proposition}[Equivalent sampling]
\label{proposition:equivalent-sampling}
The sampling process described above is equivalent to replacing, for any sample $x_i$, each word $x_{i,k}$ such that $k\notin A$ independently with probability $1/2$.
\end{proposition}

Intuitively, parsing each line of Table~\ref{tab:sampling}, Anchors flips an imaginary coin for each word not belonging to the anchor, replaces it in the perturbed example if the coin hits heads and keeps it if it hits tails. 
Proposition~\ref{proposition:equivalent-sampling} is the first step in our analysis, giving us a simple description of $\Mult_j$, the random variable defined as the multiplicity of word $\word_j$ in the perturbed sample $x$. 
Namely, for any given anchor $A$, $\Mult_j \sim \anchor_j + \binomial{\mult_j-\anchor_j}{1/2}$, where $\anchor_j$ is the number of occurrences of $\word_j$ in $A$. 

%%%%%%%%%%%%%%%%%%%%%%%%%%%%%%%%%%%%%%%%%%%%%%%%%%%%%%%%%%%%%%%%%%%%%%%%%%%%%%%%%%%%%%%%%%%%%%%%%%%%%%%%%%%%%%%%%%%%%%%%%%%%%%%

\section{EXHAUSTIVE \texorpdfstring{$p$}{p}-ANCHORS}
\label{sec:exhaustive-anchors}
In this section, we present the central object of our study, exhaustive $p$-Anchors. 
In a nutshell, it is a formalized version of the original combinatorial optimization problem of Eq.~\eqref{eq:anchors-optimization} for any evaluation function $p:\Anchors\to\Reals$. 
We describe the procedure in Section~\ref{sec:exhaustive-description}, and thereafter provide a key stability property motivating further investigations in Section~\ref{sec:exhaustive-stability}. 

\subsection{Description of the algorithm}
\label{sec:exhaustive-description}
The optimization problem of Eq.~\eqref{eq:anchors-optimization} can be decomposed in two steps: first, all anchors in $\Anchors$ such that $\precision{A}\geq 1-\epsilon$ are selected. 
We call this first subset of anchors $\Anchors_1(\epsilon)$. 
Note that $\Anchors_1(\epsilon)\neq \emptyset$ since the full anchor $[b]$ has precision $1$. 
Then, among these anchors, the ones with minimal length are kept, giving raise to $\Anchors_2(\epsilon)$. 
At this point, it is not clear from Eq.~\eqref{eq:anchors-optimization} which anchors should be selected, and we settle for the ones with the highest precision. 
Equality cases can happen at this step (for instance, there can be several anchors with precision $1$): we call $\Anchors_3(\epsilon)$ the corresponding set of anchors. 
If $\Anchors_3(\epsilon)$ is not reduced to a single element, we draw uniformly at random the selected anchor. 
% extension to generic evaluation functions
Algorithm~\ref{algo:exhaustive-anchors} formally describes this procedure for a generic evaluation function $p:\Anchors\to\Reals$, which we illustrate in Figure~\ref{algo-fig:exhaustive-anchors}. 
When using~$p$, we write $\Anchors_k^p(\epsilon)$ the sets constructed and $A^p(\epsilon)$ the selected anchor. 

The goal here is to have a flexible framework: we can use Algorithm~\ref{algo:exhaustive-anchors} with $p=\Empprec_n$ or $p=\Precision$ as a selection function, or any other function which is a good approximation of $\Precision$. 
When $p=\Precision$, we call this version of the algorithm \emph{exhaustive Anchors}, 
%, and we will refer to it as $\Precision$-Anchors,
whereas when $p=\Empprec_n$ we call this version \emph{empirical Anchors}. 

Empirical Anchors is very similar to Anchors; the main difference is that the former is looking at all possible anchors, while the latter uses an efficient approximate procedure, which we do not consider here. 
%We show in the Section \ref{sec:comparison-anchors} of the Appendix that this similarity is also valid experimentally. 
A second difference is that empirical Anchors selects anchors with maximal precision in the third step. 
This is not necessarily the case with the default implementation, since an approximate procedure is used. 
We notice, nevertheless, that the chosen anchors tend to have high precision, and we demonstrate in Section~\ref{sec:comparison-anchors} of the Appendix that empirical Anchors and the default implementation give very similar output in practice. 

\subsection{Stability with respect to the evaluation function}
\label{sec:exhaustive-stability}
%
%In this section, 
We show that applying Algorithm~\ref{algo:exhaustive-anchors} to functions taking similar values on $\Anchors$ leads to similar results. 

\alglanguage{pseudocode}
\begin{algorithm}[t]
\caption{\label{algo:exhaustive-anchors}An overview of exhaustive $p$-Anchors.}
\begin{algorithmic}[t]
    \State \textbf{input} set of candidate anchors $\Anchors$, selection function $p:\Anchors \to \Reals$, tolerance threshold~$\epsilon$ \\
    \textbf{select} $\Anchors_1^p(\epsilon) = \left\{A\in\Anchors \;\text{s.t.}\; p(A) \geq 1-\epsilon\right\}$ \\
    \textbf{select} $\Anchors_2^p(\epsilon) =\displaystyle\argmin_{A'\in\Anchors_1^p(\epsilon)}{\length{A'}}$ \\
    \textbf{select} $\Anchors_3^p(\epsilon) =\displaystyle\argmax_{A'\in\Anchors_2^p(\epsilon)}{p(A')}$ \\
    \textbf{select} $A^p(\epsilon)\in\Anchors_3^p(\epsilon)$ uniformly at random \\
    \textbf{return} $A^p(\epsilon)$
\end{algorithmic} 
\end{algorithm}

\begin{proposition}[Stability of exhaustive \texorpdfstring{$p$}{p}-Anchors]
\label{prop:exhaustive-anchors-stability}
Let $\epsilon >0$ be a tolerance threshold, $p:\Anchors \to \Reals$ be an evaluation function, and set $\Astar\defeq A^p(\epsilon)$ the output of exhaustive $p$-Anchor. 
Assume that (i) $p(\Astar)\geq 1-\epsilon/4$, and (ii) $p(A)\leq 1-3\epsilon/4$ for any $A\in\Anchors^p_2(\epsilon)\setminus\{\Astar\}$. 
Let $q:\Anchors\to \Reals$ be another evaluation function such that 
\begin{equation}
\label{eq:epsilon-approx}
\delta \defeq \sup_{A\in\Anchors} \abs{p(A)-q(A)} < \frac{\epsilon}{4}
\, .
\end{equation}
Then $A^q(\epsilon-\delta)=\Astar$. 
\end{proposition}

% intuition behind the result
In other words, if $\Astar$ is a solution with high value for the chosen $p$ function, and we have a good approximation~$q$ of~$p$, then running Algorithm~\ref{algo:exhaustive-anchors} on~$q$ instead of~$p$ will yield approximately the same result. 
\textbf{This is the key motivation for studying $\Precision$ instead of $\Empprec$, and later considering further approximations to $\Precision$:} 
one can study directly $\Precision$, or an approximation thereof, and get insights on the output of the original algorithm. 
We prove Proposition~\ref{prop:exhaustive-anchors-stability} in Section~\ref{sec:proof-prop-exhaustive-anchors-stability} of the Appendix. 
\begin{figure}[t]
\centering
\includegraphics[scale=0.48]{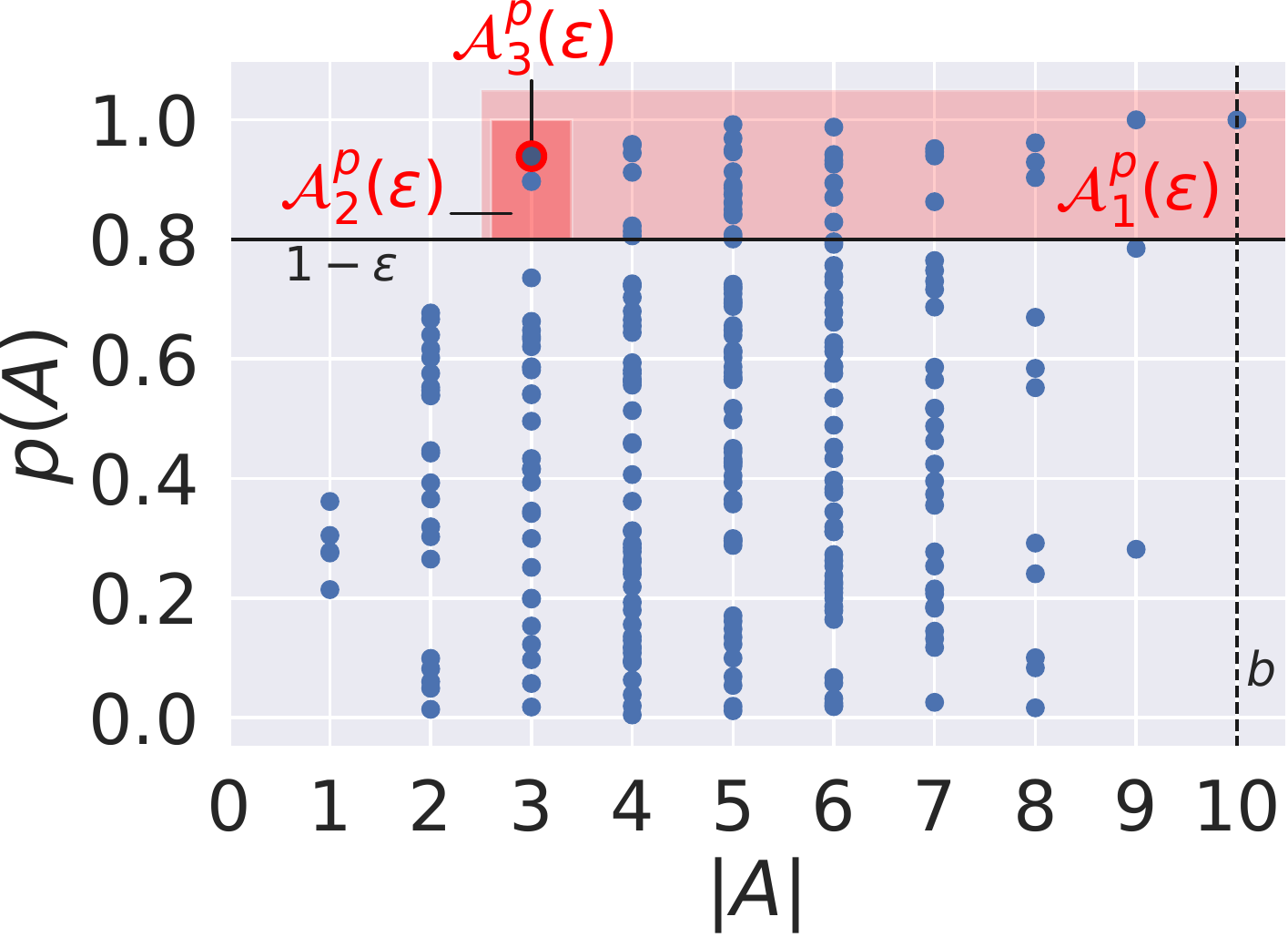}
\caption{\label{algo-fig:exhaustive-anchors}An illustration of Algorithm \ref{algo:exhaustive-anchors} with evaluation function $p=\Precision$. 
Each blue dot is an anchor, with $x$ coordinate its length and $y$ coordinate its value for $p$. 
Here, $\epsilon=0.2$ and the maximal length of an anchor is $b=10$ (the length of $\xi$). 
In the end, the anchor~$A$ such that $\length{A}=3$ and $p(A)=0.9$ is selected (red circle).}
\end{figure}

Note that, since we perturb the values of the $p$ function, an anchor with smaller length than $\Astar$ could cross the $1-\epsilon$ barrier and become a solution for exhaustive $q$-Anchors if we kept the same tolerance threshold $\epsilon$. 
This is for instance the case for the anchor with length two and highest value  of $p$ in Figure~\ref{algo-fig:exhaustive-anchors}. 
Therefore, the tolerance threshold has to decrease: Proposition~\ref{prop:exhaustive-anchors-stability} cannot be improved to show that $A^q(\epsilon)=\Astar$. 
However, having large value of $p$ for $A^p$ (assumption (i)) is not strictly necessary, what is important is that the gap between $A^p$ and the anchor with second-largest value of $p$ function in $\Anchors_2^p$ cannot be filled by $q$ (assumption (ii)). 
Otherwise, an anchor with the same length could get a larger value for $q$ that $A^p$ and be selected in the final step. 

As a first application of Proposition~\ref{prop:exhaustive-anchors-stability}, let us come back to the \emph{empirical precision}, %defined as
%\begin{equation} 
%\label{eq:empirical-precision}
\[
\Empprec_n(A) \defeq \frac{1}{n}\sum_{i=1}^n \indic{f(x_i)=1} 
\, ,
\]
%\end{equation}
where $n$ is the number of perturbed samples with $A$ fixed, as described in Section~\ref{sec:sampling}. 
Then the empirical precision satisfies the following:

\begin{proposition}[$\Empprec_n(A)$ uniformly approximates $\Precision$]
\label{prop:empirical-precision-concentration}
Recall that we denote by $b$ the number of words in $\xi$. 
Let $\delta >0$. 
With probability higher than $1-2^{b+1}\exps{-2n\delta^2}$, 
\begin{equation}
\label{eq:empirical-precision-concentration}
\forall A\in\Anchors, \qquad \abs{\Empprec_n(A) - \precision{A}} \leq \delta\, .
\end{equation}
\end{proposition}

In particular, Proposition~\ref{prop:empirical-precision-concentration} guarantees that $\Empprec_n$ and $\Precision$ satisfy Eq.~\eqref{eq:epsilon-approx} with high probability, as soon as $n\gg b/\epsilon^2$. 
This is the main motivation for studying \textit{exhaustive Anchors} in the next section: Proposition~\ref{prop:exhaustive-anchors-stability} shows that \textit{exhaustive Anchors} and \textit{empirical Anchors} will output the same result with high probability.  
We prove Proposition~\ref{prop:empirical-precision-concentration} in Section~\ref{sec:proof-empirical-precision-concentration} of the Appendix.

%%%%%%%%%%%%%%%%%%%%%%%%%%%%%%%%%%%%%%%%%%%%%%%%%%%%%%%%%%%%%%%%%%%%%%%%%%%%%%%%%%%%%%%%%%%%%%%%%%%%%%%%%%%%%%%%%%%%%%%%%%%%%%%

\section{ANALYSIS ON EXPLAINABLE CLASSIFIERS}
\label{sec:analysis}
Before presenting our main results, we describe the vectorizer we are considering in Section~\ref{sec:vectorizers}. 
We then go into the analysis by studying Anchors' behavior when applied to simple rule-based classifiers in Section~\ref{sec:simple-rules} and to linear models in Section~\ref{sec:logistic}. 
All our claims are supported by mathematical proofs, available in Section~\ref{sec:all-proofs} of the Appendix, and validated by reproducible experiments. 

\subsection{Vectorizers and immediate consequences}
\label{sec:vectorizers}
Natural language processing classifiers are mostly based on a vector representation $\Vectorizer$ of the document \citep{young2018recent}. 
The TF-IDF (Term Frequency-Inverse Document Frequency) transform \citet{luhn1957statistical} is popularly used to obtain such vectorization. 
The principle is to assign more weight to words that appear frequently in the document $\doc$, and not so frequently in the corpus $\corpus$. 
We perform our analysis by assuming that models work with the (non-normalized) TF-IDF vectorizer: 

\begin{definition}[TF-IDF]
\label{def:tf-idf}
Let $N$ be the size of the initial corpus $\corpus$, \emph{i.e.}, the number of documents in the dataset. 
Let $N_j$ be the number of documents containing the word $\word_j \in \dictionary$. 
The TF-IDF of $\doc$ is the vector $\tfidf{\doc} \in \Reals^D$ defined as 
\[
\forall j \in [D], \qquad\tfidf{\doc}_j \defeq \mult_j(\doc) \idf_j 
\,,
\]
where $\idf_j \defeq \log{\frac{N + 1}{N_j + 1}} + 1$ is the \emph{inverse document frequency} (IDF) of $\word_j$ in  $\corpus$.
\end{definition}

Note that once the TF-IDF vectorizer is fitted on a corpus $\corpus$, the vocabulary is fixed once and forever afterward. 
Meaning that, if a word is not part of the initial corpus $\corpus$, its IDF term is zero. 
As seen in Section~\ref{sec:sampling}, Anchors perturbs documents by replacing words with a fixed token \say{UNK}. 
We make the (realistic) assumption that the word \say{UNK} is not present in the initial corpus. 
Thus, \textbf{replacing any word with this token is equivalent to simply removing it, from the point of view of TF-IDF}. 

In this paper, we will always consider models trained on a (non-normalized) TF-IDF vectorization $\Tfidf$ as in Definition~\ref{def:tf-idf}. 
% normalized TF-IDF
However, we show in Section~\ref{sec:supp-normalized-tf-idf} of the Appendix that the same results hold for the $\ell_2$-normalized TF-IDF vectorization, defined as
\[
\forall j \in [D], \qquad \phi(\doc)_j \defeq \frac{\mult_j(\doc) \idf_j}{\sqrt{\sum_{j=1}^D \mult_j(\doc)^2 \idf_j^2}} 
\, .  
\]
We note that this is the default normalization in the \texttt{scikit-learn} implementation of TF-IDF. 

Since our models are of the form Eq.~\eqref{eq:def-model}, whenever the vectorizer is applied, \textbf{the exact location of the words in the document does not matter}. 
Therefore, when computing precision, only the occurrences of word $\word_j$ in anchor $A$ matter. 
For this reason, we write an anchor as $A = (\anchor_1,\ldots,\anchor_d,\ldots,\anchor_D)$. 
Since $\anchor_j \leq \mult_j$ for all $j \in [D]$, we see that $\anchor_j = 0$ for any $j > d$. 
Thus, we can write $A = (\anchor_1,\ldots,\anchor_d)$ without ambiguity. 

With this notation in hand, according to the discussion following Proposition~\ref{proposition:equivalent-sampling}, the TF-IDF of $\word_j$ in the perturbed sample $x$ will satisfy
\begin{equation}
\label{eq:key-proof}
\tfidf{x}_j = \Mult_j\idf_j \sim (\anchor_j+\binomial{\mult_j-\anchor_j}{1/2})\idf_j
\, ,
\end{equation}
where $\Mult_j$ is the random multiplicity of the word $\word_j$. 
Intuitively, $\Mult_j$ corresponds to $\anchor_j$ occurrences of $\word_j$ which cannot be removed, plus a random number of occurrences depending on the sampling.  
%, corresponding to the words in the anchor $A$, plus a random subsample of words not belonging to $A$, each present with probability $1/2$. 
%The quantity $\Mult_j$ is indeed the multiplicity of word $\word_j$ for the sample $x$, as by definition of TF-IDF. 
This has several important consequences in our analysis, the first being:

\begin{proposition}[Dummy features]
\label{prop:dummy-features}
Let $f$ be defined as in Eq.~\eqref{eq:def-model} and assume that $g$ does not depend on coordinate $j$. 
Let $\epsilon >0$ be a tolerance threshold. 
Then, for any anchor $A \in\Anchors_3^{\Precision}(\epsilon)$, $\anchor_j=0$.
\end{proposition}

Proposition~\ref{prop:dummy-features} is a natural property: if the model does not depend on word $\word_j$, then it should not appear in the explanation. 
It is often investigated in the interpretability literature, originally introduced by \citet{friedman_2004}. 
Here we use the vocabulary of \citet{sundararajan_et_al_2017}, which introduced the notion as an axiom for feature importance. 

Note that Proposition~\ref{prop:dummy-features} is not satisfied by the empirical version (with $p=\Empprec_n$), neither by the default implementation of Anchors (see Section~\ref{sec:dummy-property-experiments} in 
the Appendix).
We prove Proposition~\ref{prop:dummy-features} in Section~\ref{sec:proof-prop-dummy-features} of 
the Appendix.

%%%%%%%%%%%%%%%%%%%%%%%%%%%%%%%%%%%%%%%%%%%%%%%%%%%%%%%%%%%%%%%%%%%%%%%%%%%%%%%%%%%%%%%%%%%%%%

\subsection{Simple if-then rules}
\label{sec:simple-rules}
We now focus on classifiers relying on the presence or absence of given words. 
In this setting, the function $g$ introduced in Eq.~\eqref{eq:def-model} will take a simple form. 
Indeed, since we are working with the TF-IDF vectorizer, the \emph{presence} (resp. \emph{absence}) of word $\word_j$ in $x$ simply corresponds to the condition $\tfidf{x}_j > 0$ (resp. $\tfidf{x}_j = 0$). 

Thus $g$ is the projection on the relevant coordinates, and using Eq.~\eqref{eq:key-proof} we will be able to compute the precision of any given anchor. 
We show here the case of a model classifying documents according to the presence of a set of words. 
Different cases are presented in Section \ref{sec:additional_analysis} of the Appendix. 

\begin{figure}[t]
    \centering
    \includegraphics[scale=0.34]{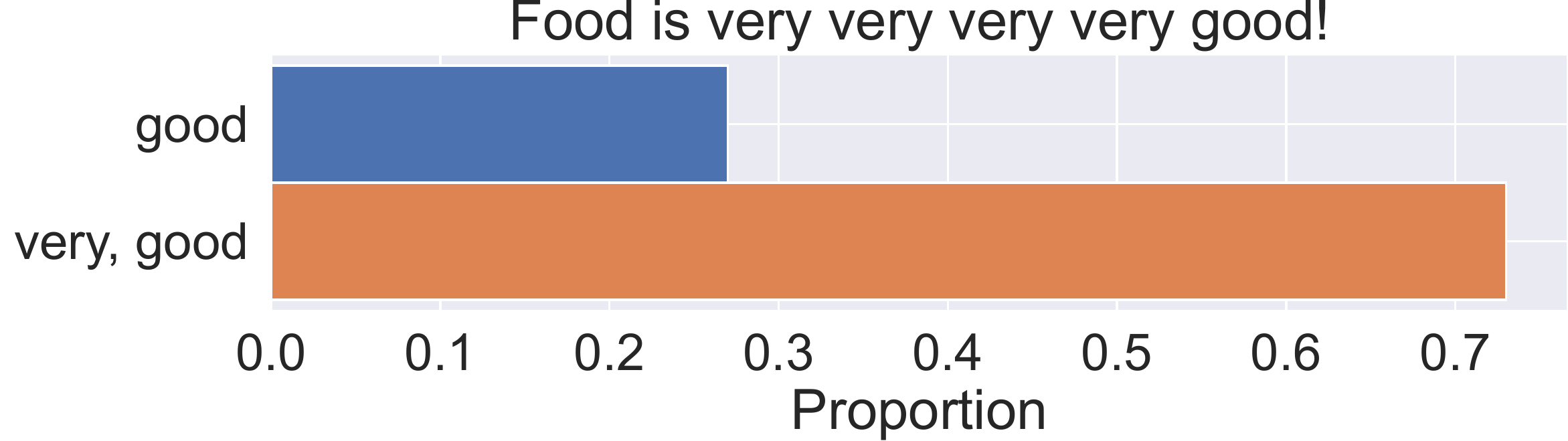} \\ \vspace{0.5cm}
    \includegraphics[scale=0.34]{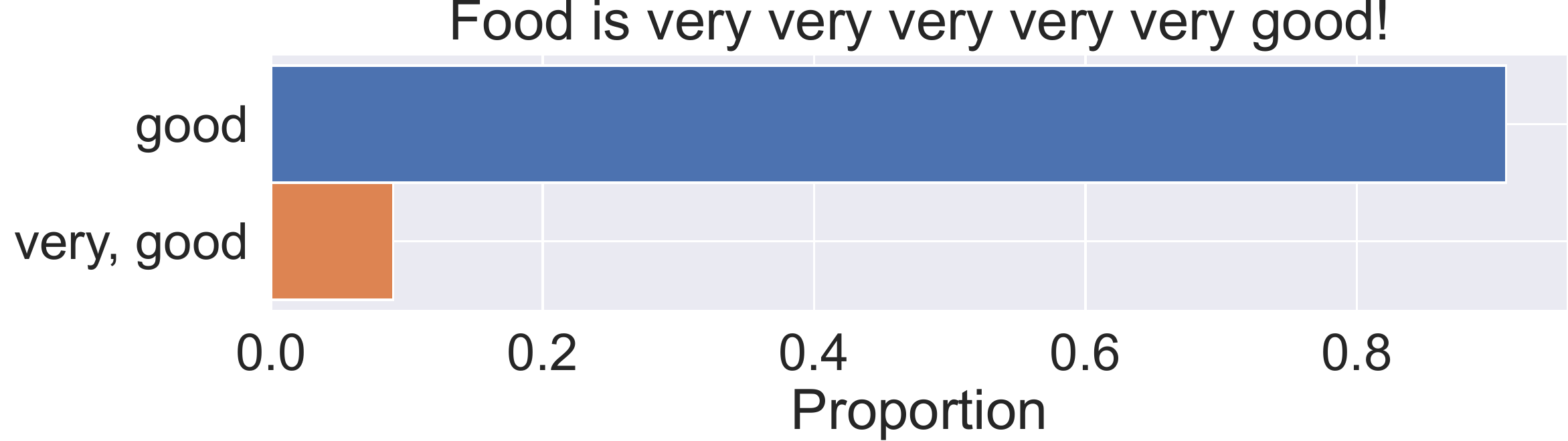}
    \caption{\label{fig:product}Making a word disappear from the explanation by adding one occurrence. The model predicts $1$ if ``very'' and ``good'' are present. The count corresponds to the appearance of the word(s) in the selected  anchor on $100$ runs of Anchors.  When the multiplicity of ``very'' in the document crosses the breakpoint value $B=4$ (by default, $\varepsilon=0.05$), it disappears from the selected anchor with high probability (on the top panel, $\mult_{\text{very}}=4$, on the bottom panel, $\mult_{\text{very}}=5$).}
\end{figure}

\begin{proposition}[Presence of a set of words]
\label{prop:product}
Assume that $\mult_1 > \cdots > \mult_k$. 
Let $J = \{1,\ldots,k\}$ be a subset of $[d]$, and assume that the model is defined as
\[
f(\doc) = \indic{\word_j \in \doc,\; \forall j \in J} = \prod_{j \in J} \indic{\word_j \in \doc} = \prod_{j \in J} \indic{\tfidf{\doc}_j > 0} 
\, .
\]
Let us define the quantities $B \defeq \floor{\log_2 1 / \epsilon}$ and $c_0 \defeq \inf \{c \in [k], \; \prod_{\ell = 1}^{k - c + 1} (1-2^{-\ell}) \geq 1 - \epsilon \}$. 
If there exists $j \in J$ such that $\mult_j > B$, then the anchor $A_{c_0}$ such that $\anchor_j=1$ for all $j\in \llbracket k-c_0+1,k\rrbracket$ and $\anchor_j=0$ otherwise
%$A_{c_0} \defeq (0,\ldots,0,1,\ldots,1)$ with $c_0$ $1$s 
%\{\word_\ell, \; \ell \in \llbracket k - c_0 + 1, k \rrbracket\}$ 
will be selected by exhaustive Anchors. 
On the contrary, if $\mult_j < B$ for all $j\in [d]$, the anchor $A_J$ such that $\anchor_j=1$ for all $j\in [k]$ and $\anchor_j=0$ otherwise 
%$A_J \defeq (1,\ldots,1)$ 
will be selected. 
\end{proposition}

% explaining the result
Note that $A_J$ contains all the words with index in $J$, \emph{i.e.}, each word in the support of the classifier~$f$. 
Instead, $A_{c_0}$ contains the $c_0$ words with the lowest multiplicity among those indexed in $J$. 
We prove Proposition~\ref{prop:product} in Section~\ref{sec:proof-prop-product} of  the Appendix. 

% concrete example 
As a concrete example, let us consider a model classifying reviews as positive if both the words \say{very} and \say{good} are present. 
The review \say{Food is very good!} will be classified as a positive one and Anchors extracts $\{\text{\say{very}}, \text{\say{good}}\}$ as an anchor, which makes a lot of sense. 
However, if the multiplicity of the word \say{very} (resp. \say{good}) exceeds the breakpoint value $B = 4$ (by default, $\epsilon = 0.05$), Proposition~\ref{prop:product} predicts that Anchors will only extract $\{\text{\say{good}}\}$ (resp. $\{\text{\say{very}}\}$). 
In this example, \textbf{we are effectively seeing a word disappear from an explanation simply because its multiplicity crosses an arbitrary threshold.} 
% This is a surprising phenomenon, which is not general: we will see later that for some models the multiplicity of a word does not influence directly its presence in the selected anchor. 
This is verified in practice, with the default implementation of Anchors (see Figure~\ref{fig:product}), with some variability coming from the sampling and the approximate optimization scheme of Anchors.

%%%%%%%%%%%%%%%%%%%%%%%%%%%%%%%%%%%%%%%%%%%%%%%%%%%%%%%%%%%%%%%

\subsection{Linear classifiers}
\label{sec:logistic}
We now shift our focus to \emph{linear classifiers}. 
Namely, in this section, for any document $\doc$, we have
\begin{equation} 
\label{eq:def-linear-classifier}
f(\doc) = \indic{\lambda^\top \tfidf{\doc} + \lambda_0 > 0}
\, ,
\end{equation}
where $\lambda \in \Reals^D$ is a vector of coefficients, and $\lambda_0 \in \Reals$ is an intercept. 
Coming back to Eq.~\eqref{eq:def-model}, $g(y) = \lambda^\top y + \lambda_0$, and $R = (0, +\infty)$. 
Eq.~\eqref{eq:def-linear-classifier} encompasses several important examples, two of which we investigate empirically: 
\begin{itemize}
    \item \textbf{the perceptron} \citep{rosenblatt_1958}, predicting exactly as in Eq.~\eqref{eq:def-linear-classifier};
    \item \textbf{logistic models}, predicting as $1$ if $\sigma(\lambda^\top \tfidf{\doc} + \lambda_0) > 1/2$, where $\sigma : \Reals \to \sigma(t) = \frac{1}{1 + \exps{-t}} \in [0,1]$ is the \emph{logistic function}. Since $\sigma(y)>1/2$ if, and only if, $y>0$, they can also be rewritten as in Eq.~\eqref{eq:def-linear-classifier}. 
\end{itemize}
Of course, this list is not exhaustive, we refer to Chapter~4 in \citet{hastie_et_al_2009} for a more complete overview.
In this setting, starting from Eq.~\eqref{eq:key-proof}, we show that the precision satisfies a Berry-Esseen-type statement \citep{berry_1941,esseen_1942}:

%\begin{proposition}[Approximating the precision of a linear classifier]
\begin{proposition}[Precision of a linear classifier]
\label{prop:precision-logistic}
%Assume that the coefficient $\lambda\in\Reals^D$ and the inverse document frequencies $\idf_j$ are such that 
Let $\lambda,\lambda_0$ be the coefficients associated to the linear classifier defined by Eq.~\eqref{eq:def-linear-classifier}. 
Assume that for all $j\in [d]$, $\lambda_j\idf_j\neq 0$. 

Define, for all $A \in\Anchors$, 
\begin{align}
\label{eq:arg-approx-prec-logistic}
%\approxprec{A} & =
%\approxprec{\anchor_1,\ldots,\anchor_d} \nonumber \\
\approxprec{A} & \defeq \frac{- \lambda_0 - \frac{1}{2} \sum_{j=1}^d \lambda_j \idf_j (\mult_j + \anchor_j)}{\sqrt{\frac{1}{4} \sum_{j=1}^d \lambda_j^2 \idf_j^2 (\mult_j - \anchor_j)}}
\, .
\end{align}
Let $\Phibar\defeq 1-\Phi$, where $\Phi$ denotes the cumulative distribution function of a $\gaussian{0}{1}$. 
Then, for any $A\in\Anchors$ such that $\length{A} \leq b / 2$, 
\begin{align}
\label{eq:precision-logistic}
\abs{\precision{A} - \phibar{\approxprec{A}}} & \leq \\ \nonumber
C\cdot \left(\frac{\max_j \lambda_j^2\idf_j^2}{\min_j \lambda_j^2\idf_j^2}\right)^{3/2}  & \cdot \left(\frac{\max_j\mult_j}{\min_j\mult_j}\right)^{3/2} \cdot \frac{1}{\sqrt{d}}
\, ,
\end{align}
where $C\approx 7.15$ is a numerical constant. 
\end{proposition}

In other words, when $d$ is large, \textbf{the precision of an anchor for a linear classifier can be well approximated by $\Phibar\circ \Approxprec$} and we can use a (local) version of Proposition~\ref{prop:exhaustive-anchors-stability} to study exhaustive $p$-Anchors with $p=\Phibar \circ \Approxprec$ instead of exhaustive Anchors. 
%The dependency in $d$ is tight, as in Berry-Esseen. 
Nevertheless, we observe a very good fit between the two terms even for small values of $d$ (see Section~\ref{sec:check-prop-precision-logistic} in the Appendix). 
Proposition~\ref{prop:precision-logistic} is proven in Section~\ref{sec:proof-prop-precision-logistic} of the Appendix. 
In Section~\ref{sec:supp-normalized-tf-idf} of the Appendix we show that in the case of normalized TF-IDF a constant with the same rate appears. 

% discuss the assumptions
A typical value for $\idf_j$ and $\mult_j$ in our setting lies between $1$ and $10$ (see Section~\ref{sec:empirical-investigation} in the Appendix). 
Thus, the main assumption is the absence of zero components in $\lambda$. 
%What could happen if too many $\lambda_j$ were equal to zero
Note also that the result is only true for anchors having a length less than half the document length. 
This is realistic, since an explanation based on more than half of the document does not occur in practice and is not really interpretable. 
%It is possible to improve in the constant by ensuring that the anchors are even smaller in relative size. 

In view of Proposition~\ref{prop:exhaustive-anchors-stability}, we can now focus on exhaustive $\Phibar\circ\Approxprec$-Anchors. 
Let us set $\gamma \defeq \lambda_0 + \sum_j \lambda_j\idf_j\mult_j$. 
Note that, since we assume $f(\xi)=1$, $\gamma >0$. 
Let us also set
\begin{equation}
\label{eq:def-positive-anchors}
\Anchors_+ \defeq \{A\in\Anchors, \text{ s.t. } \anchor_j > 0 \Rightarrow \lambda_j > 0\}
\, ,
\end{equation}
the set of anchors with support corresponding to words with a positive influence. 
\begin{figure}
    \centering
    \begin{tikzpicture}
    % axis
    \draw[<-] (0,-0.7) -- (0,8) node[left, midway, rotate=90, yshift=1.0cm, xshift=2cm] {$\phibar{\approxprec{A}} \approx\precision{A}$};

    % anchors
    \draw (-0.1,7.7) -- (0.1,7.7) node [right] {$(1,0,0,0,0,\ldots,0)$};
    \draw (-0.1,7.0) -- (0.1,7.0) node [right] {$(2,0,0,0,0,\ldots,0)$};
    \draw (0.0, 6.35) node [right, xshift=0.4cm] {$\vdots$};
    \draw (-0.1,5.5) -- (0.1,5.5) node [right] {$(\mult_1,0,0,0,0,\ldots,0)$};
    \draw (-0.1,5.1) -- (0.1,5.1) node [right] {$(\mult_1,1,0,0,0,\ldots,0)$};
    \draw (0.0, 4.45) node [right, xshift=0.4cm] {$\vdots$};
    \draw (-0.1,3.7) -- (0.1,3.7) node [right] {$(\mult_1,\mult_2,0,0,0,\ldots,0)$};
    \draw (-0.1,3.0) -- (0.1,3.0) node [right] {$(\mult_1,\mult_2,1,0,0,\ldots,0)$};
    \draw (0.0, 2.35) node [right, xshift=0.4cm] {$\vdots$};
    \draw (-0.1,1.5) -- (0.1,1.5) node [right] {$(\mult_1,\mult_2,\mult_3,1,0,\ldots,0)$};
    \draw[anchors_color] (-0.1,0.5) -- (0.1,0.5) node [right] {$(\mult_1,\mult_2,\mult_3,2,0,\ldots,0)$};
    \draw (-0.1,0.0) -- (0.1,0.0) node [right] {$(\mult_1,\mult_2,\mult_3,3,0,\ldots,0)$};

    % threshold
    \draw[red, line width=1.5pt] (-0.2,0.9) node [left] {$1-\varepsilon$} -- (4.0,0.9);

\end{tikzpicture}
    \caption{\label{fig:linear_illustation}Illustration of Proposition \ref{prop:approx-prec-maximization}. On linear models, the algorithm includes words having the highest $\lambda_j\idf_j$s first. Finally, the minimal anchor satisfying the precision condition $\phibar{\approxprec{A}} \approx\precision{A} \geq 1-\varepsilon$ is selected, which is $A=(\mult_1,\mult_2,\mult_3,2,0,\ldots,0,0)$ in the example.}
\end{figure}
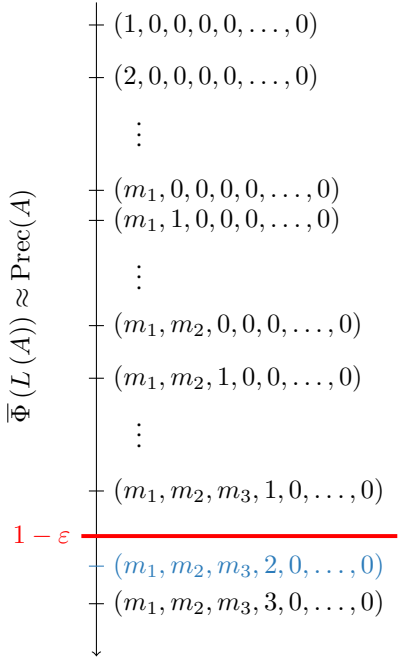
\begin{proposition}[Approximate precision maximization]
\label{prop:approx-prec-maximization}
%Assume that the $\lambda_j\idf_j$s are all distinct and order the indices of the local dictionary such that $\left(\lambda_j\idf_j\right)_{j\in [d]}$ is a decreasing sequence. 
Assume that $\lambda_1\idf_1 > \lambda_2\idf_2 > \cdots > \lambda_d\idf_d$, with at least one $\lambda_j$ greater than zero. 
Assume that $\lambda_0 > -\gamma / 2$. 
Then Algorithm~\ref{algo:exhaustive-anchors} applied to the selection function $p\defeq\Phibar \circ L$ will select an anchor $A^p\in\Anchors_+$ such that there exists $j_0\in [d]$ with the following property: for all $j<j_0$, $\anchor_j=\mult_j$, $\anchor_{j_0}\leq\mult_j$, and for all $j\geq j_0$, $\anchor_j=0$. 
\end{proposition}

\begin{table}[t]
\caption{\label{tab:logistic-similarity}Validation of Proposition~\ref{prop:approx-prec-maximization}. Average Jaccard similarity between the anchor $A$ and the first $\length{A}$ words ranked by $\lambda_j\idf_j$ for a logistic model on positive documents and low-confidently classified subset ($\texttt{pr} = g(\tfidf{\xi}) < 0.85$, or $\texttt{pr} < 0.75$).
}
    \centering
    \setlength{\extrarowheight}{4pt}%
    \setlength{\tabcolsep}{5pt}
    \begin{tabular}{c | c c c}
     & Restaurants & Yelp & IMDB \\
    \hline\hline
    full & $0.69 \pm 0.2$ & $0.35 \pm 0.2$ & $0.32 \pm 0.2$ \\
    $\texttt{pr} < 0.85$ & $0.78 \pm 0.2$ & $0.74 \pm 0.2$ & $0.45 \pm 0.2$ \\
    $\texttt{pr} < 0.75$ & $0.82 \pm 0.1$ & $0.80 \pm 0.2$ & $0.59 \pm 0.1$ \\ 
    \end{tabular}
\end{table}

In plain words, Proposition~\ref{prop:approx-prec-maximization} implies that \textbf{for a linear classifier, Anchors keeps only words with a positive influence on the prediction}. 
Moreover, it \textbf{selects the words having the highest $\lambda_j\idf_j$s first}, adding them to the anchor until the precision condition is met. 
This is a reassuring property of Anchors. 
We prove Proposition~\ref{prop:approx-prec-maximization} in Section~\ref{sec:proof-prop-approx-prec-maximization} of the Appendix. 
To demonstrate this phenomenon, we conducted the following experiment. 
We first trained a logistic model on three review datasets, achieving accuracies between $85\%$ and $88\%$ on the test set. 
We then ran Anchors with the default setting $10$ times on positively classified documents. 
For each document, we measure the Jaccard similarity between the anchor $A$ and the first $\length{A}$ words ranked by $\lambda_j\idf_j$.  
In Table~\ref{tab:logistic-similarity} we report the average Jaccard index: results confirm  Proposition~\ref{prop:approx-prec-maximization}. 

Note that the official Anchors' implementation does not apply step $3$ of Algorithm~\ref{algo:exhaustive-anchors}: when the prediction is \emph{easy} (for instance, $g(\tfidf{\xi}) \geq 0.75$ or $g(\tfidf{\xi}) \geq 0.85$), realistically $\Anchors_3$ contains more than one anchors, and the algorithm will randomly select among them. 
This explains the different similarity between the full dataset and the \emph{hard} subset. 
We also remark that the individual multiplicities do not come into play in the ranking of the $\word_j$s, unlike the discussion following Proposition~\ref{prop:product}. 
This behavior is consistent with all models that we tested (see Section~\ref{sec:additional_analysis} of the Appendix).

\section{ANCHORS ON NEURAL NETWORKS}
\label{sec:neural-nets}
\begin{table}
\centering
\caption{\label{tab:nn-similarity}Average Jaccard similarity between the extracted anchor $A$ and the first $\abs{A}$ words ranked by $\lambda_j\idf_j$ for a $3$ (top table), $10$ (middle), and $20$-layers (bottom) feed-forward neural network for positively classified documents and low-confidently ($\texttt{pr} = g(\tfidf{\xi}) < 0.85$, or $\texttt{pr} < 0.75$) classified subsets.}
    %\centering
    %
    \setlength{\extrarowheight}{4pt}%
    \setlength{\tabcolsep}{5pt}
    \begin{tabular}{c | c | c c c}
     \multicolumn{2}{c |}{} & Restaurants & Yelp & IMDB \\
    \hline \hline
    \multirow{3}{*}{\rotatebox[origin=c]{90}{$3$-layers}} & 
    full & $0.68 \pm 0.2$ & $0.50 \pm 0.3$ & $0.45 \pm 0.3$ \\
    & $\texttt{pr} < 0.85$ & $0.68 \pm 0.2$ & $0.88 \pm 0.1$ & $0.52 \pm 0.3$ \\ 
    & $\texttt{pr} < 0.75$ & $0.74 \pm 0.2$ & $0.82 \pm 0.1$ & $0.68 \pm 0.2$ \\ 
    %\end{tabular} \\
    
        %\begin{tabular}{c | c c c}
     % & Restaurants & Yelp & IMDB \\
    \hline
    \multirow{3}{*}{\rotatebox[origin=c]{90}{$10$-layers}} 
    & full & $0.76 \pm 0.2$ & $0.56 \pm 0.2$ & $0.55 \pm 0.2$ \\
    & $\texttt{pr} < 0.85$ & $0.80 \pm 0.2$ & $0.78 \pm 0.1$ & $0.79 \pm 0.2$ \\
    & $\texttt{pr} < 0.75$ & $0.83 \pm 0.1$ & $0.69 \pm 0.2$ &  $0.81 \pm 0.2$ \\
    %\end{tabular} \\
    
    %\begin{tabular}{c | c c c}
     % & Restaurants & Yelp & IMDB \\
    \hline
    \multirow{3}{*}{\rotatebox[origin=c]{90}{$20$-layers}} 
    & full & $0.73 \pm 0.1$ & $0.60 \pm 0.3$ & $0.63 \pm 0.2$ \\
    & $\texttt{pr} < 0.85$ & $-$ & $0.81 \pm 0.2 $ & $0.69 \pm 0.2$ \\
    & $\texttt{pr} < 0.75$ & $-$ & $-$ & $0.74 \pm 0.1 $ \\
    \end{tabular}

\end{table}

In this section, we show empirical results for neural networks linking the explanations provided by Anchors with the partial derivatives of the model with respect to the input. 
Intuitively, while looking at a specific prediction, Anchors generates local perturbations of the example to explain. 
The behavior of a neural network on such a local neighborhood of the example, at order one, is approximately linear. 
This implies that Proposition~\ref{prop:precision-logistic} roughly remains true, taking as linear coefficients 
\[
\forall j \in [d]\,, \qquad \lambda_j \defeq \frac{\partial g(\tfidf{x})}{\partial \tfidf{x}_j} 
\, ,
\] 
where $\frac{\partial g(\tfidf{x})}{\partial \tfidf{x}_j}$ is the partial derivative of the model $g$ with respect to the word $\word_j$. 
In practice, the implication is that \textbf{Anchors selects the words corresponding to the highest partial derivatives of the model with respect to the input, reweighted by the inverse document frequencies,} until the precision condition is met.

To validate this conjecture, we trained three feed-forward neural networks on three datasets, and measured for each document of the test set the Jaccard similarity between the anchor $A$ and the first $\length{A}$ words ranked for all $j \in [d]$, by $\lambda_j\idf_j = \frac{\partial g(\tfidf{x})}{\partial \tfidf{x}_j} \idf_j$.
This is similar to the experiments from the previous Section.
Details on the training are reported in Section~\ref{sec:additional_nn} of the Appendix, we achieved  accuracy around $90\%$ in each case.  
We ran Anchors with default setting $10$ times on positively classified documents to account for the randomness in Anchors' optimization. 

Table~\ref{tab:nn-similarity} %, \ref{tab:nn10-similarity}, and \ref{tab:nn20-similarity} 
shows the results for the three networks. 
%a $3$, $10$, and $20$-layers feed-forward neural network, applied to Restaurants, Yelp, and IMDB datasets for sentiment analysis. 
There is a significant overlap between the anchors selected by Anchors and the subset suggested by our analysis, which becomes a near match for examples that are hard to predict. 
As for the previous section, we notice that when the prediction is \emph{easy}, \emph{i.e.}, the classifier is very confident about the prediction, the anchor is selected at random among many candidates. This is due to the fact that one needs to strongly perturb (that is, remove a lot of words from) a document confidently classified to change its prediction.  
Because these models perform better than linear models, confidence is frequently extremely high, making the process much more random, which motivates lower similarity values. 
More details about the experiments are available in Section~\ref{sec:additional_nn} of the Appendix. 

Compared with the Anchors' algorithm for text classification, obtaining explanations for a prediction in this way would be a faster and more efficient procedure (since the randomness due to the optimization scheme is avoided) for obtaining explanations on neural networks and, we conjecture, any other differentiable classifier. 
Clearly, this would not be a model-agnostic approach, since it requires knowledge of the model gradient and the inverse document frequencies for each example to explain. 

We note that this is a somewhat surprising result: without our theoretical analysis and empirical evidence, we would intuitively expect to obtain explanations for this class of models from the gradient reweighted by the input (\emph{i.e.}, the entire vectorization). 

We conjecture that, for any differentiable classifier, it is possible to predict the behavior of Anchors by extending the results for linear classifiers, considering a first-order approximation of the model.

%%%%%%%%%%%%%%%%%%%%%%%%%%%%%%%%%%%%%%%%%%%%%%%%%%%%%%%%%%%%%%%%%%%%%%%%%%%%%%%%%%%%%%%%%%%%%%%%%%%%%%%%%%%%%%%%%%%%%%%%%%%%%%%

\section{Conclusion}
\label{sec:conclusion}
In this paper, we presented the first theoretical analysis of Anchors. 
Specifically, we formalized the implementation for textual data, in particular giving insights on the sampling procedure. 
We then studied Anchors' behavior on simple if-then rules and linear models. 
To this end, we introduced an approximate, tractable version of the algorithm, which is close to the default implementation. 
Our analysis showed that Anchors provides meaningful results when applied to these models, which is supported by experiments with the official implementation. 
Finally, we exploited our theoretical claim about explainable classifiers to obtain empirical results for neural networks, yielding a surprising result that links the classifier gradient to the importance of words for a prediction. 
When having access to the model, this result can be used as a faster and more efficient method of obtaining explanations. 

This work uncovered some surprising results that emphasize the importance of theoretical analysis in the development of explainability methods. 
We believe that the insights presented in this article may be valuable for researchers and practitioners in natural language processing who seek to correctly interpret Anchors' explanations. Furthermore, the analysis framework we developed can aid the explainability community in designing new methods based on sound theoretical foundations and in scrutinizing existing ones.
As future work, we plan to extend this analysis to other classes of models, such as CART trees, and to more advanced text vectorizers. 
We also plan to study Anchors' behavior on images and tabular data.

%%%%%%%%%%%%%%%%%%%%%%%%%%%%%%%%%%%%%%%%%%%%%%%%%%%%%%%%%%%%%%%%%%%%%%%%%%%%%%%%%%%%%%%%%%%%%%%%%%%%%%%%%%%%%%%%%%%%%%%%%%%%%%%

\subsubsection*{Acknowledgements}

This work has been supported by the French government, through the NIM-ML project (ANR-21-CE23-0005-01), and by EU Horizon 2020 project AI4Media (contract no.~951911).

%%%%%%%%%%%%%%%%%%%%%%%%%%%%%%%%%%%%%%%%%%%%%%%%%%%%%%%%%%%%%%%%%%%%%%%%%%%%%%%%%%%%%%%%%%%%%%%%%%%%%%%%%%%%%%%%%%%%%%%%%%%%%%%

% \subsubsection*{References}
\bibliography{references}

%%%%%%%%%%%%%%%%%%%%%%%%%%%%%%%%%%%%%%%%%%%%%%%%%%%%%%%%%%%%%%%%%%%%%%%%%%%%%%%%%%%%%%%%%%%%%%%%%%%%%%%%%%%%%%%%%%%%%%%%%%%%%%%

% If you have textual supplementary material
% \appendix
% \onecolumn
\appendix

\onecolumn
\aistatstitle{Appendix for the paper \\
                \say{A Sea of Words: An In-Depth Analysis of Anchors for Text Data}}
\paragraph{Organization of the Appendix.}
We start by providing all proofs of results presented in the paper in Section~\ref{sec:all-proofs}. 
In particular, in Section~\ref{sec:supp-normalized-tf-idf} we show that all our finding are valid for a normalized TF-IDF vectorization. 
Section~\ref{sec:technical-results} collects all required technical results. 
Finally, for further empirical validation of our findings, the interested reader will find additional experiments in Section~\ref{sec:additional-experiments}. 
\section{PROOFS}
\label{sec:all-proofs}
In this section, we collect all the proofs omitted from the main paper.
In Section~\ref{sec:proof-eq-sampling} we prove Proposition~\ref{proposition:equivalent-sampling}. 
Sections~\ref{sec:proof-prop-exhaustive-anchors-stability} and \ref{sec:proof-empirical-precision-concentration} refer to Section~\ref{sec:exhaustive-anchors}.
Sections~\ref{sec:proof-prop-dummy-features} to \ref{sec:proof-prop-approx-prec-maximization} contain the proofs for Section~\ref{sec:analysis}. 
In Section~\ref{supp:sec-simple-rules} we provide an additional result to those presented in Section~\ref{sec:simple-rules}. 
In Section~\ref{sec:supp-normalized-tf-idf} we prove that what was stated for TF-IDF vectorization remains valid for a normalized TF-IDF vectorization. 

In all proofs where there is only one tolerance threshold $\epsilon$ and where the selection function $p$ is clear from context, we write $\Anchors_k$ instead of $\Anchors_k^p(\epsilon)$ for $k\in [3]$. 

%%%%%%%%%%%%%%%%%%%%%%%%%%%%%%%%%%%%%%%%%%%%%%%%%%%%%%%%%%%%%%%%%%%%

\subsection{Proof of Proposition~\ref{proposition:equivalent-sampling}: Equivalent sampling}
\label{sec:proof-eq-sampling}
Let $i\in [n]$ be fixed and let $R\subseteq [b]$ be the (random) set of replaced indices for this specific perturbed example. 
Let us first compute the probability that a given word with index $k\in [b]\setminus A$ is removed:
\begin{align*}
\proba{k \in R} & = \proba{i \in R_k} \tag{definition of $R_k$} \\
& = \sum_{\ell=0}^n\condproba{i \in R_k}{B_k=\ell} \cdot \proba{B_k=\ell} & \tag{law of total probability} \\
& = \sum_{\ell=0}^n \binom{n-1}{\ell-1} \binom{n}{\ell}^{-1} \cdot \binom{n}{\ell} \frac{1}{2^n} = \frac{2^{n-1}}{2^n} \tag{uniform distribution  among all subsets} \\ \\
\proba{k \in R} &= \frac{1}{2} 
\, .
\end{align*}

Let us now show that the removals are independent.
First, we notice that the independence column-wise is verified by definition: for each word with index $k\in [b]$, the number $B_k$ of copies to perturb is drawn independently by construction. 
Next, for a given column $k\in [b]$, let us show that the removal from examples to examples are independent. 
Let $i,j\in [n]$, with $i\neq j$. 
We write
\begin{align*}
\proba{i,j \in R_k} &= \sum_{\ell = 0}^n \condproba{i,j\in R_k}{B_k=\ell}\proba{B_k=\ell} \tag{law of total probability} \\
&= \sum_{\ell=0}^n \binom{n-2}{\ell - 2}\binom{n}{\ell}^{-1}\binom{n}{\ell} \frac{1}{2^n} = \frac{2^{n-2}}{2^n}  \tag{uniform distribution among all subsets} \\
\proba{i,j \in R_k}  &= \frac{1}{4}
\, .
\end{align*}
According to the first part of the proof, this is exactly $\proba{i\in R_k}\cdot \proba{j\in R_k}$, and we can conclude. 
\qed 

%%%%%%%%%%%%%%%%%%%%%%%%%%%%%%%%%%%%%%%%%%%%%%%%%%%%%%%%%%%%%%%%%%%

%
\subsection{Proof of Proposition~\ref{prop:exhaustive-anchors-stability}: Stability of exhaustive \texorpdfstring{$p$}{p}-Anchors}
\label{sec:proof-prop-exhaustive-anchors-stability}
Recall that we set $\Astar = A^p(\epsilon)$. 
First, let us show that $\Anchors_1^q(\epsilon-\delta)\subseteq\Anchors_1^p(\epsilon)$. 
Let $A\in\Anchors_1^q(\epsilon-\delta)$. 
Using Eq.~\eqref{eq:epsilon-approx}, we have
\[
p(A) \geq q(A) - \delta \geq 1-(\epsilon-\delta) - \delta = 1-\epsilon
\, .
\]
Therefore, $A\in\Anchors_1^p(\epsilon)$. 
We directly deduce that $\Anchors_2^q(\epsilon-\delta)\subseteq\Anchors_2^p(\epsilon)$. 
We now show that $\Anchors_2^q(\epsilon-\delta)$ is non-empty, and more precisely that it contains $\Astar$. 
Indeed,
\[
q(\Astar) \geq p(\Astar) - \delta \geq 1 - \epsilon/4 - \delta \geq 1-(\epsilon-\delta) 
\, ,
\]
since $\delta < \epsilon/4 < 3\epsilon/8$. 
At this point, it suffices to show that $\Astar$ has (strict) maximal $q$ among $\Anchors_2^q(\epsilon-\delta)$. 
Let us pick any $A \in \Anchors_2^q(\epsilon-\delta)$ such that $A \neq \Astar$. 
Then, 
\begin{align*}
q(\Astar) - q(A) & = q(\Astar) - p(\Astar) + p(\Astar) - p(A) + p(A) - q(A) \\
& \geq -\delta + 1 - \epsilon/4 - (1 - 3\epsilon/4) -\delta \geq \epsilon/4 > 0
\, ,
\end{align*}
since $\delta < \epsilon/4$. 
Finally, there is no uniform random draw (last step of Algorithm~\ref{algo:exhaustive-anchors}), since $\Anchors_3^p(\epsilon) = \Anchors_3^q(\epsilon-\delta) = \{\Astar\}$.
\qed

%%%%%%%%%%%%%%%%%%%%%%%%%%%%%%%%%%%%%%%%%%%%%%%%%%%%%%%%%%%%%%%%%%

\subsection{Proof of Proposition~\ref{prop:empirical-precision-concentration}: \texorpdfstring{$\Empprec_n(A)$ uniformly approximates $\Precision$}{}}
\label{sec:proof-empirical-precision-concentration}
Let $A\in\Anchors$ be any anchor and let $x_1,\ldots,x_n$ be perturbed examples associated to this anchor. 
For all $i\in[n]$, the random variables $Y_i\defeq \indic{f(x_i)=1}\in\{0,1\}$ are independent and bounded by construction.
Since $\precision{A} = \expec{\Empprec_n(A)}$, we can apply  Hoeffding's inequality to the $Y_i$s \citep[Theorem~2.8]{boucheron2013concentration}. 
We obtain
\begin{equation} 
\label{eq:direct-hoeffding}
\proba{\abs{\Empprec_n(A) - \precision{A}} > \delta} =\proba{n\abs{\Empprec_n(A) - \precision{A}} > n\delta} \leq 2\exps{-2n\delta^2} 
\, .
\end{equation}
There are less than $2^b$ anchors (since we are not considering the empty anchor as a valid anchor), and we can conclude \emph{via} a union bound argument. 
\qed

%%%%%%%%%%%%%%%%%%%%%%%%%%%%%%%%%%%%%%%%%%%%%%%%%%%%%%%%%%%%%%%%%%

%
\subsection{Proof of Proposition~\ref{prop:dummy-features}: Dummy features}
\label{sec:proof-prop-dummy-features}
%
%Let us write $A$ short for $A^{\Precision}(\epsilon)$. 
Let $A\in\Anchors_3^{\Precision}(\epsilon)$. 
If $\anchor_j=0$, there is nothing to prove. 
Thus let us assume that $\anchor_j >0$ and come to a contradiction. 
Let us set $A^j$ the anchor identical to $A$ but with coordinate $j$ to zero. 
The precision of $A$ is given by 
\[
\precision{A} = \probaunder{g(\vectorizer{x}) \in R}{A} = \probaunder{g(\ldots,\tfidf{x}_j,\ldots)\in R}{A}
\, .
\]
According to the discussion preceding Proposition~\ref{prop:dummy-features}, $\tfidf{x}_j=\Mult_j\idf_j$, where $\Mult_j\sim \anchor_j + B_j$, with $B_j\sim\binomial{\mult_j-\anchor_j}{1/2}$ (this is Eq.~\eqref{eq:key-proof}). 
Since $g$ does not depend on coordinate $j$ and the sampling is independent, $g(\ldots,\idf_j(\anchor_j + B_j),\ldots)$ is equal in distribution to $g(\ldots,\idf_j \binomial{\mult_j}{1/2},\ldots)$. 
In particular, $\precision{A^j}=\precision{A}$. 
Since $\length{A^j} < \length{A}$, we can conclude. 
\qed 

%%%%%%%%%%%%%%%%%%%%%%%%%%%%%%%%%%%%%%%%%%%%%%%%%%%%%%%%%%%%%%%%%%

\subsection{Proof of Proposition~\ref{prop:product}: Presence of a set of words}
\label{sec:proof-prop-product}
Let $p_i\defeq 1-\frac{1}{2^{\mult_i}}$. % as in the proof of Proposition~\ref{prop:small-decision-tree}. 
Since the model only depends on the coordinates belonging to $J$, according to Proposition~\ref{prop:dummy-features}, we can restrict ourselves to anchors such that $\anchor_j \neq 0$ if $j\in J$. 
Let us start by computing the precision of any candidate anchor $A\in\Anchors$. 
We write
\begin{align*}
\precision{A} & = \expecunder{\indic{f(x)=f(\xi)}}{A} \\
& = \probaunder{f(x)=1}{A} \tag{since $f(\xi)=1$} \\
& = \probaunder{\word_j\in x,\,\forall j\in J}{A} \\
& = \prod_{j\in J} \probaunder{\word_j\in x}{A} \tag{by independence} \\
& = \prod_{j\in J} \probaunder{\idf_j \Mult_j > 0}{A}  \\
& = \prod_{j\in J} \probaunder{\anchor_j+B_j >0}{A} \\
\precision{A} & = \prod_{j\in J} p_j^{\indic{\anchor_j=0}} 
\,.
\end{align*}

Let us now apply Algorithm~\ref{algo:exhaustive-anchors} step by step in each of the cases outlined in the statement of the result. 

\paragraph{Case (I): $\displaystyle{\max_{j\in J}{\mult_j}}\leq B$.}
If for all $j\in J,\; \anchor_j>0$, then, according to the previous discussion, 
\[
\precision{A}=\prod_{j\in J}p_j^{\indic{\anchor_j=0}}=1
\, .
\]
Therefore, the anchor $(1,\ldots,1)$ belongs to $\Anchors_1$. 
If, instead, there exists $j\in J$ such that $\anchor_j=0$, then $\sum_{j\in J} \indic{\anchor_j=0} \geq 1$ and 
\begin{align*}
\precision{A} & = \prod_{j\in J}p_j^{\indic{\anchor_j=0}} \\
& \leq \prod_{j\in J} \left(1-\frac{1}{2^B}\right)^{\indic{\anchor_j=0}} \tag{since $\mult_j\leq B$ for all $j\in J$} \\
& = \left(1-\frac{1}{2^{B}}\right)^{\sum_{j\in J} \indic{\anchor_j=0}} \tag{since $\sum_{j\in J} \indic{\anchor_j=0} \geq 1$} \\
& \leq 1-\frac{1}{2^{B}} \\
\precision{A} & < 1 - \epsilon \tag{since $1-\frac{1}{2^{B}} \leq 1-\epsilon$ by definition of $B$}
\,.
\end{align*}
Thus $\Anchors_1$ consists of anchors having at least one occurrence of each word of $J$, and these anchors only. 
In $\Anchors_2$, Algorithm~\ref{algo:exhaustive-anchors} will select the anchor $A_J=\{\word_j,\; j\in J\}$, such that $\anchor_j=1$ for all $j\in J$ and $\anchor_j=0$ if $j\notin J$, which is the shortest anchor satisfying the precision condition. 
Since there are no equality cases, $\Anchors_3$ is a singleton and we can conclude. 

\paragraph{Case (II): $\displaystyle{\max_{j\in J}{\mult_j}} > B$.}
We first make two claims:

\begin{claim}
We can restrict our analysis to anchors $A\in\Anchors$ such that $\anchor_j\in\{0,1\}$ for $j\in J$.
\end{claim}            

Indeed, for any $j\in J$, the model $f$ is checking the presence of the word $\word_j$ in a document $\delta$, \emph{i.e.}, that $\tfidf{\delta}_j>0$, disregarding of its multiplicity.
As said before, the anchor $A_J$ (such that $\anchor_j=1$ for all $j\in J$) has precision $1$. 
Any other anchor $A=(\anchor_1,\ldots,\anchor_d)$
%$A_j\cup\{\word_i\}$, $i\in J$, 
such that $\anchor_i\geq 2$ has the same precision, but higher length. 
        
Now let us consider two indices $j$ and $j'$ in $J$ such that $j>j'$ (implying, $p_j<p_{j'}$) and an anchor $A$ such that $\anchor_j=\anchor_{j'}=0$. 
We set $A^j$ (resp. $A^{j'}$) the anchor identical to $A$ except coordinate $j$ (resp. $j'$) put to $1$. 
Let $I_A\defeq\{j\in [d],\, \anchor_j>0\}$. 
\begin{claim}
If $j>j'$, $\precision{A^j} > \precision{A^{j'}}$. 
\end{claim}

Indeed,
\begin{align*}
\precision{A^j} & = \prod_{\ell\in J,\: \anchor_\ell = 0} p_\ell \\
& = \prod_{\ell\in J\setminus(I_A\cup\{j\})} p_\ell \\
& = p_{j'}\cdot\prod_{\ell\in J\setminus(I_A\cup\{j\}\cup\{j'\})} p_\ell \\
& > p_{j}\cdot\prod_{\ell\in J\setminus(I_A\cup\{j\}\cup\{j'\})} p_\ell \tag{since $p_j<p_{j'}$} \\
\precision{A^j} & = \precision{A^{j'}} 
\, .
\end{align*}

As a consequence, for any anchor of fixed length, we can get higher precision by moving indices to the right. 
Therefore, the anchor $A_{c_0}$ will be selected by Algorithm~\ref{algo:exhaustive-anchors},  
see Figure \ref{fig:product_candidates} for an illustration. 
\qed

\begin{figure}
    \centering
    \begin{table}[H]
    \centering
    \begin{tabular}{c c c c c c c c c | c}
        $\anchor_1$ & $\anchor_2$ & $\ldots$ & $\anchor_{k-c_0}$ & $\anchor_{k-c_0+1}$ & $\anchor_{k-c_0+2}$ & $\ldots$ & $\anchor_{k-1}$ & $\anchor_k$ & $\precision{A} = \displaystyle \prod_{j\in J} p_j^{\indic{\anchor_j=0}} = \prod_{\ell=1}^k p_\ell^{\indic{\anchor_\ell=0}} $ \\
        \hline
        $0$ & $0$ & \ldots & $0$ & $0$ & $0$ & \ldots & $0$ & $0$ & $\displaystyle\prod_{j\in J} p_j < 1-\epsilon$ \\
        $0$ & $0$ & \ldots & $0$ & $0$ & $0$ & \ldots & $0$ & $1$ & $\displaystyle\prod_{\ell=1}^{k-1} p_\ell < 1-\epsilon$ \\
        $0$ & $0$ & \ldots & $0$ & $0$ & $0$ & \ldots & $1$ & $1$ & $\displaystyle\prod_{\ell=1}^{k-2} p_\ell  < 1-\epsilon$ \\
        \vdots & \vdots & \vdots & \vdots & \vdots & \vdots & \vdots & $1$ & $1$ & $ \precision{A} < 1-\epsilon$ \\
        $0$ & $0$ & \ldots & $0$ & $0$ & $1$ & \ldots & $1$ & $1$ & $\displaystyle\prod_{\ell=1}^{k-c_0+1} p_\ell < 1-\epsilon$ \\
        \rowcolor{anchors_color!50} $0$ & $0$ & \ldots & $0$ & $1$ & $1$ & \ldots & $1$ & $1$ & {$\displaystyle\prod_{\ell=1}^{k-c_0} p_\ell \geq 1-\epsilon$} \\
        $0$ & $0$ & \ldots & $1$ & $1$ & $1$ & \ldots & $1$ & $1$ & $\displaystyle\prod_{\ell=1}^{k-c_0-1} p_\ell > 1-\epsilon$ \\
        \vdots & \vdots & $1$ & $1$ & $1$ & $1$ & $1$ & $1$ & $1$ & $ \precision{A} > 1-\epsilon$ \\
        $1$ & $1$ & $1$ & $1$ & $1$ & $1$ & $1$ & $1$ & $1$ & $1 > 1-\epsilon$ \\
    \end{tabular}
\end{table}

    \caption{\label{fig:product_candidates}Illustration of Proposition~\ref{prop:product}. Anchors extraction for a model classifying words according to the presence or absence of words $\word_j$, $j\in J=\{1,\ldots,k\}$, ranked such that $\mult_1>\mult_2>\cdots>\mult_k$. 
    The anchor such that $\anchor_\ell=1$ for $\ell\in\llbracket k-c_0+1,k\rrbracket$ and $\anchor_\ell=0$ otherwise, is the minimal anchor satisfying the precision condition.}
\end{figure}

%%%%%%%%%%%%%%%%%%%%%%%%%%%%%%%%%%%%%%%%%%%%%%%%%%%%%%%%%%%%%%%%%%
\subsection{Proof of Proposition~\ref{prop:precision-logistic}: Precision of a linear classifier}
\label{sec:proof-prop-precision-logistic}
Let us set 
\[
Z_d \defeq \lambda_0 + \sum_{j=1}^d \lambda_j\idf_j \Mult_j
\, ,
\]
where $\Mult_j= \anchor_j + B_j$ are the random multiplicities, that is, $\anchor_j$ is the number of anchored words for $j$ and, as before, $B_j\sim\binomial{\mult_j-\anchor_j,1/2}$. 
In our notation, the problem of evaluating the precision of an anchor $A$ is that of evaluating accurately $\proba{Z_d > 0}$. 
From Section~\ref{sec:binomial-wonderland}, we see that, for all $j\in [d]$, 
\[
\expec{\Mult_j} = \frac{1}{2}(\mult_j+\anchor_j)\, ,\qquad \text{and}\qquad \var{\Mult_j} = \frac{1}{4}(\mult_j-\anchor_j)
\, .
\]
We deduce that 
\begin{equation}
\label{eq:aux-mean-zd}
\expec{Z_d} = \frac{1}{2} \sum_{j=1}^d \lambda_j\idf_j(\mult_j+\anchor_j)
\quad\text{and}\quad
\var{Z_d} = \frac{1}{4}\sum_{j=1}^{d} \lambda_j^2\idf_j^2 (\mult_j-\anchor_j)
\, .
\end{equation}

By the Berry-Esseen theorem for non-identically distributed version \citep{shevtsova_2010}, uniformly in $t\in\Reals$, it holds that
\begin{equation}
\label{eq:technical-berry-esseen-1}
\abs{\proba{Z_d \leq t} - \Phi\left(\frac{t-\expec{Z_d}}{\sqrt{\var{Z_d}}}\right)} \leq \cberryesseen\cdot \frac{\sum_{j=1}^{d}\expec{\abs{\lambda_j\idf_j \Mult_j - \expec{\lambda_j\idf_j \Mult_j}}^3}}{\var{Z_d}^{3/2}}
\, ,
\end{equation}
where $C\approx 7.15$ is a numerical constant. 
Setting $t=-\lambda_0$ in the previous display, we recognize $1-$ the definition of the precision. 
Using Eq.~\eqref{eq:aux-mean-zd}, we have obtained the left-hand side of Eq.~\eqref{eq:precision-logistic}. 
The numerator is upper bounded by first writing
\begin{align*}
\expec{\abs{\lambda_j\idf_j \Mult_j - \expec{\lambda_j\idf_j \Mult_j}}^3} &\leq \max_j \abs{\lambda_j\idf_j}^3 \cdot \expec{\abs{B_j-\expec{B_j}}^3} \tag{definition of $\Mult_j$} \\
&\leq \max_j \abs{\lambda_j\idf_j}^3  \cdot \frac{1}{\sqrt{8\pi}}\left(\frac{\mult_j-\anchor_j}{2}\right)^{3/2} \tag{Eq.~\eqref{eq:bound-third-absolute-moment-binomial}} \\
\expec{\abs{\lambda_j\idf_j \Mult_j - \expec{\lambda_j\idf_j \Mult_j}}^3} &\leq \frac{1}{\sqrt{\pi}} \cdot \max_j \abs{\lambda_j\idf_j}^3 \cdot (\max_j\mult_j)^{3/2}
\, .
\end{align*}
We deduce that 
\begin{equation}
\label{eq:technical-berry-esseen-2}
\sum_{j=1}^{d}\expec{\abs{\lambda_j\idf_j \Mult_j - \expec{\lambda_j\idf_j \Mult_j}}^3} \leq \frac{1}{\sqrt{\pi}} \cdot \left(\max_j \lambda_j^2\idf_j^2\right)^{3/2} \cdot (\max_j\mult_j)^{3/2} \cdot d
\, .
\end{equation}
Regarding the denominator, we have
\begin{align*}
\var{Z_d} &= \frac{1}{4}\sum_{j=1}^d \lambda_j^2\idf_j^2 (\mult_j-\anchor_j) \tag{Eq.~\eqref{eq:aux-mean-zd}} \\
&\geq \frac{1}{4} \cdot \left(\min_{j} \lambda_j^2\idf_j^2 \right) \cdot (b-\length{A}) \tag{definition of $b$ and $\length{A}$} \\
\var{Z_d} &\geq \frac{1}{8} \cdot \left(\min_{j} \lambda_j^2\idf_j^2 \right) \cdot \min_j\mult_j \cdot d
\, .
\end{align*}
Thus
\begin{equation}
\label{eq:technical-berry-esseen-3}
\var{Z_d}^{3/2} \geq \frac{1}{16\sqrt{2}} \cdot \left(\min_{j} \lambda_j^2\idf_j^2 \right)^{3/2} \cdot \left(\min_{j} \right)^{3/2} \cdot d^{3/2}
\, .
\end{equation}
In particular, this is a positive quantity. 
Coming back to Eq.~\eqref{eq:technical-berry-esseen-1}, we see that
\[
\abs{\proba{Z_d \leq t} - \Phi\left(\frac{t-\expec{Z_d}}{\sqrt{\var{Z_d}}}\right)} \leq 
\frac{16\cberryesseen \sqrt{2}}{\sqrt{\pi}} \cdot \left(\frac{\max_j \lambda_j^2\idf_j^2}{\min_j \lambda_j^2\idf_j^2}\right)^{3/2} \cdot \left(\frac{\max_j\mult_j}{\min_j\mult_j}\right)^{3/2} \cdot \frac{1}{\sqrt{d}}
\, ,
\]
as announced. 
\qed 

%%%%%%%%%%%%%%%%%%%%%%%%%%%%%%%%%%%%%%%%%%%%%%%%%%%%%%%%%%%%%%%%%%%%

\subsection{Proof of Proposition~\ref{prop:approx-prec-maximization}: Approximate precision maximization}
\label{sec:proof-prop-approx-prec-maximization}
We start by proving two lemmas. 
The first shows that we can restrict ourselves to $\Anchors_+$ when considering the minimization of $\Approxprec$. 

\begin{lemma}[Restriction to positive anchors]
\label{lemma:L-restriction}
Let $A\in\Anchors$ be such that $\anchor_j >0$ whereas $\lambda_j <0$. 
Then
\[
\approxprec{\ldots,0,\ldots} < \approxprec{\ldots,\anchor_j,\ldots} 
\, .
\]
\end{lemma}

\begin{proof}
Keeping in mind that $\lambda_j >0$, we notice that $-\lambda_j\idf_j(\mult_j+\anchor_j) > -\lambda_j\idf_j \mult_j$, and that $\lambda_j^2\idf_j^2(\mult_j-\anchor_j) < \lambda_j^2\idf_j^2\mult_j$. 
In other terms, removing the word $\word_j$ from the anchor $A$ both decreases the numerator and increases the numerator of $L$. 
\end{proof}

The second shows that the minimization of $L$ on $\Anchors_+$ is straightforward, modulo a technical assumption on the size of the intercept. 

\begin{lemma}[Minimization of $\Approxprec$]
\label{lemma:L-minimization-new-version}
Assume that $\lambda_0> -\gamma / 2$. 
Assume further that the indices of the local dictionary are ordered such that the $\lambda_j\idf_j$s are strictly decreasing. 
Then, for any $k < \ell$ such that $\lambda_k,\lambda_\ell > 0$,
\begin{equation} 
\label{eq:L-minimization-csq-1-new}
\approxprec{\anchor_1,\ldots,\anchor_k + 1,\ldots,\anchor_\ell,\ldots,\anchor_d} 
< \approxprec{\anchor_1,\ldots,\anchor_k,\ldots,\anchor_\ell + 1,\ldots,\anchor_d}
\, .
\end{equation}
\end{lemma}

\begin{proof}
First, since $A\in\Anchors_+$ and $\lambda_0\geq -\gamma / 2$, we deduce that 
\begin{equation}
\label{eq:L-technical-assumption}
\lambda_0 + \frac{1}{2}\sum_j \lambda_j\idf_j\mult_j + \frac{1}{2}\sum_j \lambda_j \idf_j \anchor_j \geq 0
\, .
\end{equation}

Now, we will prove both inequalities by a function study. 
For any $j\in [d]$, let us set $\alpha_j\defeq \lambda_j\idf_j$ . 
We also set $\Omega_1\defeq \sum_j \alpha_j(\mult_j-\anchor_j)$, and $\Omega_2\defeq \sum_j \alpha_j^2(\mult_j - \anchor_j)$. 
Further, for any $a,b\in\Reals$, let us define the mapping 
\[
f_{a,b}(t) \defeq \frac{a+t}{\sqrt{b-t^2}}
\, .
\]
With this notation in hand, Eq.~\eqref{eq:L-minimization-csq-1-new} becomes
\begin{equation*} 
\frac{\gamma - \frac{1}{2}(\Omega_1 - \alpha_k)}{\sqrt{\Omega_2 - \alpha_k^2}} >\frac{\gamma - \frac{1}{2}(\Omega_1 - \alpha_\ell)}{\sqrt{\Omega_2 - \alpha_\ell^2}}
\end{equation*}
which is simply
\[
f_{2\gamma-\Omega_1,\Omega_2}(\alpha_k) > f_{2\gamma-\Omega_1,\Omega_2}(\alpha_\ell)
\, .
\]
Observe that, by our assumptions, $2\gamma - \Omega_1 > 0$, and $\Omega_2 > 0$.
It is straightforward to show that $f_{a,b}'(t)=\frac{at+b}{(b-t^2)^{3/2}}$, and therefore $f_{2\gamma-\Omega_1,\Omega_2}(\alpha_k)$ is a \emph{strictly increasing} mapping on $[0,\sqrt{b}]$.  
Since $\sqrt{\Omega_2} \geq \alpha_k > \alpha_\ell > 0$, we can conclude. 
\end{proof}

\paragraph{Proof of Proposition~\ref{prop:approx-prec-maximization}.}
In this proof, we set $p=\Phibar\circ \Approxprec$. 
Let $A^F$ be the anchor containing all words of $\xi$. 
In our notation, 
\[
A^F = (\mult_1,\ldots,\mult_d)
\, .
\]
We first notice that $\Anchors_1^p$ is non-empty since $\phibar{\approxprec{A^F}}=1$. 
By construction, $\Anchors_2^p$, consisting of anchors of $\Anchors_1^p$ of minimal length, is non-empty as well. 
Lemma~\ref{lemma:L-restriction} ensures that $\Anchors_3^p$, the anchors of $\Anchors_2^p$ with the highest $p$ value, is a non-empty subset of $\Anchors_+$.
Indeed, one can remove the anchors corresponding to the indices $j\in [d]$ such that $\lambda_j\idf_j < 0$, and increase the value of $p$. 
Since we assumed that at least one $\lambda_j$ is positive, it is always possible to do this removal. 
Finally, let $\ell$ be the common length of the anchors belonging to $\Anchors_3^p$. 
Since we satisfy the assumptions of Lemma~\ref{lemma:L-minimization-new-version}, we see that the $p$ value of any anchor of length $\ell$ is \emph{strictly} increasing if we swap indices towards the lower indices. 
We deduce the result. 
\qed 

%%%%%%%%%%%%%%%%%%%%%%%%%%%%%%%%%%%%%%%%%%%%%%%%%%%%%%%%%%%%%%%%%%%%%%%%%%%%%%%%%%%%%
\subsection{Additional result for Section~\ref{sec:simple-rules}: Simple if-then rules}
\label{supp:sec-simple-rules}
We present here a further result on a simple if-then classifier based on the presence of disjoint subsets of words. 
\begin{proposition}[Small decision tree]
\label{prop:small-decision-tree}
Let the (binary) classifier $f$ be defined as follows:
\begin{align*}
f(\doc) & = \indic{(\word_1 \in \doc \; \text{and} \; \word_2 \in \doc) \; \text{or} \; \word_3 \in \doc} 
\end{align*}
Then, for any $\epsilon >0$, the anchor $A = (0,0,1)$, will be selected by exhaustive Anchors. 
\end{proposition}

For example, consider the sentiment analysis task, and $f$ the model returning $1$ (a positive prediction) if words \say{not} \emph{and} \say{bad} \emph{or} the word \say{good} are present in the document. 
Proposition~\ref{prop:small-decision-tree} implies that only the word \say{good} will be selected as an anchor. 
This is a satisfying property and corresponds to the intuition that we have from Anchors:  in this class of examples, \textbf{the smallest rule is provably selected by Anchors.}  
We prove Proposition~\ref{prop:small-decision-tree} in Section~\ref{sec:proof-prop-small-decision-tree} of the Appendix. 

Of course, the scope of Proposition~\ref{prop:small-decision-tree} is limited. 
It is possible to obtain similar results for other simple sets of rules, though challenging to present these results with a sufficient amount of generality. 
\paragraph{Proof of Proposition~\ref{prop:small-decision-tree}}
\label{sec:proof-prop-small-decision-tree}

Let us start by computing $\precision{A}$ for any candidate anchor $A\in\Anchors$. 
Since the model only depends on the first three coordinates, according to Proposition~\ref{prop:dummy-features}, we can restrict ourselves to  $A=(\anchor_1,\anchor_2,\anchor_3)$. 
In this proof, we set $p_i\defeq 1-\frac{1}{2^{\mult_i}}$ the probability of keeping the word $\word_i$ while sampling and $B_k\sim \binomial{\mult_j-\anchor_j}{1/2}$. 
The precision of a candidate anchor $A\in\Anchors$ (Eq.~\eqref{eq:precision}) associated to $f$ is 
\begin{align*}
\precision{A} & = \expecunder{\indic{f(x)=f(\xi)}}{A} \\
& = \probaunder{f(x)=1}{A} \\
& = \probaunder{\tfidf{x}_1>0}{A} \cdot \probaunder{\tfidf{x}_2>0}{A} \cdot (1-\probaunder{\tfidf{x}_3>0}{A}) + \probaunder{\tfidf{x}_3>0}{A} \tag{by independence} \\
& = \probaunder{\Mult_1>0}{A} \cdot \probaunder{\Mult_2>0}{A} \cdot (1-\probaunder{\Mult_3>0}{A}) + \probaunder{\Mult_3>0}{A} \tag{since $\idf_j>0$ for all $j\in [d]$} \\
& = \probaunder{\anchor_1+B_1>0}{A} \cdot \probaunder{\anchor_2+B_2>0}{A} \cdot \left(1-\probaunder{\anchor_3+B_3>0}{A}\right) + \probaunder{\anchor_3+B_3>0}{A} \\ 
& = \begin{cases}
1 \,, &\text{if}\: \anchor_3>0 \,, \\
1 \,, &\text{if}\: \anchor_1, \anchor_2 > 0 \:\text{and}\: \anchor_3=0 \,, \\
p_2(1-p_3)+p_3 =  1 - \frac{1}{2^{\mult_2+\mult_3}} \,, &\text{if}\: \anchor_1 > 0 \:\text{and}\:  \anchor_2, \anchor_3=0 \,, \\
p_1(1-p_3)+p_3 = 1 - \frac{1}{2^{\mult_1+\mult_3}} \,, &\text{if}\: \anchor_2 > 0 \:\text{and}\:  \anchor_1, \anchor_3=0 \,, \\
p_1p_2(1-p_3)+p_3 =  1 - \frac{2^{\mult_1} + 2^{\mult_2} - 1}{2^{\mult_1 + \mult_2 + \mult_3}} \,, &\text{if}\: \anchor_1 = \anchor_2 = \anchor_3 = 0 \,. \\
\end{cases}
\end{align*}

Now let us follow Algorithm~\ref{algo:exhaustive-anchors} step by step.
According to the previous discussion, any anchor 
%containing $\{\word_1,\word_2\}$ 
such that $\anchor_1$ and $\anchor_2$ are positive 
and/or $\anchor_3 > 0$ has precision $1$, and thus belongs to $\Anchors_1$.
In particular, the anchor $(0,0,1)$ belongs to $\Anchors_1$. 
We note that it also has minimal length $1$, and therefore belongs to $\Anchors_2$.
Finally, any other anchor with the same length will have a smaller precision, since $p_2(1-p_3)+p_3<1$, $p_1(1-p_3)+p_3<1$, and $p_1p_2(1-p_3)+p_3<1$.
In conclusion, $\Anchors_3$ is reduced to a singleton and the anchor $A=(0,0,1)$ will be selected by exhaustive Anchors. 
\qed

%%%%%%%%%%%%%%%%%%%%%%%%%%%%%%%%%%%%%%%%%%%%%%%%%%%%%%%%%%%%%%%%%%%%%%%%%%%%%%%%%%%%%%%%%%%%

\subsection{Normalized TF-IDF}
\label{sec:supp-normalized-tf-idf}
In this section we show that our theoretical results demonstrated considering a TF-IDF vectorization as defined in Definition~\ref{def:tf-idf} still hold for the $\ell_2$-normalized TF-IDF vectorization, defined as
\[
\forall j \in [D], \qquad \phi(\doc)_j \defeq \frac{\mult_j(\doc) \idf_j}{\sqrt{\sum_{j=1}^D \mult_j(\doc)^2 \idf_j^2}} 
\, ,
\]
that is, the default normalization in the \texttt{scikit-learn} implementation of TF-IDF. 
The main result of this section is the following:

\begin{proposition}[Normalized-TF-IDF, Berry-Esseen]
\label{prop:normalized-tf-idf-berry-esseen}
Assume that $0 < \vmin \leq \idf_j \leq \vmax$ and $\minmult \leq \mult_j\leq \maxmult$ for all $j\in [d]$. 
Assume further that $A$ is not the empty anchor, that $\length{A}\leq b/2$ and that $a_j < m_j$ for all $j$. 
Finally, assume that $\norm{\lambda}=1$ as $d\to +\infty$. 
For all $t\in\Reals$, define
\[
P(t) \defeq \Phi\left(\frac{t\sqrt{\sum_j \{ (\mult_j+\anchor_j)^2 + \mult_j-\anchor_j\}\idf_j^2 } - \sum_j \lambda_j(\mult_j+\anchor_j)\idf_j}{\sqrt{\sum_j \lambda_j^2(\mult_j-\anchor_j)\idf_j^2}}\right)
\, .
\]
Then 
\begin{equation} 
\label{eq:normalized-tf-idf-berry-esseen}
\abs{\proba{\lambda^\top \normtfidf{x} \leq t} - P(t)} \leq \frac{2 \sqrt{M} \vmin^{-3/2}}{d^{3/4-\epsilon}} + 2 \expl{-\frac{d^{1+2\varepsilon }\vmin^6}{4\vmax^4 \Mult^4}}
\, .
\end{equation}
% where $C = \frac{2 c \sqrt{M}}{\vmin^{3/2}}$. 
\end{proposition}

As a direct consequence of Proposition~\ref{prop:normalized-tf-idf-berry-esseen}, we know that 
a good approximation of $\precision{A}$ in the normalized TF-IDF case is $\phibar{L(A)}$, with 
\begin{equation}
\label{eq:approx-normalized-tf-idf}
L(A) = \frac{-\lambda_0\sqrt{\sum_j \{ (\mult_j+\anchor_j)^2 + \mult_j-\anchor_j\}\idf_j^2 } - \sum_j \lambda_j(\mult_j+\anchor_j)\idf_j}{\sqrt{\sum_j \lambda_j^2(\mult_j-\anchor_j)\idf_j^2}}
\, .
\end{equation}
This is reminiscent of Eq.~\eqref{eq:precision-logistic} in the non-normalized case. 
When $\lambda_0=0$, the analysis of the maximization problem is a subcase of the non-normalized case, and we recover the same result. 
Although $\lambda_0=0$ can be a reasonable assumption (assuming centered data and no intercept), we conjecture that the result is true for a larger range of $\lambda_0$, similarly to the unnormalized case. 

Let us now prove Proposition~\ref{prop:normalized-tf-idf-berry-esseen}. 
Let us set
\[
Z_d \defeq \frac{\sum_{j=1}^d \lambda_j\idf_j \Mult_j}{\sqrt{\sum_{j=1}^d \idf_j^2 \Mult_j^2}}
\qquad \text{and} \qquad 
\Ztilde_d \defeq \frac{\sum_{j=1}^d \lambda_j\idf_j \Mult_j}{\sqrt{\frac{1}{4} \sum_{j=1}^d \{(\mult_j+\anchor_j)^2 +\mult_j-\anchor_j\}} }
\, .
\]
Intuitively, when $d$ is large enough, both these quantities are close with high probability, and $\Ztilde_d$ has the same structure as the linear form studied in the normalized case, up to a constant. 
Thus, the analysis boils down to the previous case, modulo the following:

\begin{proposition}[$Z_d$ and $\Ztilde_d$ are close with high probability]
\label{prop:cv-proba-normalized-tf-idf}
Let $\epsilon \in (0,1/2)$. 
Assume that $0 < \vmin \leq \idf_j \leq \vmax$ and $\minmult \leq \mult_j\leq \maxmult$ for all $j\in [d]$. 
Assume further that $A$ is not the empty anchor. 
Finally, assume that $\norm{\lambda}=1$ as $d\to +\infty$. 
Then
\[
\proba{\abs{Z_d - \Ztilde_d} > \frac{c}{d^{1/2-\epsilon}}} 
%\leq 2 \expl{-\frac{d^2 s^2 \vmin^6}{4 \vmax^4 \Mult^4}} 
\leq 2 \expl{-\frac{c^2 d^{1+2\varepsilon }\vmin^6}{4\vmax^4 \Mult^4}} 
\, ,
\]
for any small positive constant $c$. 
%where we used $s=\frac{c}{d^{1/2-\epsilon}}$, with $\epsilon>0$.
\end{proposition}

% The assumption over the $\lambda$ coefficients is satisfied, for instance, whenever $\lambda$ is inferred from a sparse logistic regression. 
% In that case, the coefficients have small norm, regardless of $d$. 
%We now turn to the proof of Proposition~\ref{prop:cv-proba-normalized-tf-idf}. 

\paragraph{Proof of Proposition~\ref{prop:cv-proba-normalized-tf-idf}. }
In this proof, we write $\denom\defeq \sum_{j=1}^d \idf_j^2 \Mult_j^2$. 
We begin by computing the expectation of $\denom$. 
We know that $\Mult_j = \anchor_j + B_j$, where $B_j \sim \binomial{\mult_j-\anchor_j}{1/2}$. 
Therefore, 
\begin{align*}
\expec{\Mult_j^2} &= \expec{\anchor_j^2 + 2\anchor_j B_j + B_j^2} \\
&= \anchor_j^2 + \anchor_j (\mult_j - \anchor_j) + \frac{1}{4}(\mult_j-\anchor_j)^2 + \frac{1}{4}(\mult_j-\anchor_j) \\
\expec{\Mult_j^2} &= \frac{1}{4} (\mult_j+\anchor_j)^2 + \frac{1}{4}(\mult_j-\anchor_j)
\, ,
\end{align*}
where we used Lemma~\ref{lemma:binomial-moments} to compute $\expec{B_j^2}$. 
By linearity, we deduce that 
\[
\expec{\denom} = \frac{1}{4} \sum_{j=1}^d \{(\mult_j+\anchor_j)^2 +\mult_j-\anchor_j\}
\, .
\]
Note that, with this notation in hand, 
\[
Z_d = \frac{\sum_{j=1}^d \lambda_j\Mult_j\idf_j}{\sqrt{\denom}}
\qquad\text{ and }\qquad 
\Ztilde_d = \frac{\sum_{j=1}^d \lambda_j\Mult_j\idf_j}{\sqrt{\expec{\denom}}}
\, .
\]
We need to prove the following, which shows that $\denom$ is concentrated around its expectation:

\begin{lemma}[Concentration of $\denom$]
\label{lemma:hoeffding-denominator}
Assume that $0 < \vmin \leq \idf_j\leq \vmax$  and that $\minmult \leq \mult_j\leq \maxmult$ for all $j\in [d]$. 
Then, for all $t > 0$, 
\[
\proba{\abs{\denom - \expec{\denom}} > t} 
\leq 2\expl{\frac{-2t^2}{d\vmax^4\maxmult^4}} 
\, .
\]
\end{lemma}

\begin{proof}
This is a straightforward application of Hoeffding's inequality once we notice that the random variables $\idf_j^2\Mult_j^2$ are bounded and independent, and that $\sum_j \mult_j^4\idf_j^4 \leq d\vmax^4\maxmult^4$ under our assumptions. 
\end{proof}

Note that Lemma~\ref{lemma:hoeffding-denominator} is tight, since Hoeffding's inequality is tight for Bernoulli random variables, a case which is possible under our assumption. 
Lemma~\ref{lemma:hoeffding-denominator} allows controlling the small deviations of $\denom$, a fact that we will maybe not use in the following, but can nonetheless be useful to split a complicated event.
Next, we control the size of $\denom$. 

\begin{lemma}[$\denom$ is small with high probability]
\label{lemma:control-denominator-small}
Assume that $0 < \vmin \leq \idf_j\leq \vmax$ for all $j\in [d]$.  
Then
\[
\proba{\denom < \frac{1}{4}d\vmin^2 }
\leq  2\expl{\frac{-d\vmin^4}{8\vmax^4\maxmult^4}}
\, .
\]
\end{lemma}

\begin{proof}
We write
\begin{align*}
\proba{\denom < \frac{1}{4}d\vmin^2 } &= \proba{\denom - \expec{\denom} < \frac{1}{4}d\vmin^2 - \expec{\denom}} \\
&\leq \proba{ \abs{\denom - \expec{\denom}} <  \frac{1}{2}d\vmin^2 -  \frac{1}{4}d\vmin^2}
\, ,
\end{align*}
since $\frac{1}{4}\{(\mult_j+\anchor_j)^2 + \mult_j - \anchor_j\} \geq \frac{1}{4} \cdot (1^2+1)$ for all $j\in [d]$. 
We conclude by applying Lemma~\ref{lemma:hoeffding-denominator} with $t=\frac{1}{4}d\vmin^2$. 
\end{proof}

Now we can control the key quantity:

\begin{lemma}[Control of the key quantity]
\label{lemma:control-key-quantity}
Assume that $0 < \vmin \leq \idf_j\leq \vmax$ and $\minmult \leq \mult_j\leq \maxmult$ for all $j\in [d]$. 
Assume further that $A$ is not the empty anchor. 
% DGA: I don't remember if we already assume this in the definition of anchors
%Assume further that $\sum_j \anchor_j^2\idf_j^2 \geq 1$. 
Then, for any $t>0$, 
\[
\proba{\abs{1 - \frac{\sqrt{\denom}}{\sqrt{\expec{\denom}}}} > t}
\leq 2\expl{ \frac{-d^2t^2\vmin^6}{4\vmax^4\maxmult^4}} 
\, .
\]
\end{lemma}

In particular, by taking $t$ of the order $\frac{1}{d^{1/2-\epsilon}}$ for some $\epsilon >0$, we see that $\proba{\abs{1-\sqrt{\denom / \expec{\denom}}} > t} = \littleo{1}$. 

\begin{proof}
Multiplying by $\expec{\denom}$, we see that we want to control
\[
\proba{\abs{\sqrt{\denom} - \sqrt{\expec{\denom}}} > t\sqrt{\expec{\denom}}}
\, .
\]
Since $\expec{\denom}\leq \frac{1}{2}d\vmin^2$, we can simply control 
\[
\proba{\abs{\sqrt{\denom} - \sqrt{\expec{\denom}}} > \frac{1}{2}dt\vmin^2 }
\, .
\]
Additionally, since $A$ is not the empty anchor, $\denom \geq \vmin^2$ almost surely, which is positive. 
Since the mapping $x\mapsto \sqrt{x}$ is $1/2\sqrt{C}$-Lipschitz on $[C,+\infty)$, we see that 
\[
\abs{\sqrt{\denom} - \sqrt{\expec{\denom}}} \leq \frac{1}{2\vmin} \abs{\denom - \expec{\denom}}
\, ,
\]
which allows us to focus on 
\[
\proba{\abs{\denom - \expec{\denom}} > dt\vmin^3}
\, . 
\]
We control this last display using Lemma~\ref{lemma:hoeffding-denominator}, and we obtain
\[
\proba{ \abs{1 - \frac{\sqrt{\denom}}{\sqrt{\expec{\denom}}}} > t } \leq 2\expl{ \frac{-d^2t^2\vmin^6}{4\vmax^4\maxmult^4}}
\, ,
\]
as promised. 
\end{proof}

Finally, coming back to the original problem, we write 
\begin{equation} 
\label{eq:decomposition-diff-zd}
\proba{\abs{Z_d - \Ztilde_d} > s} = \proba{\abs{Z_d}\cdot \abs{1 - \frac{\sqrt{\denom}}{\sqrt{\expec{\denom}}}} > s} \leq \proba{\abs{1 - \frac{\sqrt{\denom}}{\sqrt{\expec{\denom}}}} > s}
\, .
\end{equation}
By Cauchy-Schwarz inequality, we have 
\[
\abs{\sum_j \lambda_j\Mult_j\idf_j} \leq \norm{\lambda} \cdot \sqrt{D}
\, ,
\]
and we deduce that $\abs{Z_d}\leq \norm{\lambda}=1$ under our assumptions. 
Coming back to Eq.~\eqref{eq:decomposition-diff-zd}, we can therefore take $s = \frac{1}{d^{1/2-\epsilon}}$ and use Lemma~\ref{lemma:control-key-quantity} to conclude. 

\qed

We can now conclude this section with the proof of our main result.

\paragraph{Proof of Proposition~\ref{prop:normalized-tf-idf-berry-esseen}.}
Let us set $\num \defeq \sum_j \lambda_j\Mult_j\idf_j$. 
With this notation, $\proba{\lambda^\top \normtfidf{x} \leq  t} =\proba{Z_d \leq t}$, and 
\[
\lambda^\top \normtfidf{x} = \frac{\num}{\sqrt{\denom}} = Z_d
\quad \text{ and }\qquad 
\Ztilde_d = \frac{N}{\sqrt{\expec{D}}}
\, .
\]
Let $s >0$. 
Using Lemma~\ref{lemma:cv-proba-implies-cv-dist}, we have 
\[
\proba{\Ztilde_d \leq t-s} - \proba{\abs{Z_d-\Ztilde_d} > s} \leq \proba{Z_d \leq t} \leq  \proba{\Ztilde_d \leq t+s} + \proba{\abs{Z_d - \Ztilde_d} > s}
\, .
\]
Let us set $s = \frac{1}{d^{1/2-\epsilon}}$ for some small $\epsilon >0$. 
By Proposition~\ref{prop:cv-proba-normalized-tf-idf}, we know that $\proba{\abs{Z_d - \Ztilde_d} > s} \leq 2 \expl{-\frac{d^{1+2\varepsilon} \vmin^6}{4\vmax^4 \Mult^4}}$. 
Let us now turn towards the remaining terms, depending on $\Ztilde$. 
We write, for any $u\in\Reals$, 
\begin{align*}
\proba{\Ztilde_d \leq u} &= \proba{\frac{\num}{\sqrt{\expec{\denom}}} \leq u} \\
&= \proba{\frac{\num - \expec{\num}}{\sqrt{\var{\num}}} \leq \frac{u\sqrt{\expec{\denom}} - \expec{\num}}{\sqrt{\var{\num}}}} \\
\proba{\Ztilde_d \leq u} &= \Phi\left( \frac{u\sqrt{\expec{\denom}} - \expec{\num}}{\sqrt{\var{\num}}} \right) + \littleo{1}
\, ,
\end{align*}
uniformly in $u$, 
where we used Proposition~\ref{prop:precision-logistic} in the last derivation. 

Since $\Phi$ is $1$-Lipschitz, we see that 
\[
\abs{\proba{\Ztilde_d \leq t+s} - \Phi\left( \frac{t\sqrt{\expec{\denom}} - \expec{\num}}{\sqrt{\var{\num}}} \right)} \leq s\frac{\sqrt{\expec{\denom}}}{\sqrt{\var{\num}}} + 2 \expl{-\frac{d^{1+2\varepsilon} \vmin^6}{4\vmax^4 \Mult^4}}
\, .
\]
Moreover, under our assumptions, $\sqrt{\expec{\denom} / \var{\num}} = \bigo{1}$ and $s=\frac{1}{d^{1/2-\epsilon}}$. 
Therefore, 
\[
\proba{\Ztilde_d \leq t+s} = \Phi\left( \frac{t\sqrt{\expec{\denom}} - \expec{\num}}{\sqrt{\var{\num}}} \right) + 
\frac{1}{d^{1/2-\epsilon}} \frac{\sqrt{\expec{\denom}}}{\sqrt{\var{\num}}} +
2 \expl{-\frac{d^{1+2\varepsilon} \vmin^6}{4\vmax^4 \Mult^4}}
\, .
\]

Additionally, one can show that $\var{\Mult_j^2}=\nu(\mult_j-\anchor_j)$, with 
\[
\nu(x) \defeq \frac{x}{8} (2(2a+x)^2+x-1)
\, .
\]
% \begin{align*}
% \var{\Mult_j^2} & = \expec{\Mult_j^4} - \expec{\Mult_j^2}^2 = \expec{\left(\anchor_j + \binomial{\mult_j-\anchor_j}{1/2}\right)^4} - \expec{\Mult_j^2}^2 \\
% & = \anchor_j^4 + 4\anchor_j^3\expec{B} + 6\anchor_j^2\expec{B^2} +  4\anchor_j\expec{B^3} + \expec{B^4} - \expec{\Mult_j^2}^2  \\
% & = \anchor_j^4 + 4\anchor_j^3\frac{\mult_j-\anchor_j}{2} + 6\anchor_j^2\left(\frac{(\mult_j-\anchor_j)^2}{4} + \frac{\mult_j - \anchor_j}{4}\right) +  4\anchor_j\left(\frac{(\mult_j-\anchor_j)^3}{8} + \frac{3(\mult_j-\anchor_j)^2}{8}\right) \\
% &\quad + \frac{(\mult_j-\anchor_j)^4}{16} + \frac{3(\mult_j-\anchor_j)^3}{8} + \frac{3(\mult_j-\anchor_j)^2}{16} - \frac{\mult_j-\anchor_j}{8} - \left(\frac{(\mult_j+\anchor_j)^2}{4} + \frac{\mult_j-\anchor_j}{4}\right)^2  \\
% & = \frac{1}{8}(-2\anchor_j^3+\anchor_j^2(1-2\mult_j)+\anchor_j(2\mult_j^2-2\mult_j+1)+\mult_j(2\mult_j^2+\mult_j-1)) \\
% f(\mult_j-\anchor_j) & \geq f(1) = \anchor_j^2 + \anchor_j + \frac{1}{4} \geq \frac{1}{4}
% \,. 
% \end{align*}
% since $\exists j \text{ s.t. } \mult_j > \anchor_j$
% beacuse |A| <= b/2 
It is straightforward to show that $\nu$ is non-decreasing. 
When $\mult_j > \anchor_j$, we see that 
\[
\nu(\mult_j-\anchor_j) \geq \nu(1) = \anchor_j^2 + \anchor_j + \frac{1}{4} \geq \frac{1}{4}
\, .
\]
Therefore $\sum_j \idf_j^4 \var{\Mult_j^2}\geq \frac{d}{4}\vmin^4$. 
Under our assumptions, $\norm{\lambda}=1$, and by applying Cauchy-Schwarz inequality  
%\[
$\expec{\denom} \leq \sqrt{d}\vmax \maxmult
\, ,
$
%\]
we deduce
\[
\frac{\expec{\denom}}{\var{\num}} \leq \frac{4 \vmax M}{\sqrt{d} \vmin^4}
\, .
\]

% \begin{align*}
% \frac{\expec{\denom}}{\var{\num}} & = 
% \frac{\sum_{j=1}^d \lambda_j \idf_j \expec{\Mult_j}}{\sum_{j=1}^d \idf_j^4 \var{\Mult_j^2}} 
% % \frac{\frac{1}{2}\sum_{j=1}^d \lambda_j \idf_j (\mult_j + \anchor_j)}{\frac{1}{4}\sum_{j=1}^d \idf_j^4 (\mult_j - \anchor_j)} =
% % 2 \frac{\sum_{j=1}^d \lambda_j \idf_j (\mult_j + \anchor_j)}{\sum_{j=1}^d \idf_j^4 (\mult_j - \anchor_j)} 
% \leq \frac{\norm{\lambda} \cdot \sqrt{\sum_{j=1}^d \idf_j^2 \expec{\Mult_j}^2}}{\sum_{j=1}^d \idf_j^4 \var{\Mult_j^2}} 
% = \frac{\sqrt{\sum_{j=1}^d \idf_j^2 \frac{1}{4}(\anchor_j+\mult_j)^2}}{\sum_{j=1}^d \idf_j^4 \var{\Mult_j^2}} \\
% & \leq \frac{\sqrt{d}\vmax\Mult}{\sum_{j=1}^d \idf_j^4 \var{\Mult_j^2}}
% \leq \frac{\sqrt{d}\vmax\Mult}{\sum_{j=1}^d \idf_j^4 \left(\anchor_j^2 + \anchor_j + \frac{1}{4}\right)} 
% \leq \frac{4 \vmax M}{\sqrt{d} \vmin^4}
%  \, ,     
% \end{align*}

Thus, we find that 

\begin{align*}
\proba{\Ztilde_d \leq t+s} & \leq \Phi\left( \frac{t\sqrt{\expec{\denom}} - \expec{\num}}{\sqrt{\var{\num}}} \right) + 
\frac{1}{d^{1/2-\epsilon}}\frac{2 \sqrt{M}}{d^{1/4} \vmin^{3/2}} +
2 \expl{-\frac{d^{1+2\varepsilon} \vmin^6}{4\vmax^4 \Mult^4}} \\
& = \Phi\left( \frac{t\sqrt{\expec{\denom}} - \expec{\num}}{\sqrt{\var{\num}}} \right) + 
\frac{2 \sqrt{M}}{d^{3/4-\epsilon} \vmin^{3/2}} +
2 \expl{-\frac{d^{1+2\varepsilon} \vmin^6}{4\vmax^4 \Mult^4}}
\, .
\end{align*}

The same reasoning applies to $t-s$, and we can conclude by recognizing $\Phi\left( \frac{t\sqrt{\expec{\denom}} - \expec{\num}}{\sqrt{\var{\num}}} \right)$ as $P(t)$. 

% We obtain the second part of the result simply by setting $t = 1/2-\lambda_0$ in Eq.~\eqref{eq:normalized-tf-idf-berry-esseen} and taking the complementary event. 
\qed 

%%%%%%%%%%%%%%%%%%%%%%%%%%%%%%%%%%%%%%%%%%%%%%%%%%%%%%%%%%%%%%%%%%
%%%%%%%%%%%%%%%%%%%%%%%%%%%%%%%%%%%%%%%%%%%%%%%%%%%%%%%%%%%%%%%%%%
%%%%%%%%%%%%%%%%%%%%%%%%%%%%%%%%%%%%%%%%%%%%%%%%%%%%%%%%%%%%%%%%%%
%%%%%%%%%%%%%%%%%%%%%%%%%%%%%%%%%%%%%%%%%%%%%%%%%%%%%%%%%%%%%%%%%%
%%%%%%%%%%%%%%%%%%%%%%%%%%%%%%%%%%%%%%%%%%%%%%%%%%%%%%%%%%%%%%%%%%

%
\section{TECHNICAL RESULTS}
\label{sec:technical-results}
We present here some technical results that were used in our analysis regarding binomial random variables (Section~\ref{sec:binomial-wonderland}) and two additional lemmas (Section~\ref{sec:supp-proba-lemmas}). 
\subsection{Binomial wonderland}
\label{sec:binomial-wonderland}
In this section, we collect some facts about binomial random variables. 
We focus on the case $p=1/2$ because of the sampling scheme of Anchors, with a few exceptions. 
We start with straightforward moment computations, which are stated here for completeness' sake.

\begin{lemma}[Moments of the binomial distribution]
\label{lemma:binomial-moments}
Let $m\geq 1$ be an integer and $B\sim \binomial{m}{1/2}$. 
Then
\[
\expec{B} = \frac{m}{2}, \quad \expec{B^2} = \frac{m^2}{4} + \frac{m}{4}, \quad \expec{B^3} = \frac{m^3}{8} + \frac{3m^2}{8}, \quad \text{and}\quad \expec{B^4} = \frac{m^4}{16} + \frac{3m^3}{8} + \frac{3m^2}{16} - \frac{m}{8}
\, .
\]
\end{lemma}

In particular, $\var{B}=m/4$.

\begin{proof}
We use the formula 
\[
\expec{B^p} = \sum_{k=1}^p S_{k,p}m^{\underline{k}}\frac{1}{2^k}
\, ,
\]
where $m^{\underline{k}} = m(m-1)\cdots (m-k+1)$ and $S_{k,p}$ are the Stirling numbers of the second kind (see \citet{knoblauch_2008} for instance). 
\end{proof}

Next, we turn to the computation of the third absolute moment of the binomial, which intervenes in the proof of Proposition~\ref{prop:precision-logistic}. 

\begin{lemma}[Third absolute moment of the binomial]
\label{lemma:third-absolute-moment-binomial}
Let $m\geq 1$ be an even integer. 
Then
\[
\expec{\abs{B-m/2}^3} = \frac{m^2}{2^{m+2}}\binom{m}{m/2}
\, .
\]
\end{lemma}

From Lemma~\ref{lemma:third-absolute-moment-binomial}, we deduce that 
\begin{equation}
\label{eq:bound-third-absolute-moment-binomial}
\forall m\geq 1,\qquad 
\expec{\abs{B-m/2}^3} \leq \frac{1}{\sqrt{8\pi}} m^{3/2}
\, ,
\end{equation}
where we used the well-known bound $\binom{m}{m/2}\leq \frac{\sqrt{2}2^m}{\sqrt{\pi m}}$. 
Eq.~\eqref{eq:bound-third-absolute-moment-binomial} is better than a Jensen-type bound, which can be obtained by noticing that
\[
\left(\expec{\abs{B-m/2}^3}\right)^{4/3} \leq \expec{(B-\expec{B})^4} = \frac{m}{4}\left(1+\frac{3m-6}{4}\right)
\, ,
\]
where we used Lemma~\ref{lemma:binomial-moments} in the last step. 
This last expression is less than $3m^2/16$, and this approach yields $\expec{\abs{B-m/2}^3}\leq (3/16)^{3/4}m^{3/2}$. 
Since $1/\sqrt{8\pi}\approx 0.2$ whereas $(3/16)^{3/4}\approx 0.28$, we prefer the use of  Eq.~\eqref{eq:bound-third-absolute-moment-binomial} when bounding the third absolute moment of the binomial. 

\begin{proof}
We follow \citet{diaconis_zabell_1991}. 
First, we notice that the polynomial $(X-m/2)^3$ can be written
\begin{equation}
\label{eq:cubic-polynomial-kravchuk}
(X-m/2)^3 = \frac{-3}{4}P_3^m(X) + \frac{-3m+2}{8} P_1^m(X)
\, ,
\end{equation}
where $P_k^m$ denotes the Kravchuk polynomial of order $k$ \citep{macwilliams_sloane_1977}. 
%Let us set for any integer $k\in [m]$, let us set $b(k;m,p)$
Using Lemma~1 of \citet{diaconis_zabell_1991}, we see that 
\begin{align*}
\frac{1}{2^m}\sum_{k=0}^{m/2} \binom{m}{k} (k-m/2)^3 &= \frac{-3}{4}\frac{1}{2^m}\sum_{k=0}^{m/2} \binom{m}{k}P_3^m(k) + \frac{-3m+2}{8} \frac{1}{2^m}\sum_{k=0}^{m/2} \binom{m}{k}P_3^m(k)  \\
&= \frac{-3}{4} \frac{m}{6} \frac{1}{2^m}\binom{m}{m/2}P_2^{m-1}(m/2) \\
&+ \frac{-3n+2}{8} \frac{m}{2}\frac{1}{2^m}\binom{m}{m/2}P_0^{m-1}(m/2) \\
\frac{1}{2^m}\sum_{k=0}^{m/2} \binom{m}{k} (k-m/2)^3 &= \frac{-m^2}{2^{m+3}}\binom{m}{m/2}
\, .
\end{align*}
Observing that the third absolute moment is twice the absolute value of the last display yields the desired result. 
\end{proof}

\begin{remark}
It is unfortunately not possible to obtain a simple closed-form for a parameter of the binomial $p$ not equal to $1/2$ using this method. 
Indeed, using the more general expression of the Kravchuk polynomials (sometimes called the Meixner polynomials \citep{meixner_1934}) the decomposition obtained in Eq.~\eqref{eq:cubic-polynomial-kravchuk} becomes $(X-mp)^3 = \sum_{q=0}^{3}\lambda_q P_q^m(X)$, with 
\[
\begin{cases}
\lambda_0 &= mp(1-p)(1-2p) \\
\lambda_1 &= \frac{-3m+2}{4}(1-p) + \frac{3m-6}{4}(1-p)(1-2p)^2 \\
\lambda_2 &= 6(1-2p)(1-p)^2 \\
\lambda_3 &= -6(1-p)^3
\, .
\end{cases} 
\]
In particular, $\lambda_0$ is nonzero whenever $p\neq 1/2$. 
Therefore, the partial sums of the binomial coefficients make their appearance, for which there is no simple closed-form. 
\end{remark}

\subsection{Other probability results}
\label{sec:supp-proba-lemmas}
\begin{lemma}[Probability splitting]
	\label{lemma:probability-splitting}
	Let $X$ and $Y$ be two random variables, $t\in\Reals$ and $\epsilon >0$.
	Then 
	\[
	\proba{Y\leq t} \leq \proba{X \leq t+\epsilon} + \proba{\abs{X-Y} > \epsilon}
	\, .
	\] 
\end{lemma}

\begin{proof}
	This result is classical, we report the proof for completeness' sake. 
	\begin{align*}
	\proba{Y \leq t} &= \proba{Y \leq t, X\leq t+\epsilon} + \proba{Y\leq t, X > t+\epsilon} \\
	&\leq \proba{X \leq t+\epsilon} + \proba{Y-X \leq t-X, t-X < -\epsilon} \\
	&\leq \proba{X \leq t+\epsilon}  + \proba{Y-X < -\epsilon} \\
	&\leq \proba{X \leq t+\epsilon}  + \proba{Y-X < -\epsilon} + \proba{Y-X > \epsilon}\\
	\proba{Y\leq t} &\leq \proba{X \leq t+\epsilon} + \proba{\abs{X-Y} > \epsilon}
	\, .
	\end{align*}
\end{proof}

As a direct consequence, we have the following:

\begin{lemma}[Convergence in probability implies convergence in distribution]
	\label{lemma:cv-proba-implies-cv-dist}
	Let $X$ and $Y$ be two random variables, $t\in\Reals$ and $s >0$.
	Then 
	\[
	\proba{X\leq t-s} - \proba{\abs{X-Y} > s} \leq \proba{Y \leq t} \leq \proba{X \leq t+s} + \proba{\abs{X-Y} > s}
	\, .
	\]
\end{lemma}

\begin{proof}
	Applying Lemma~\ref{lemma:probability-splitting} to $Y$ and $X$ instead of $X$ and $Y$, and $t-s$ instead of $t$ yields
	\begin{equation}
	\label{eq:lemma-inverse-call}
	\proba{X \leq t-s} \leq \proba{Y\leq t} + \proba{\abs{X-Y} > s}
	\, .
	\end{equation}
	Combined with the original statement, we obtain the result.
\end{proof}

%%%%%%%%%%%%%%%%%%%%%%%%%%%%%%%%%%%%%%%%%%%%%%%%%%%%%%%%%%%%%%%%%%%

%
\section{ADDITIONAL EXPERIMENTAL RESULTS}
\label{sec:additional-experiments}
In this section, we collect additional experimental results omitted from the main paper due to space limitations. 
Specifically, in Section~\ref{sec:empirical-investigation} we report statistics about the TF-IDF vectorization, in Section~\ref{sec:comparison-anchors} we empirically show that Anchors and exhaustive Anchors produce similar explanations, In Section~\ref{sec:dummy-property-experiments} we provide a counterexample proving that the default implementation of Anchors does not satisfy Property~\ref{prop:dummy-features} (Dummy Property). 
Sections~\ref{sec:check-prop-precision-logistic} and~\ref{sec:check-prop-normalized-tf-idf-berry-esseen} provide empirical validation of Propositions~\ref{prop:precision-logistic} and~\ref{prop:normalized-tf-idf-berry-esseen}, respectively.
Finally, additional experimental results for Sections~\ref{sec:analysis} and~\ref{sec:neural-nets} are in Sections~\ref{sec:additional_analysis} and~\ref{sec:additional_nn}. 
The code used for the experiments is available at \url{https://github.com/gianluigilopardo/anchors_text_theory}.

\paragraph{Setting.} 
All the experiments reported in this Section and in the paper are implemented in \texttt{Python} and executed on CPUs. 
Three dataset are used: Restaurant Reviews (available at \url{https://www.kaggle.com/hj5992/restaurantreviews}), Yelp Reviews (available at \url{https://www.kaggle.com/omkarsabnis/yelp-reviews-dataset}), and IMDB Reviews (available at \url{https://www.kaggle.com/datasets/lakshmi25npathi/imdb-dataset-of-50k-movie-reviews}). 
Unless otherwise specified, all the experiments work with the official implementation of Anchors (available and licensed at \url{https://github.com/marcotcr/anchor}) and default parameters.  
The vectorizer is always TF-IDF from \url{https://scikit-learn.org/stable/modules/generated/sklearn.feature\_extraction.text.TfidfVectorizer.html} with the option \texttt{norm=None}. 
When experiments require it (Sections~\ref{sec:comparison-anchors}, \ref{sec:dummy-property-experiments}, \ref{sec:additional_analysis}, \ref{sec:additional_nn}), we use $75\%$ of the dataset for training and $25\%$ for testing. 
All machine learning models used in the experiments were trained with the default parameters of \url{https://scikit-learn.org/}. 
Finally, we remark that we always consider documents with positive predictions, \emph{i.e.}, such that $f(\doc)=1$.

%%%%%%%%%%%%%%%%%%%%%%%%%%%%%%%%%%%%%%%%%%%%%%%%%%%%%%%%%%%%%%%%%%

\subsection{Typical values of \texorpdfstring{$\mult_j$}{} and \texorpdfstring{$\idf_j$}{}}
\label{sec:empirical-investigation}

Figure~\ref{fig:histogram_restaurant} and Figure~\ref{fig:histogram_yelp} show statistics about the TF-IDF transforms of the two considered datasets. 
In Figure~\ref{fig:histogram_restaurant} the average document length $b$ is $11$: each document is a short review, generally containing one or two short sentences, while in Figure~\ref{fig:histogram_yelp} the average length $b$ is $133$: documents are quite longer.
This significant difference in documents size is also visible in the multiplicities. 
In Figure~\ref{fig:histogram_restaurant}, the typical value for the term frequency $\mult_j$ is $1$ and it is rarely higher than $3$, while in Figure~\ref{fig:histogram_yelp} the average is closer to $2$ and multiplicities greater than $10$ are present. 
In contrast, the average, median, and maximum value for the inverse document frequency $\idf_j$ are around $7$ for both datasets: indeed, considering their size is around $N=1000$ and that the typical value for $N_j$ is $1$, we get $\idf_j = \log\frac{N+1}{N_j+1}+1 \approx 7$.

\begin{figure}[h]
    \centering
    \includegraphics[scale=0.45]{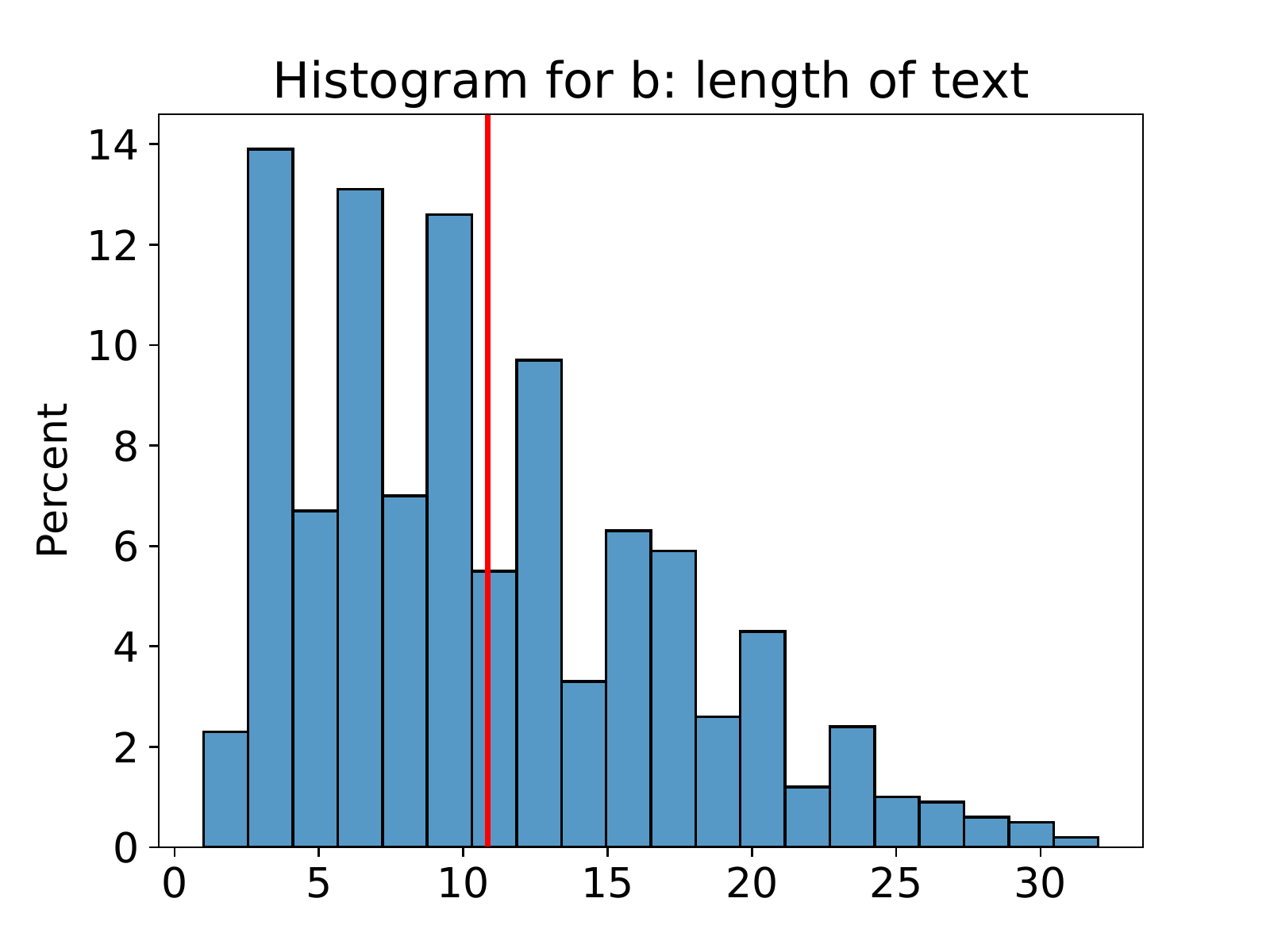}
    \includegraphics[scale=0.45]{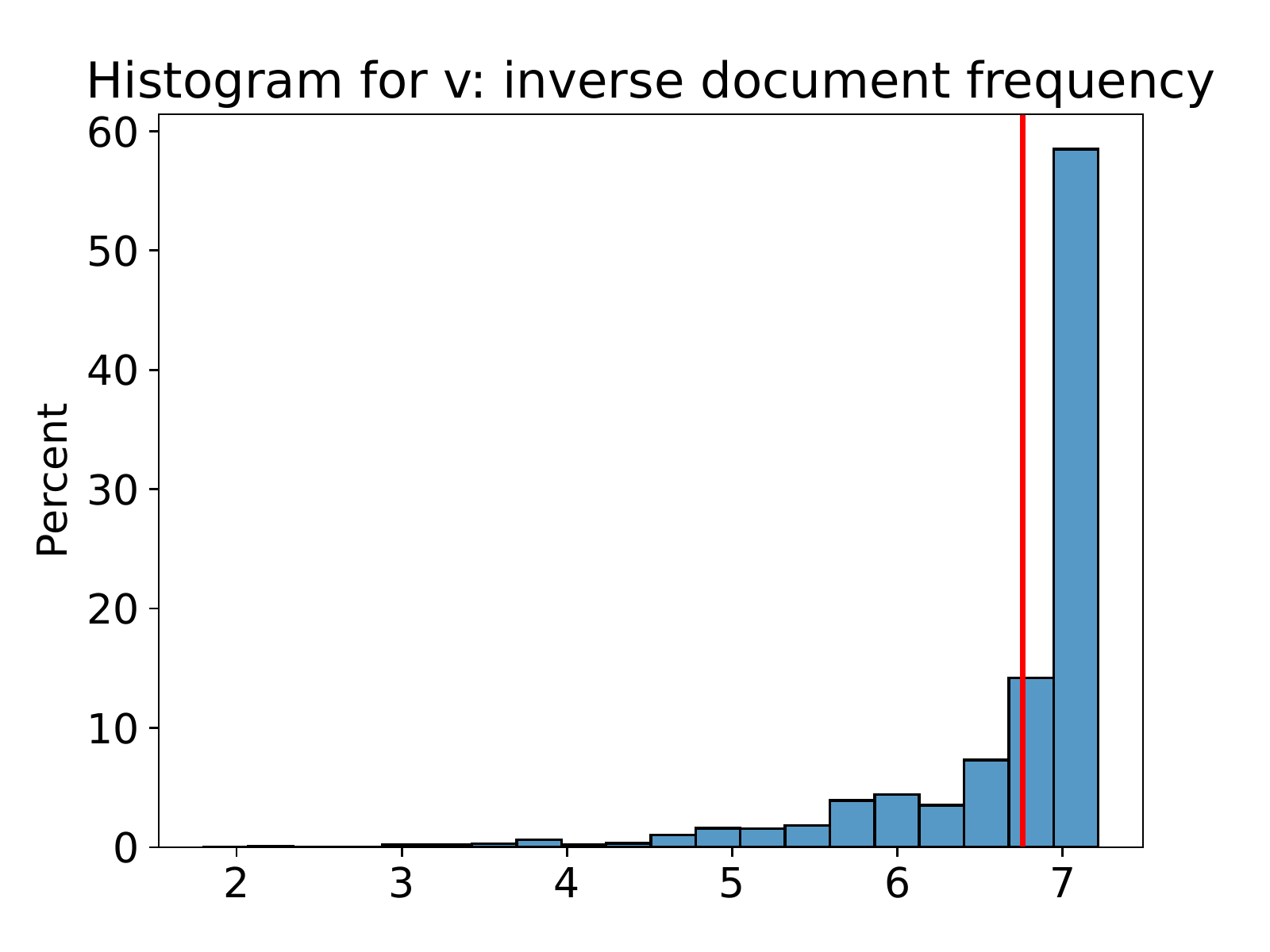}
    \includegraphics[scale=0.45]{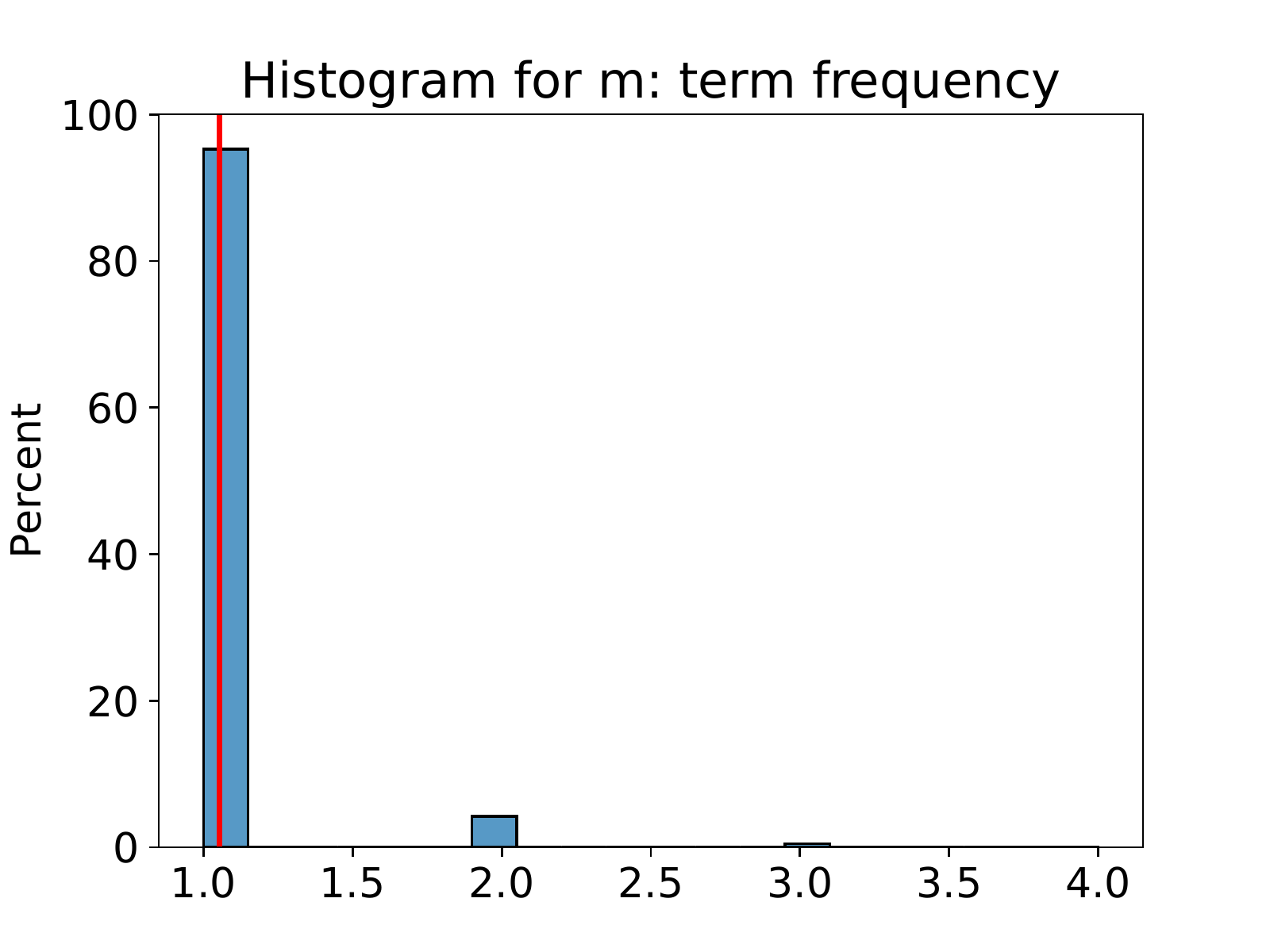}
    \includegraphics[scale=0.45]{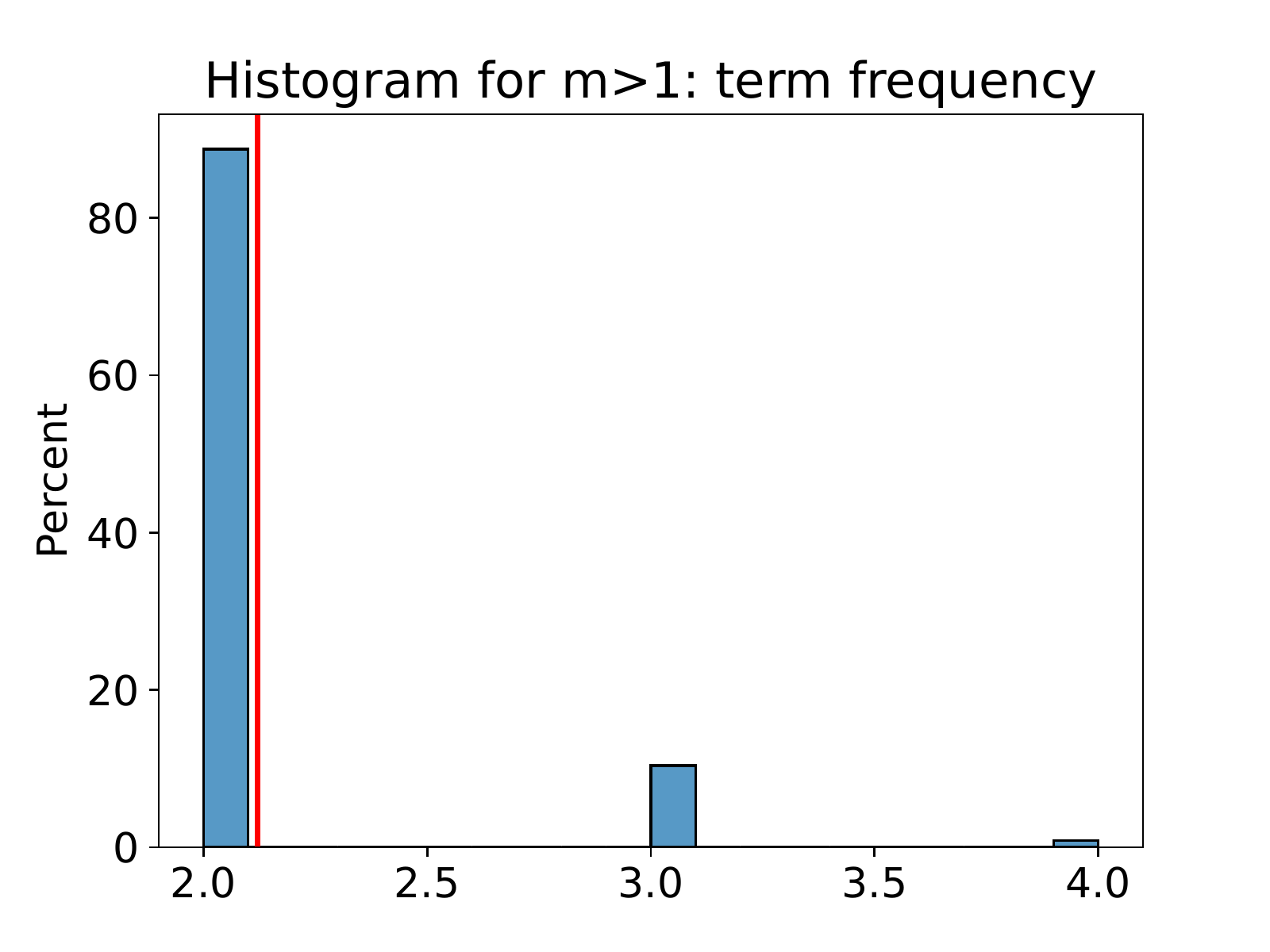}
    \caption{\label{fig:histogram_restaurant}Histograms for length of the document $b$ (upper left), the inverse document frequency $\idf_j$ (upper right), the term frequency $\mult$ (lower left), the term frequency when $\mult >1$ (lower right) for the Restaurant Reviews dataset from \texttt{https://www.kaggle.com/hj5992/restaurantreviews}. Average value is reported in red.}
\end{figure}

\begin{figure}[h]
    \centering
    \includegraphics[scale=0.45]{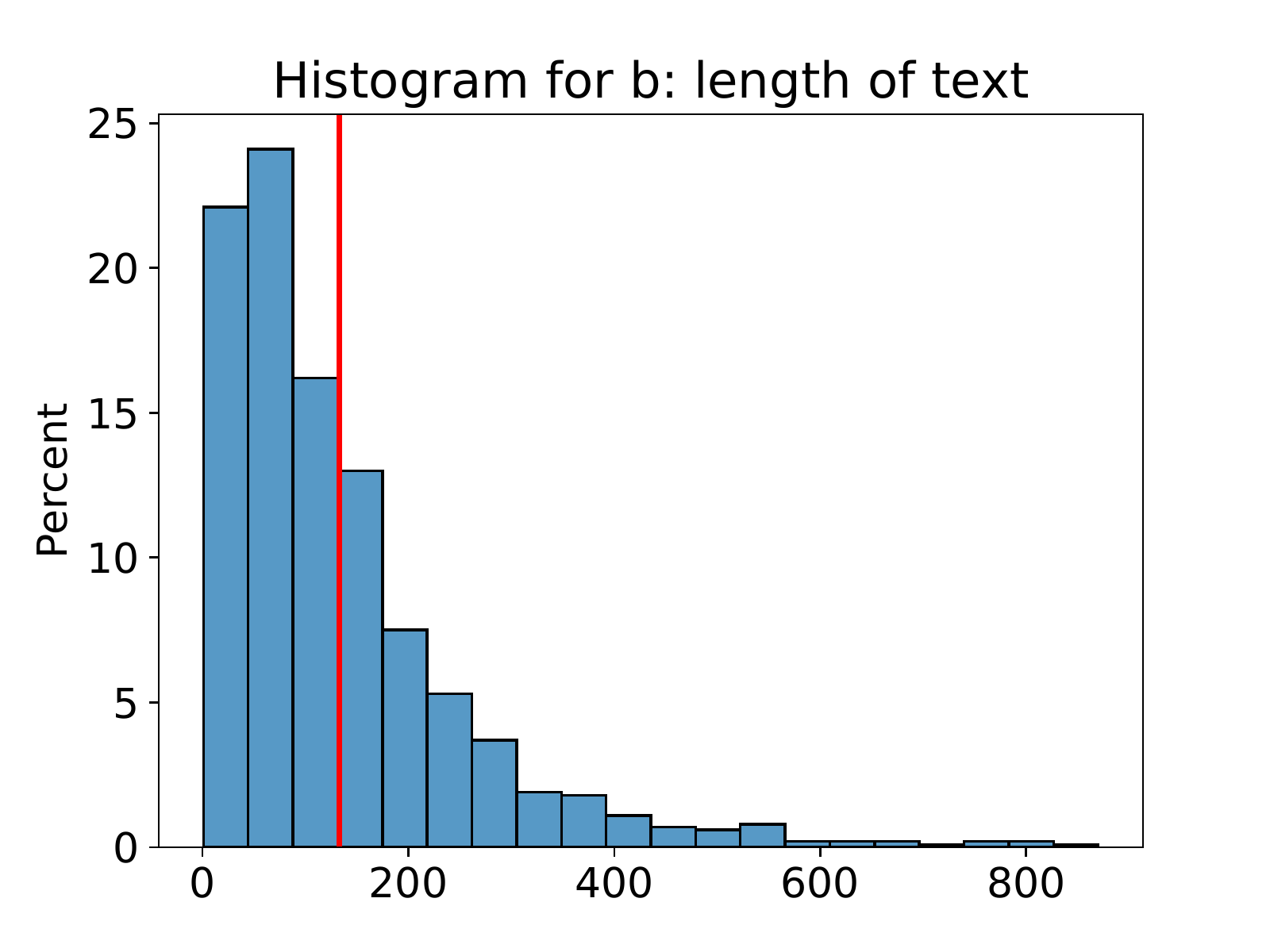}
    \includegraphics[scale=0.45]{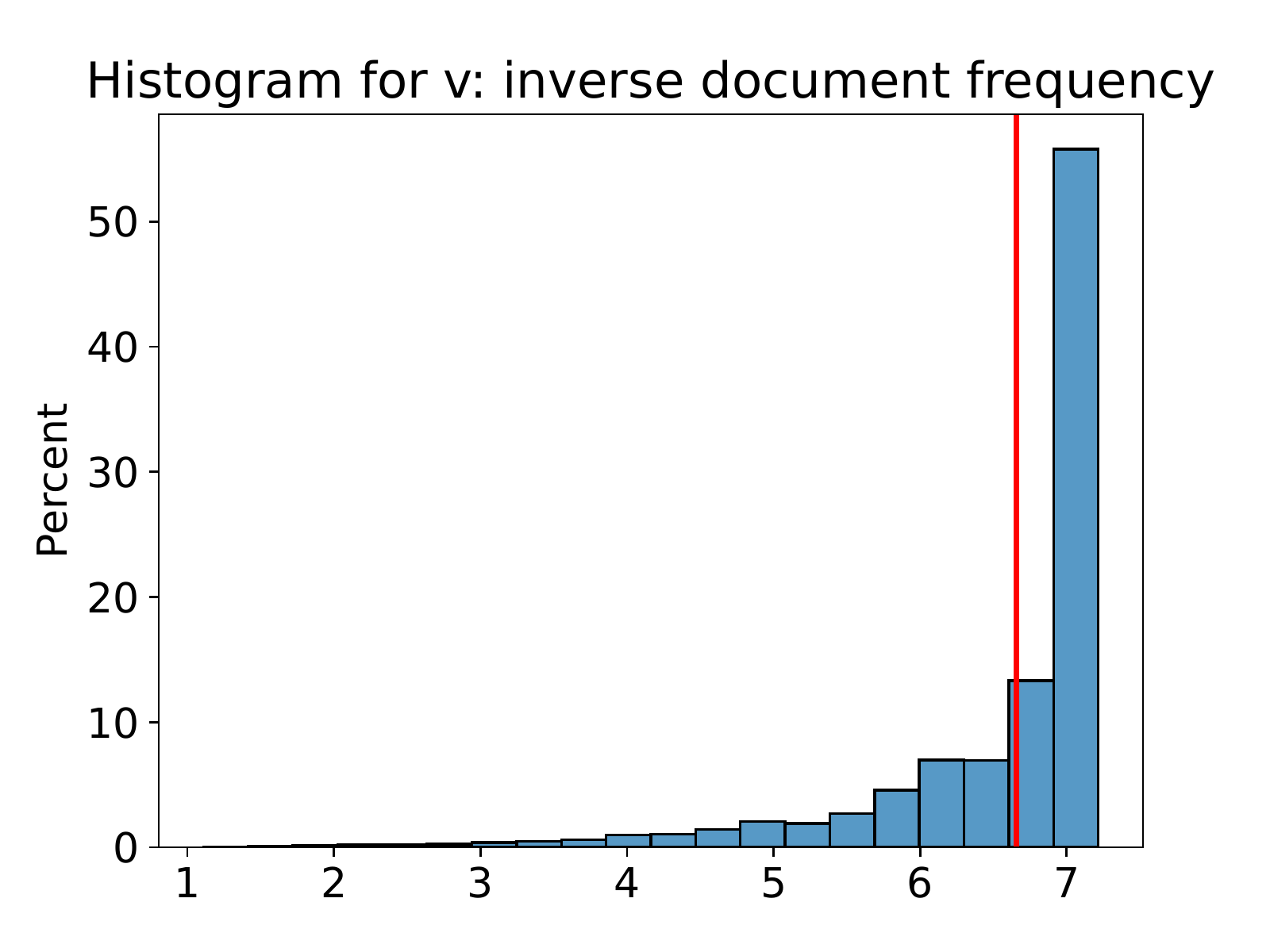}
    \includegraphics[scale=0.45]{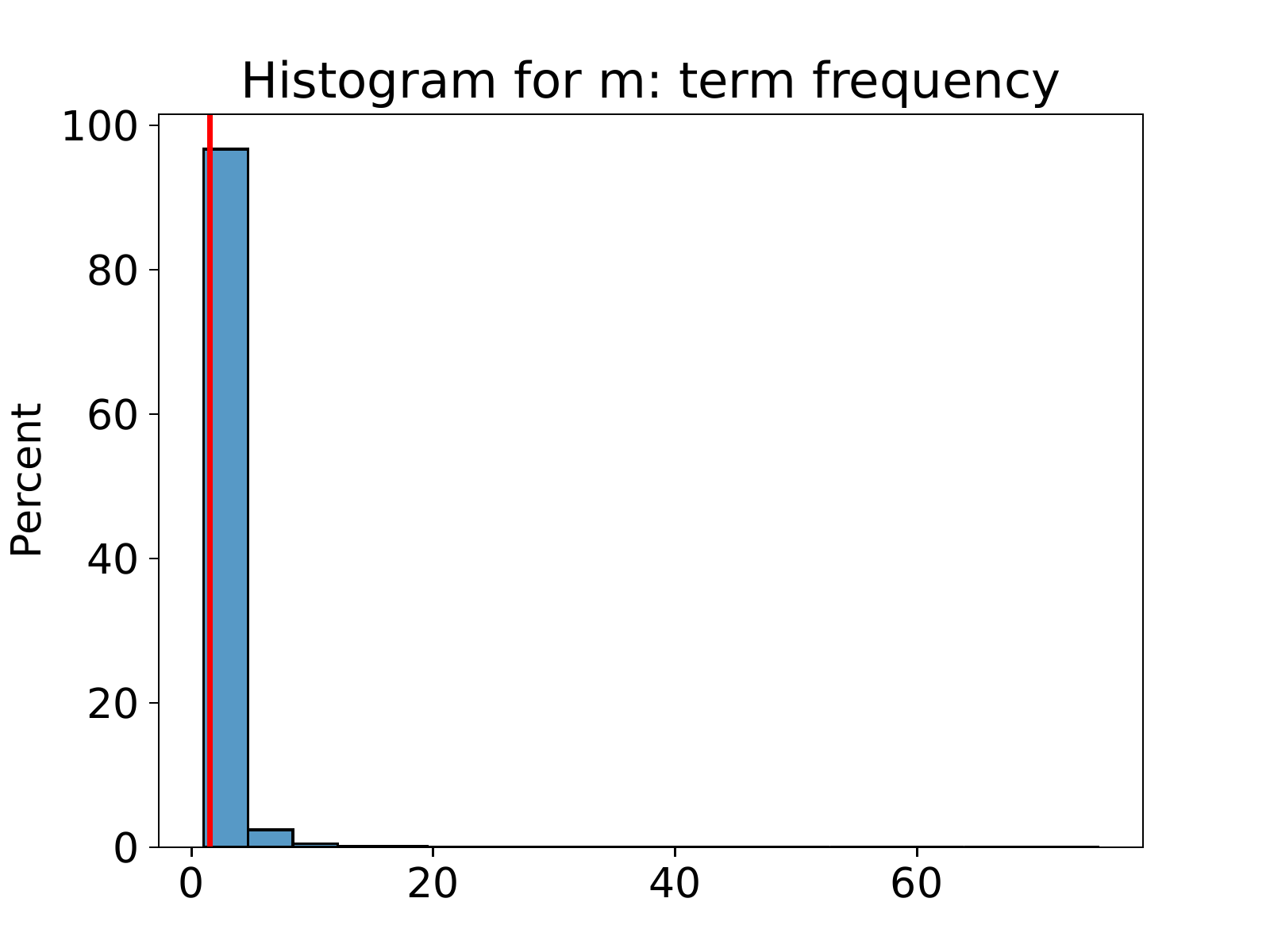}
    \includegraphics[scale=0.45]{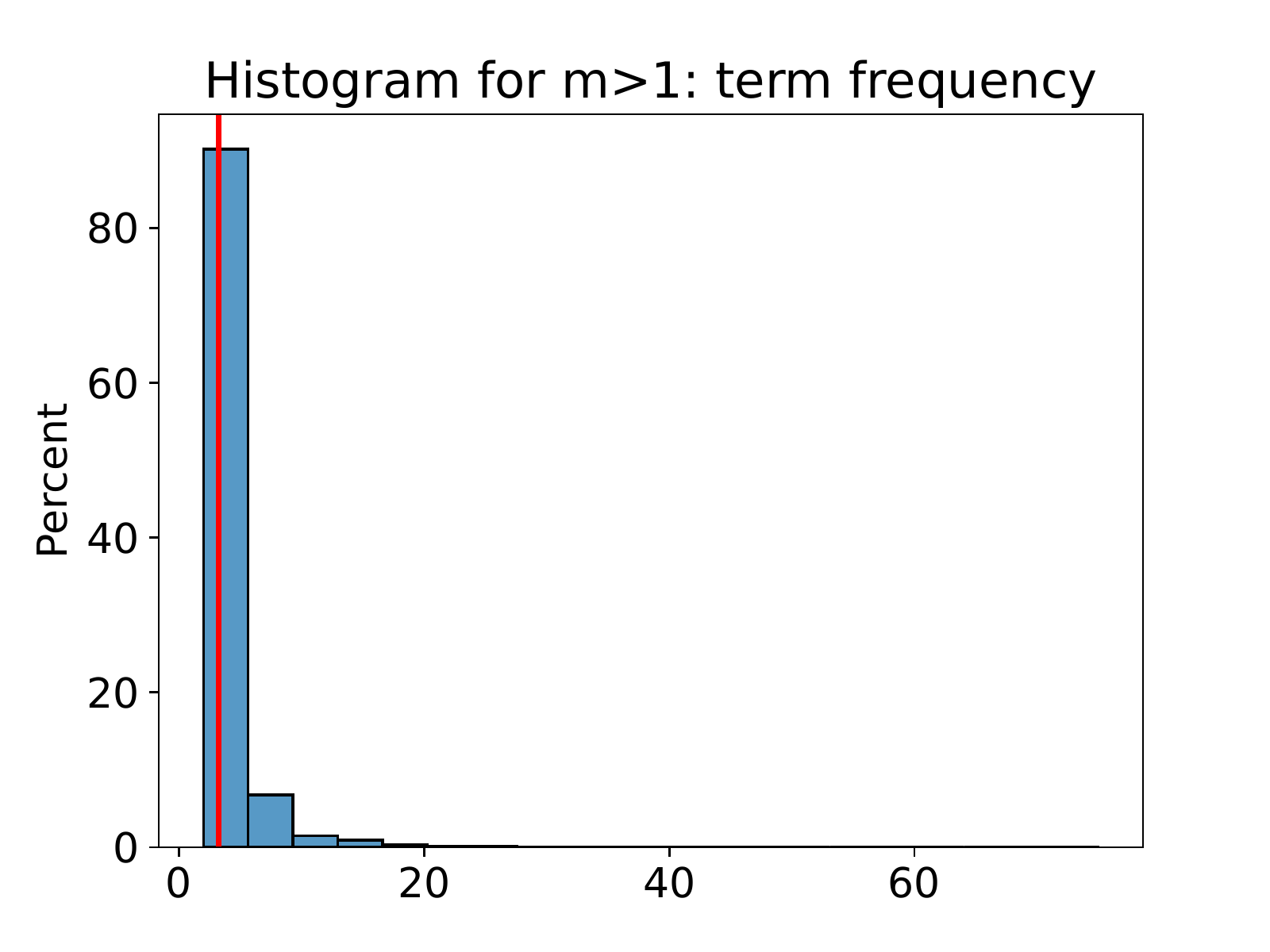}
    \caption{\label{fig:histogram_yelp}Histograms for length of the document $b$ (upper left), the inverse document frequency $\idf_j$ (upper right), the term frequency $\mult$ (lower left), the term frequency when $\mult >1$ (lower right) for a subset of the Yelp Reviews dataset from \texttt{https://www.kaggle.com/omkarsabnis/yelp-reviews-dataset}.
    Note that the maximum value of multiplicity $\mult$ is $75$ in this case, while the average (in red) is $1.5$.}
\end{figure}

%%%%%%%%%%%%%%%%%%%%%%%%%%%%%%%%%%%%%%%%%%%%%%%%%%%%%%%%%%%%%%

\subsection{Comparison between Anchors and exhaustive Anchors}
\label{sec:comparison-anchors}
We compute the similarity through the Jaccard index, defined as
\begin{equation*}
     J(A^P, A^d) \defeq \frac{\card{A^P \cap A^d}}{{\card{A^P \cup A^d}}} = \frac{\card{A^P \cap A^d}}{\card{A^P} + \card{A^d} - \card{A^P \cap A^d}}
     \, ,
\end{equation*}
where $A^P$ is the anchor obtained by running empirical Anchors (exhaustive version with empirical precision as an evaluation function) and $A^d$ with default implementation.  
Table~\ref{tab:similarity} shows the average Jaccard index for the two datasets considered and five different models. 
Overall, the output of the two methods is quite similar. 

\vspace{1cm}
\begin{table}[H]
    \caption{\label{tab:similarity}Jaccard similarity between exhaustive Anchors and default implementation.}
    \centering
    \begin{tabular}
    {c | c c c c c }
     & Indicator & DTree & Logistic & Perceptron & RandomForest \\  % & time(s) \\
     \hline
    Restaurants & $1.00$ & $1.00$ & $0.90$ & $0.87$ & $0.93$ \\  % & $1691.55$ \\ 
    Yelp & $1.00$ & $1.00$ & $0.71$ & $0.68$ & $0.75$ \\
    \end{tabular}
\end{table}

As shown in Figure~\ref{fig:histogram_yelp}, the Yelp dataset has longer documents, making Anchors more unstable (namely outputting quite different anchors for the same model / document configuration). 
This explains why the similarity is lower in that case. 
In addition, Anchors requires a computational capacity that grows exponentially with the length of the document (and the length of the optimal anchor). 
This makes it particularly onerous to apply empirical Anchors to large documents.
Indeed, the experiment of Table~\ref{tab:similarity} requires about half an hour on Restaurants reviews, while more than $24$ hours are needed on Yelp reviews.

%%%%%%%%%%%%%%%%%%%%%%%%%%%%%%%%%%%%%%%%%%%%%%%%%%%%%%%%%%%%%%%%%%

\subsection{Dummy property}
\label{sec:dummy-property-experiments}
We report a counterexample showing that the default implementation of Anchors does not satisfy Proposition~\ref{prop:dummy-features}. 
In Figure~\ref{fig:dummy}, the word \emph{indie} appears in $4$ anchors, even though the model does not depend on it. 
While the frequency of occurrence is not high, it is still non-zero. 
This is slightly problematic in our opinion: since the model does not depend on the word \emph{indie}, its appearance in the explanation is misleading for the user. 
We conjecture that this behavior is entirely due to the optimization procedure used in the default implementation of Anchors, since the exhaustive version is guaranteed not to have this behavior by Proposition~\ref{prop:dummy-features}. 
We want to emphasize that there is nothing special with the example presented here and other counterexamples can be readily created. 

\begin{figure}[t]
    \centering
    \includegraphics[scale=0.4]{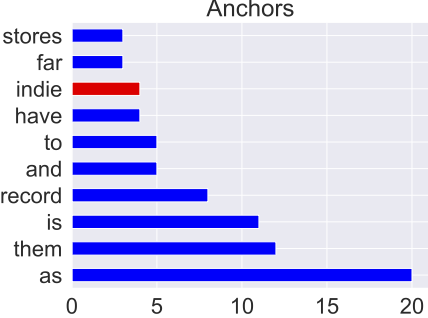}
    \caption{\label{fig:dummy}$10$ most frequent anchors on $100$ runs of default Anchors on a logistic model with zero coefficient for \emph{indie} and arbitrary coefficients for the other words. 
    \emph{indie} is a dummy feature, but still appears in $4$ anchors. } 
\end{figure}

%%%%%%%%%%%%%%%%%%%%%%%%%%%%%%%%%%%%%%%%%%%%%%%%%%%%%%%%%%%%%%%%%%

\subsection{Empirical validation of Proposition~\ref{prop:precision-logistic}: Precision of a linear classifier} 
\label{sec:check-prop-precision-logistic}
Figure \ref{fig:precision-logistic} shows an empirical validation for Proposition~\ref{prop:precision-logistic} for different document size and for anchors of different sizes. 
The fit between the empirical distribution and $\Phibar\circ \Approxprec$ is much better as predicted by Proposition~\ref{prop:precision-logistic}, even for small values of $d$. 
This motivates our further study of the approximate precision instead of the precision. 
From the results in Figure~\ref{fig:precision-logistic}, we can see why the anchors need to be small with respect to the document size: if they are two large, the approximation of the precision is not justified. 
We remark, again, that this assumption is entirely reasonable, since an anchor using more than half the document to explain a prediction is not interpretable. 
In addition, Anchors rarely returns such anchors. 

\begin{figure}[h]
    \centering
    \includegraphics[scale=0.4]{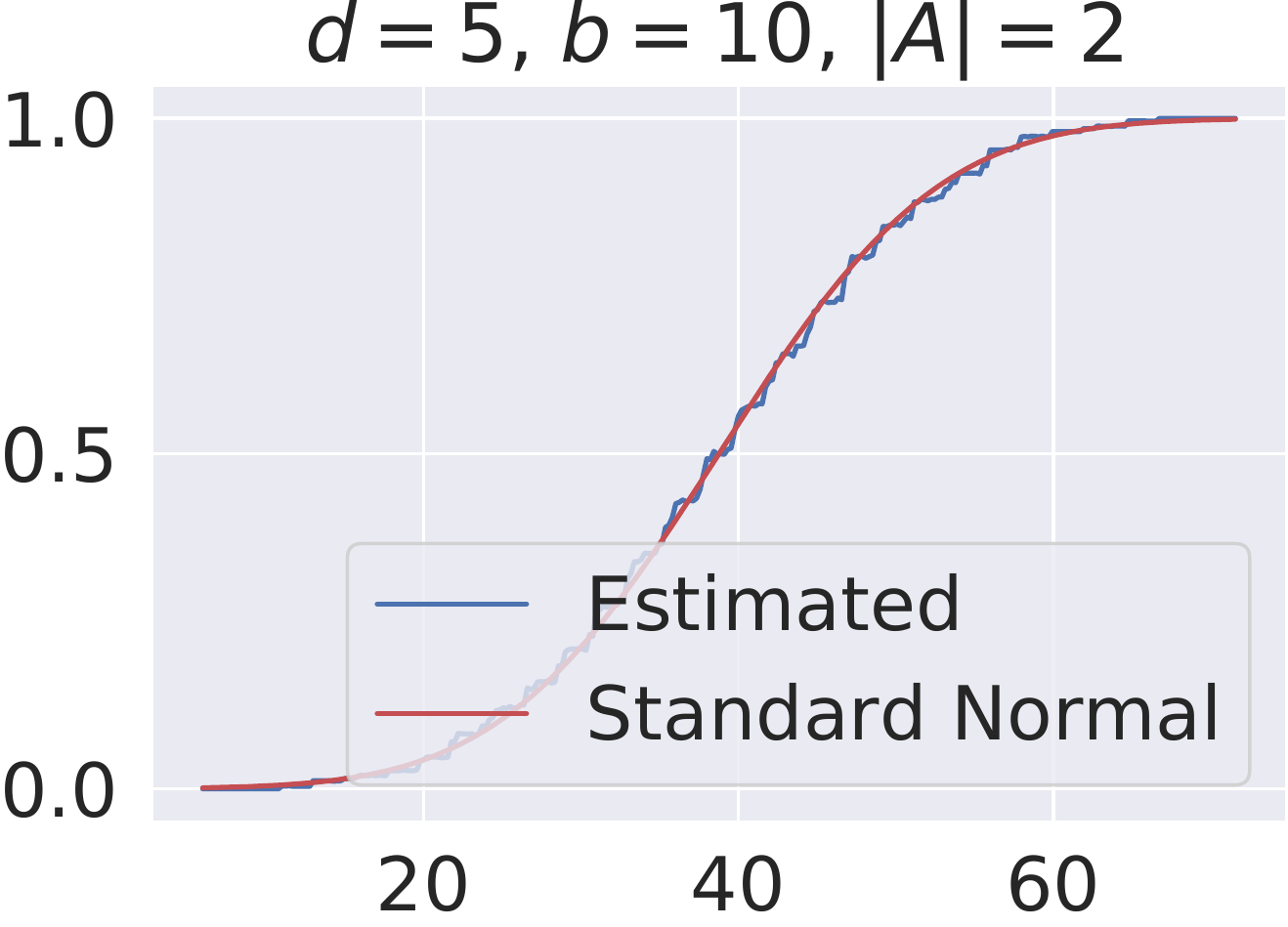}
    \hspace{1cm}
    \includegraphics[scale=0.4]{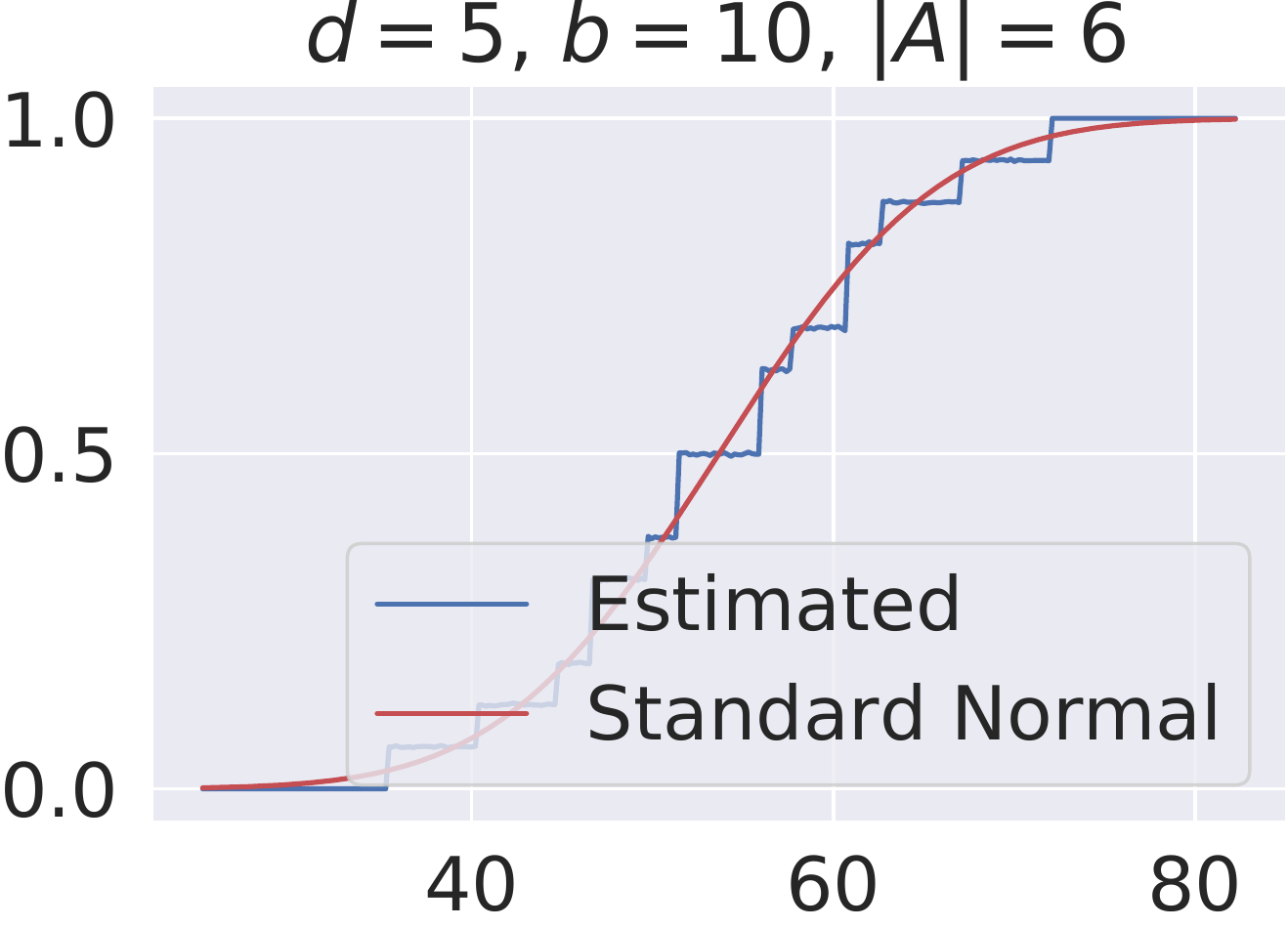} \\
    \vspace{0.5cm}
    \includegraphics[scale=0.4]{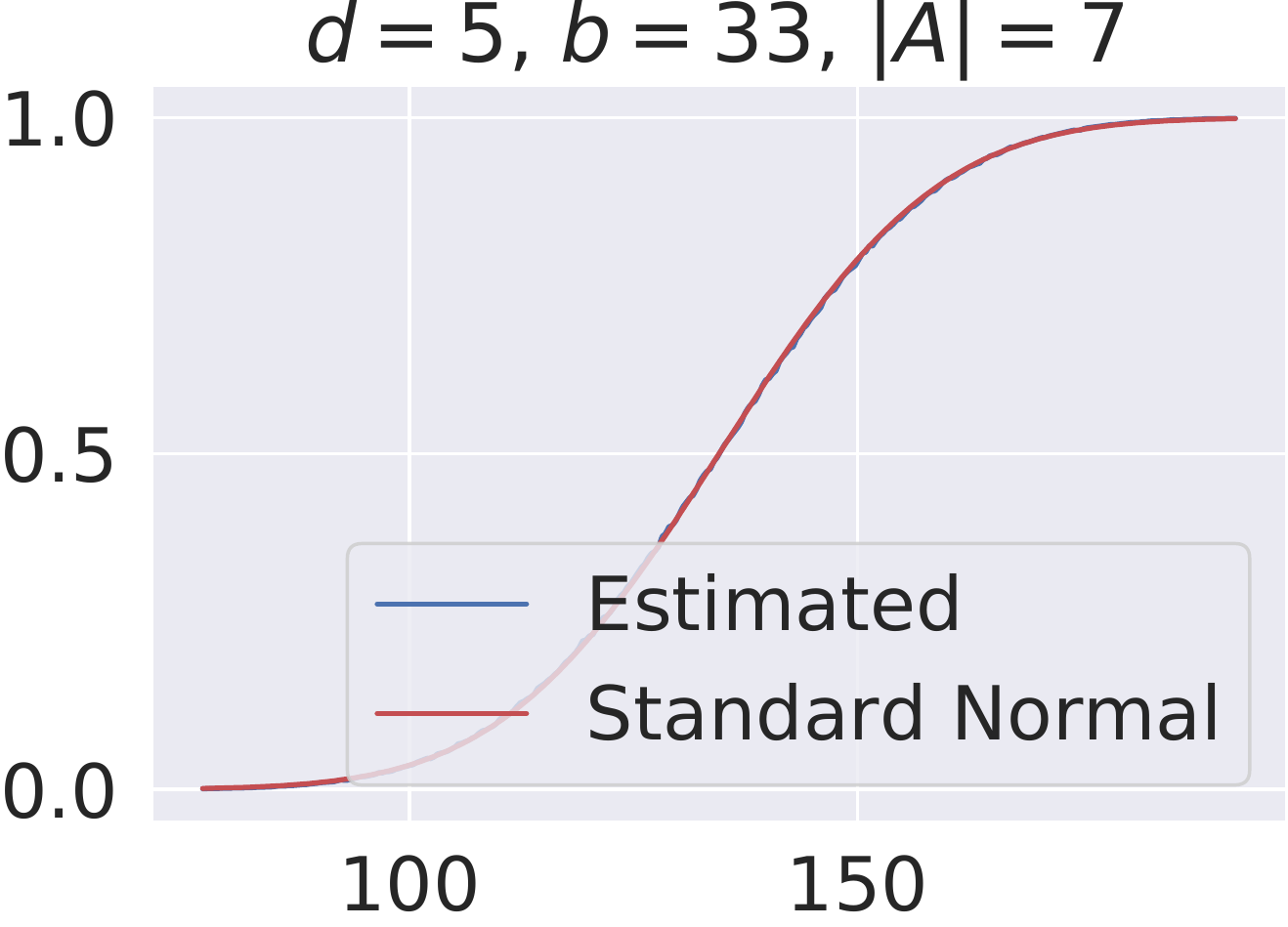}
    \hspace{1cm}
    \includegraphics[scale=0.4]{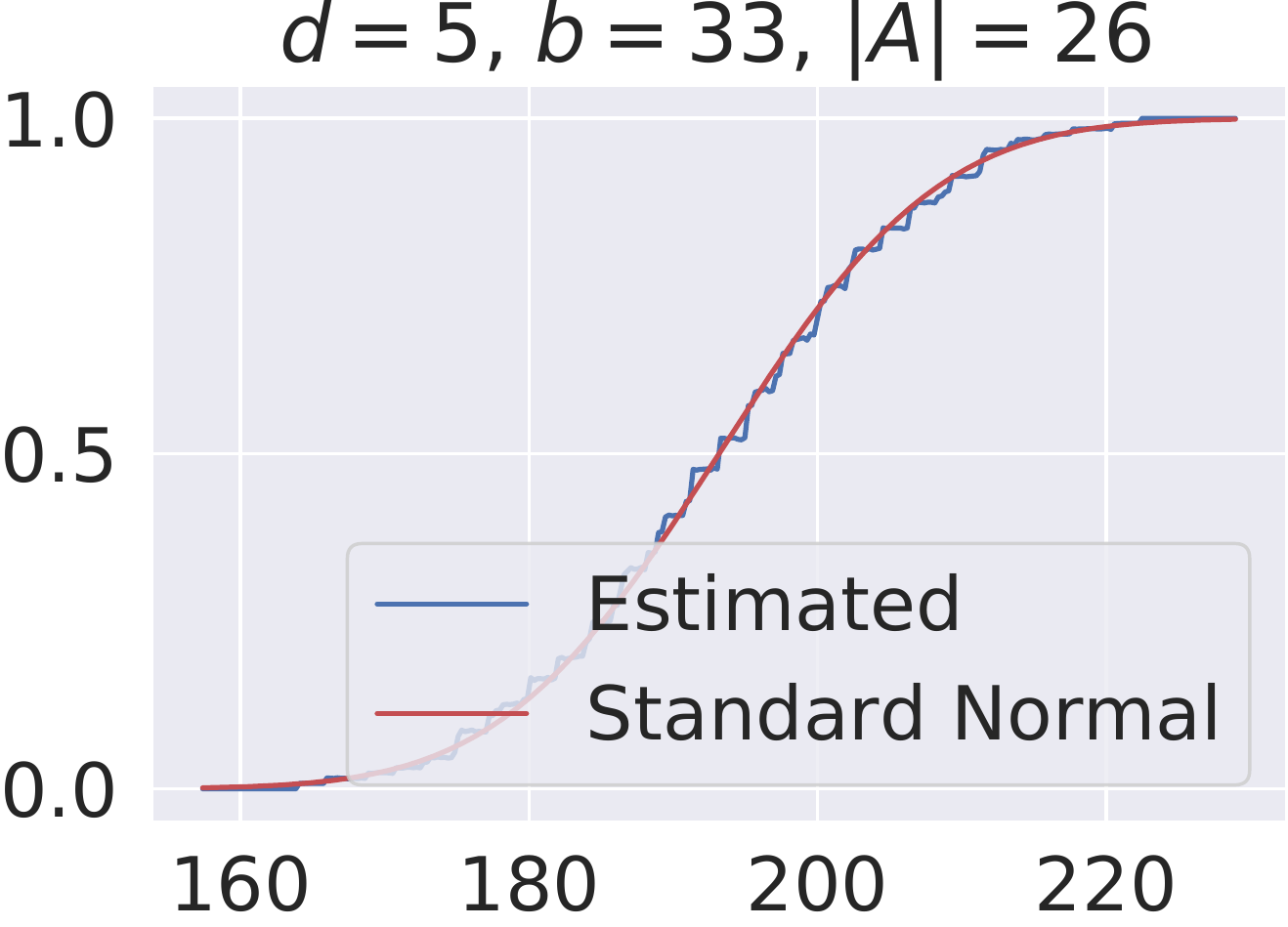} \\
    \vspace{0.5cm}
    \includegraphics[scale=0.4]{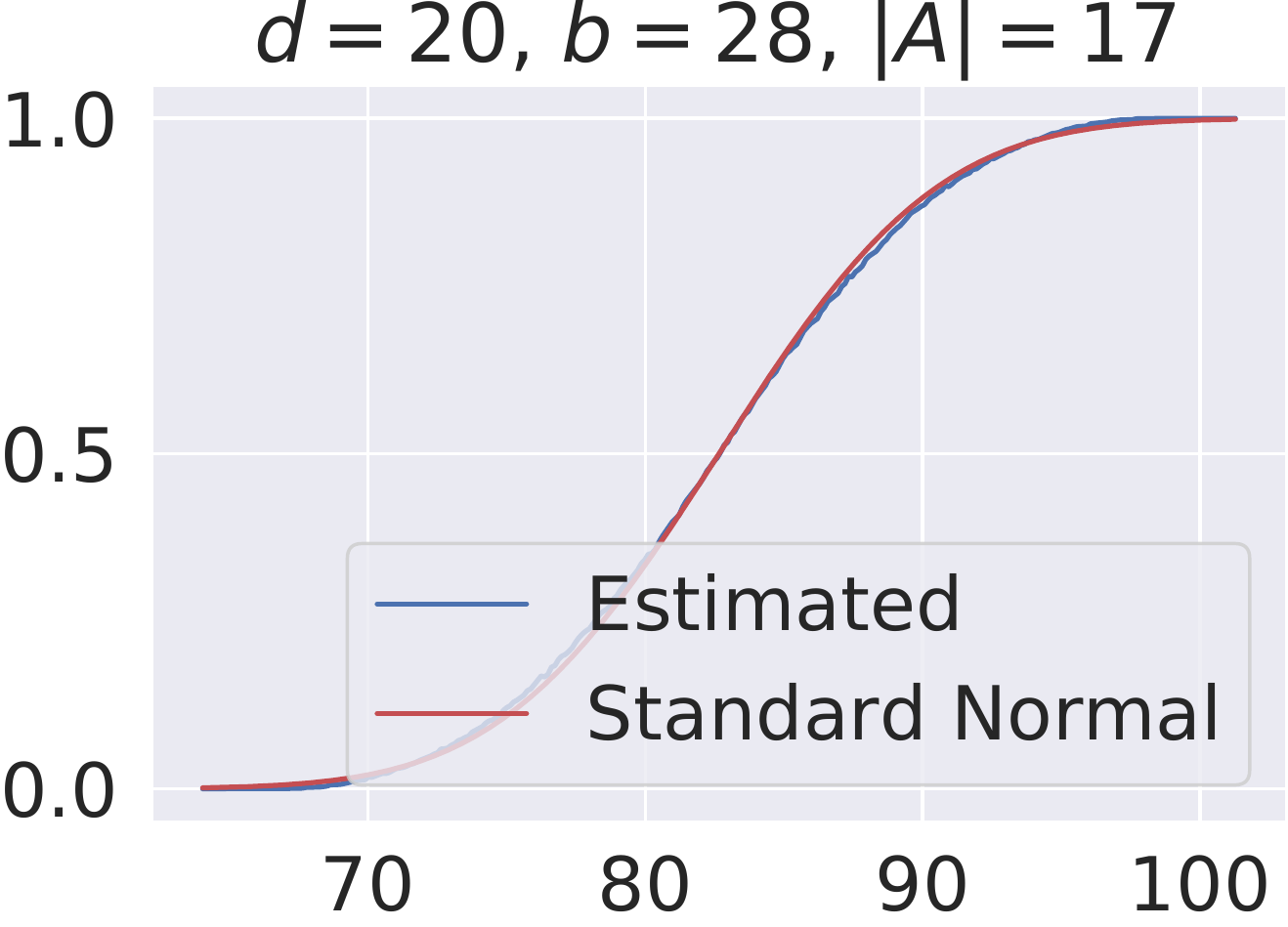}
    \hspace{1cm}
    \includegraphics[scale=0.4]{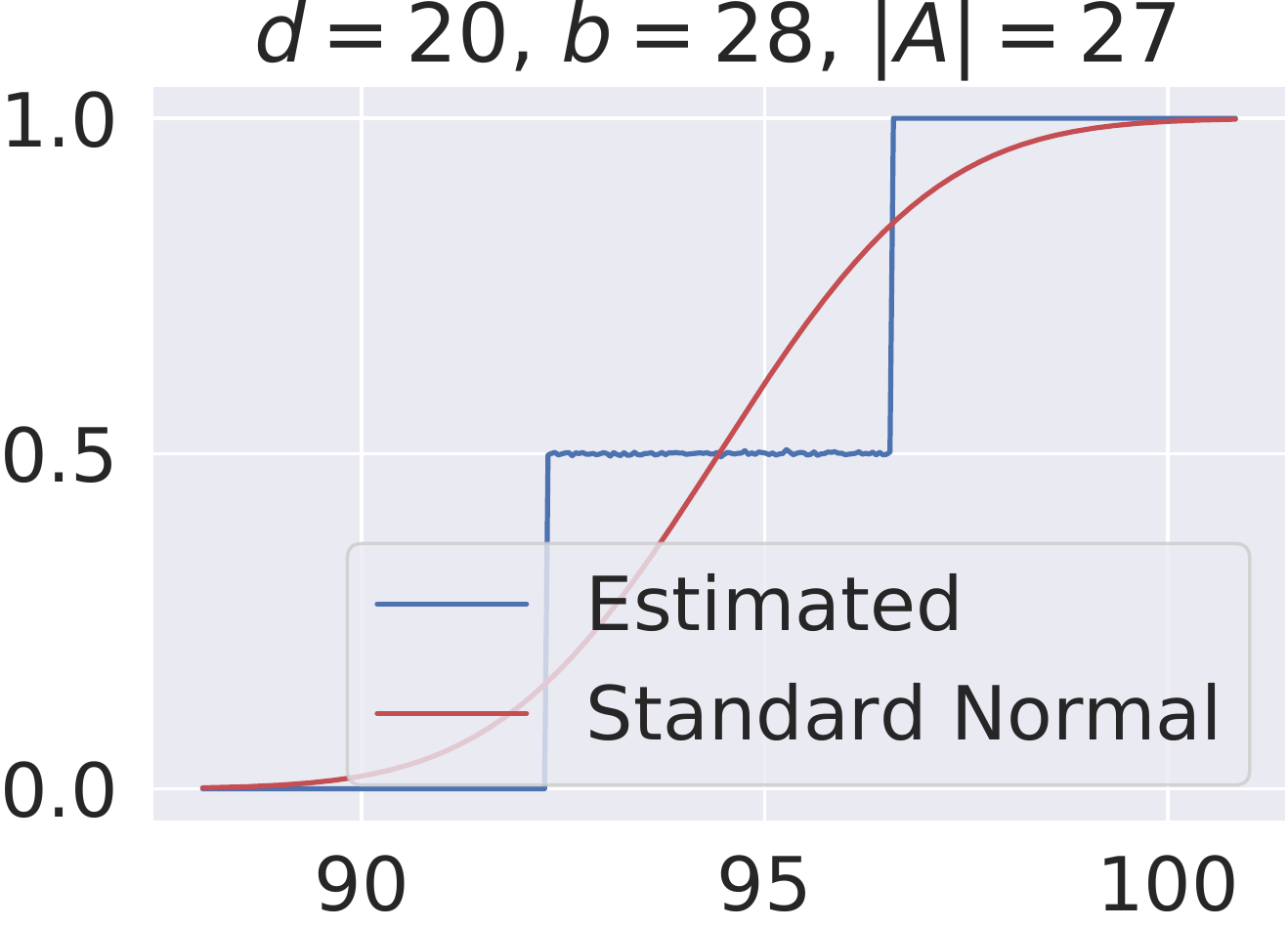} \\
    \vspace{0.5cm}
    \includegraphics[scale=0.4]{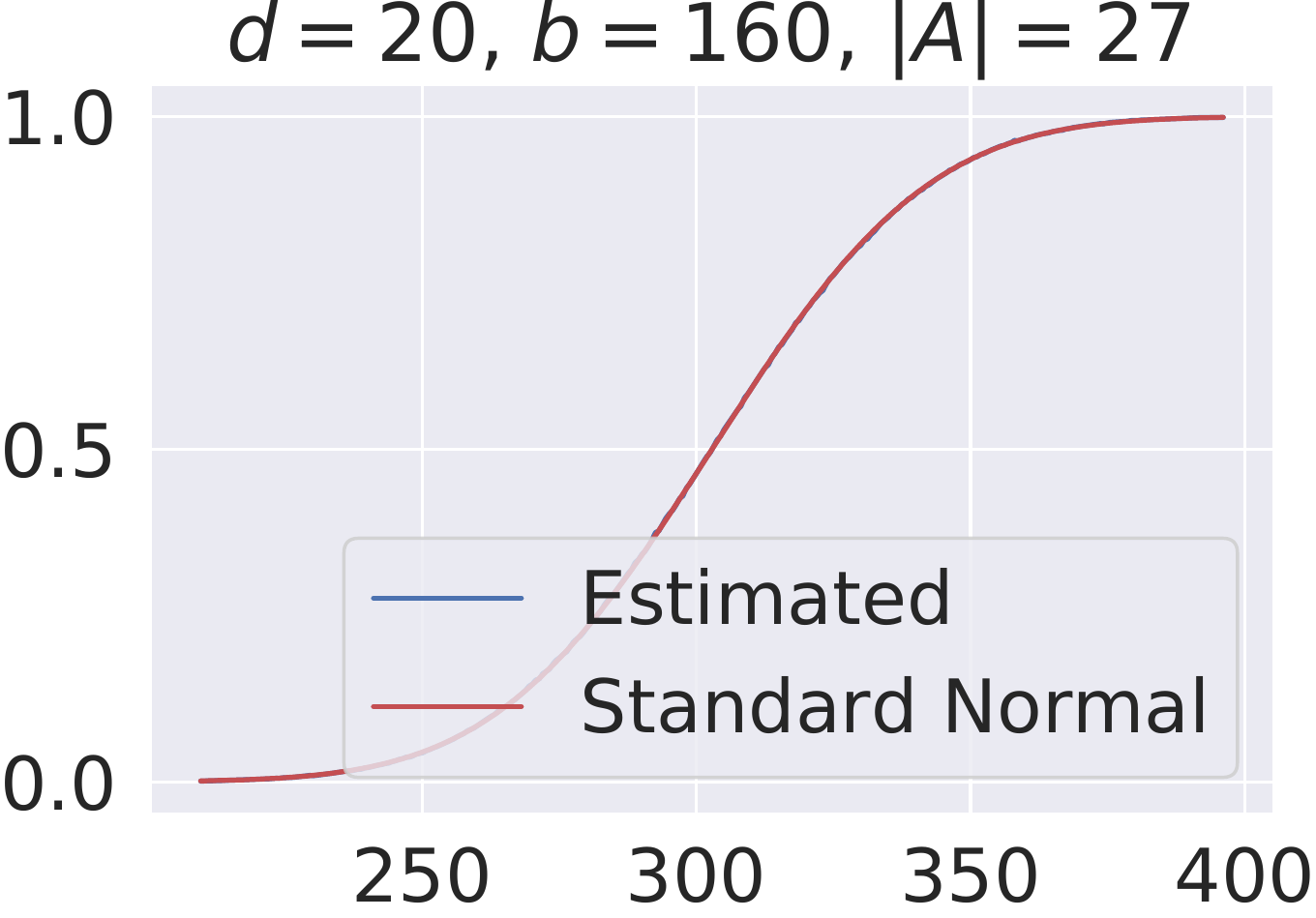}
    \hspace{1cm}
    \includegraphics[scale=0.4]{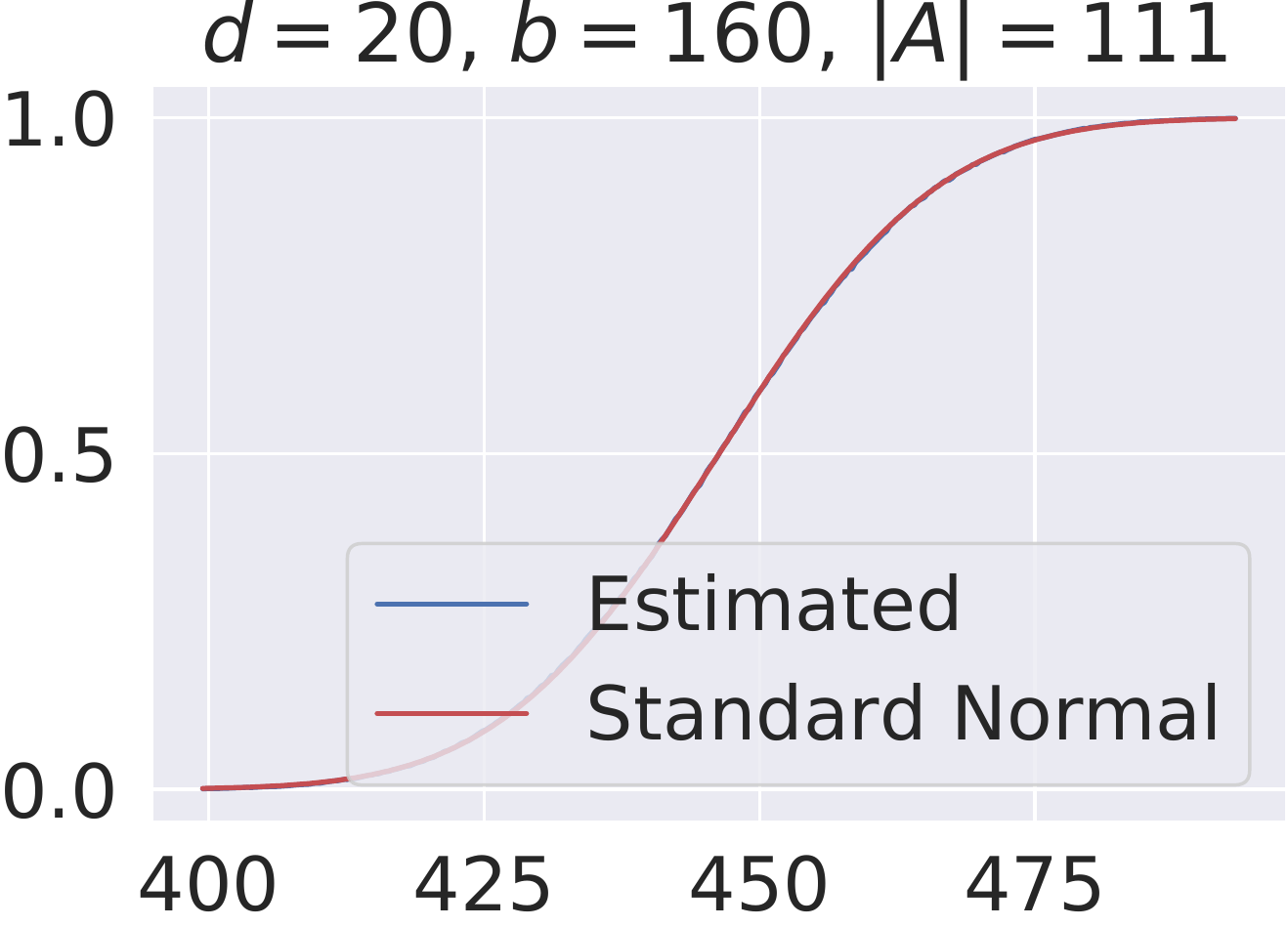}
    \caption{\label{fig:precision-logistic}Illustration of Proposition~\ref{prop:precision-logistic}. The multiplicities are arbitrary numbers between $1$ and $10$. The $\lambda_j\idf_j$ where drawn according to a Gaussian. Monte-Carlo simulation of the probability in blue $10^5$ simulations). In red, the cumulative distribution function of the $\gaussian{0}{1}$.  Note that Proposition~\ref{prop:precision-logistic} assumes $\length{A} \leq b/2$: the approximation may be inaccurate when this assumption is not satisfied (right column). }
\end{figure}
%%

%%%%%%%%%%%%%%%%%%%%%%%%%%%%%%%%%%%%%%%%%%%%%%%%%%%%%%%%%%%%%%%%%%

\subsection{Additional experiments for Section~\ref{sec:analysis}: Analysis on explainable classifiers}
\label{sec:additional_analysis}

In this Section we report additional experiments for our Analysis on explainable classifiers. 
First, we validate our results on simple if-then-rules: Figure~\ref{fig:small-decision-tree} and Figure~\ref{fig:suppl_product} illustrate Proposition~\ref{prop:small-decision-tree} and Proposition~\ref{prop:product}, respectively. 

Second, we validate Proposition~\ref{prop:approx-prec-maximization} as in Figure~\ref{fig:precision-logistic}, \emph{i.e.}, after training a logistic model (Figure~\ref{fig:suppl_precision_logistic}) and a perceptron model (Figure~\ref{fig:suppl_precision_perceptron}), we apply a shift $S$ to the intercept $\lambda_0$, as follows
\begin{equation} 
\label{eq:suppl_def-linear-classifier}
f(\doc) = \indic{\lambda^\top \tfidf{\doc} + (\lambda_0-S) > 0}
\, .
\end{equation}
As $S$ increases, the prediction becomes harder, and longer anchors are needed to reach the precision threshold. 
When a new word is included, we show that, as predicted by Proposition~\ref{prop:approx-prec-maximization}, the first word with higher $\lambda_j\idf_j$ is picked. 

%%%%%%%%%%%%%%%%%%%%%%%%%%%%%%%%%%%%%%%%%%%%%%%%%%%%%%%%
\paragraph{Error bars}
\label{sec-supp:error-bars}
In our experiments there are two sources of variability, coming from different runs and documents, as we ran $10$ times Anchors on each positively classified document. 
Figure~\ref{fig:hist-std} shows the standard deviation for $10$ runs on Restaurant reviews (model is a $10$-layers neural network): for half the documents, it is actually zero. 
\begin{figure}[h]
\centering
\includegraphics[scale=0.4]{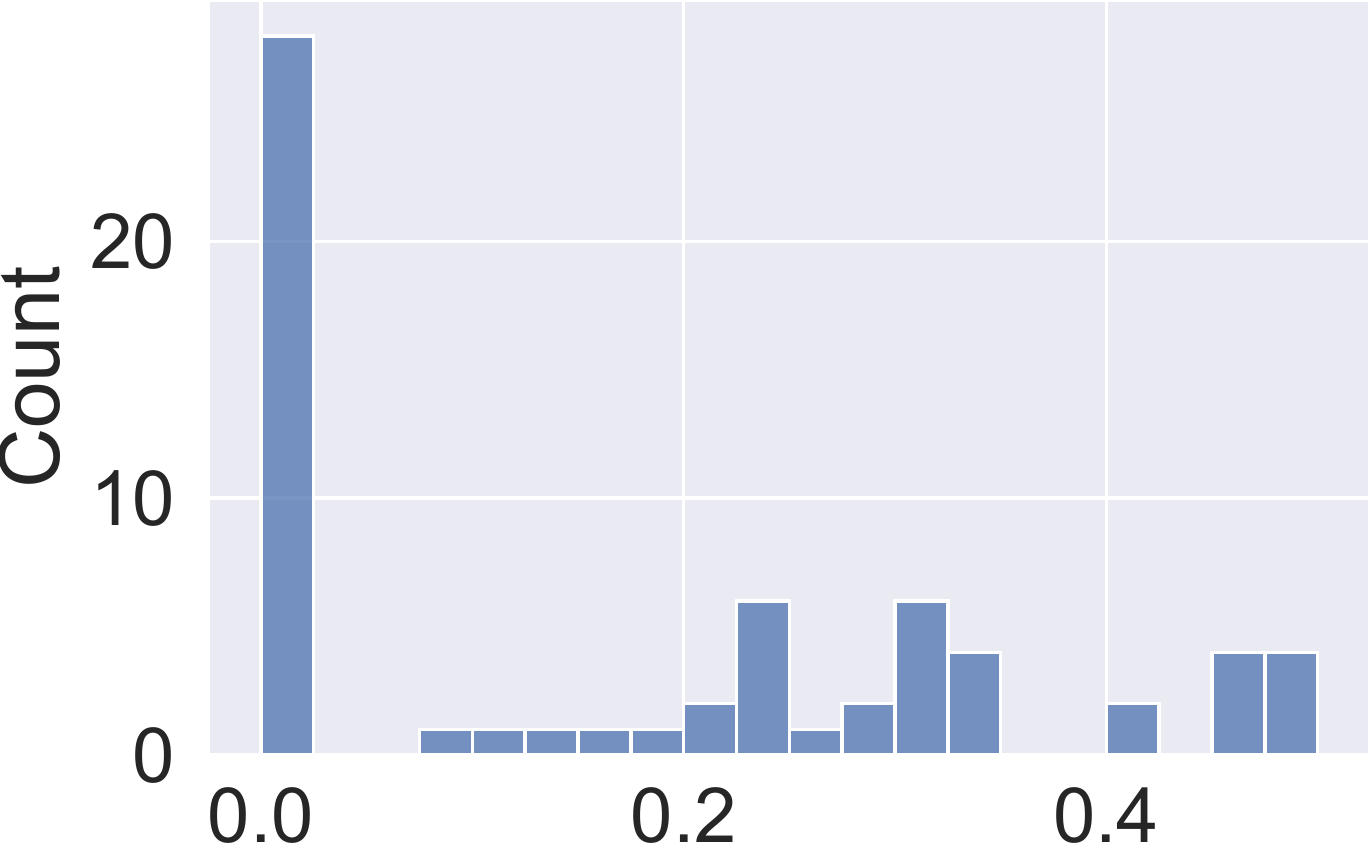}
    \caption{\label{fig:hist-std}Jaccard index standard deviations for $10$ runs on Restaurant reviews (model is a $10$-layers neural network), same experiment as Table 2 in the paper. Overall standard deviation is $0.39$, while the average one is $0.17$.}
\end{figure}
%%%%%%%%%%%%%%%%%%%%%%%%%%%%%%%%%%%%%%%%%%%%%%%%%%%%%%%%

%
\begin{figure}[h]
    \centering
    \begin{minipage}{0.4\textwidth}\centering
        \includegraphics[scale=0.3, right]{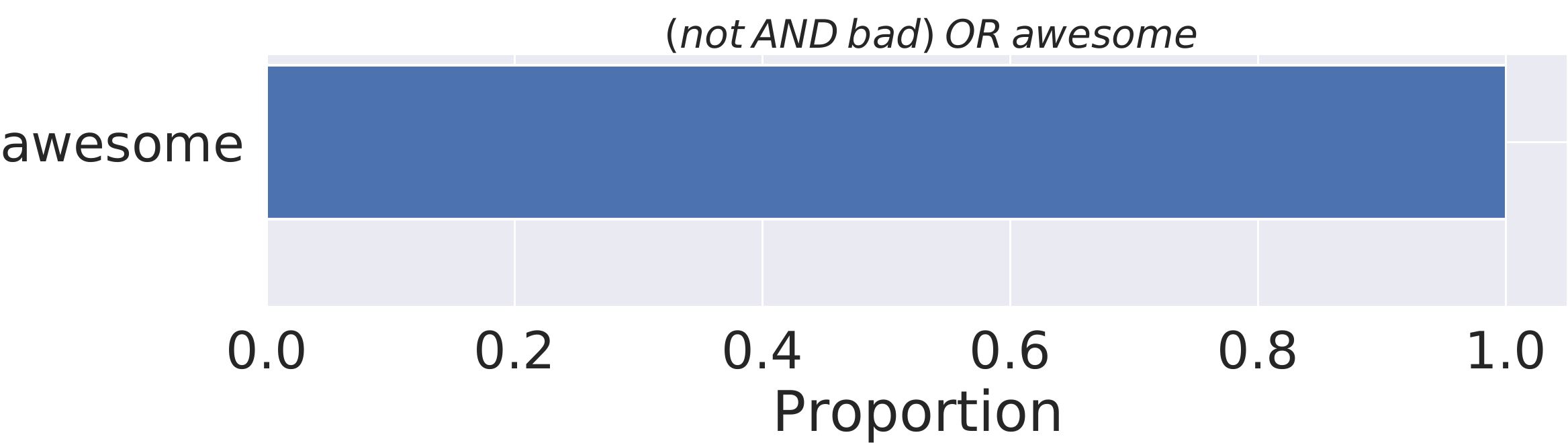} \\
        \vspace{0.3cm}
        \includegraphics[scale=0.3, right]{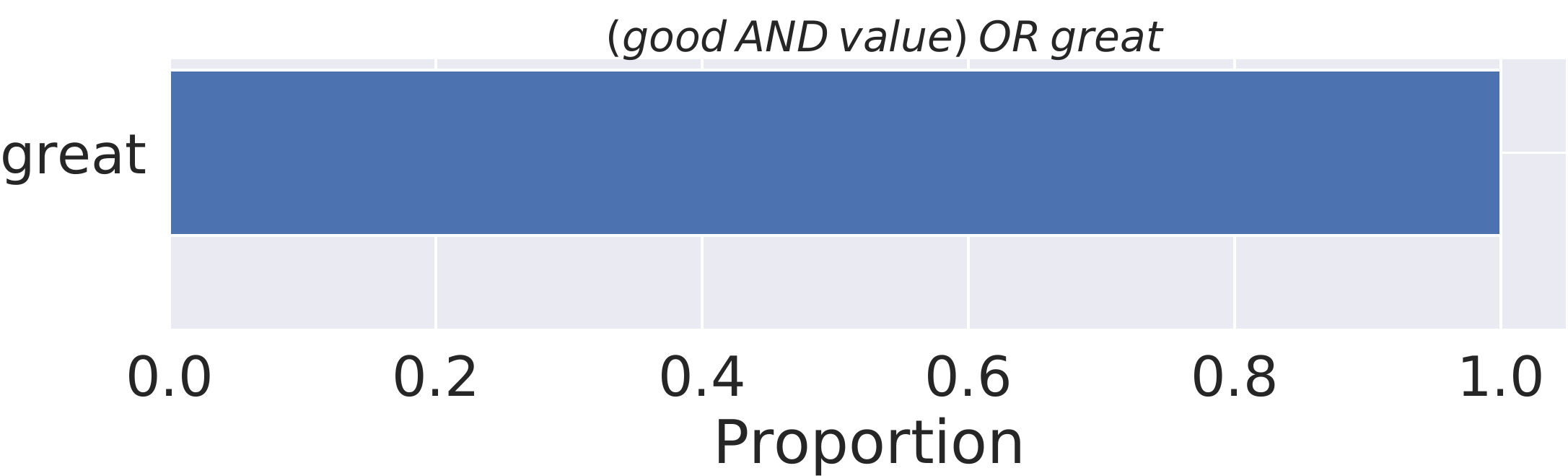} 
    \end{minipage}
    %\hfill
    \begin{minipage}{0.4\textwidth}\centering
        \includegraphics[scale=0.3, right]{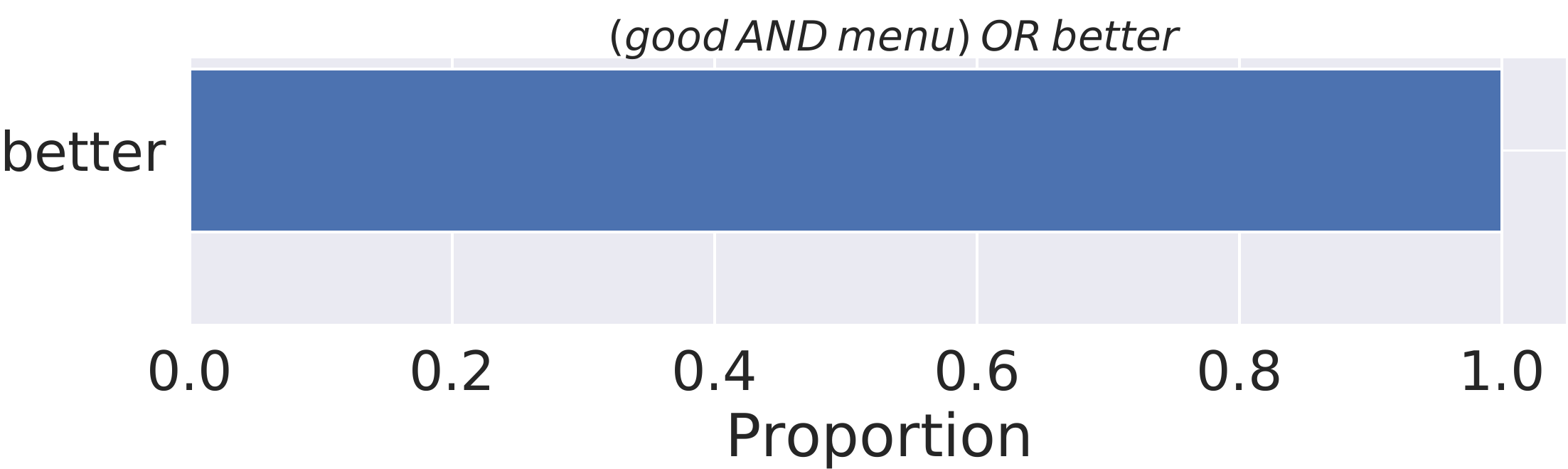} \\
        \vspace{0.3cm}
        \includegraphics[scale=0.3, right]{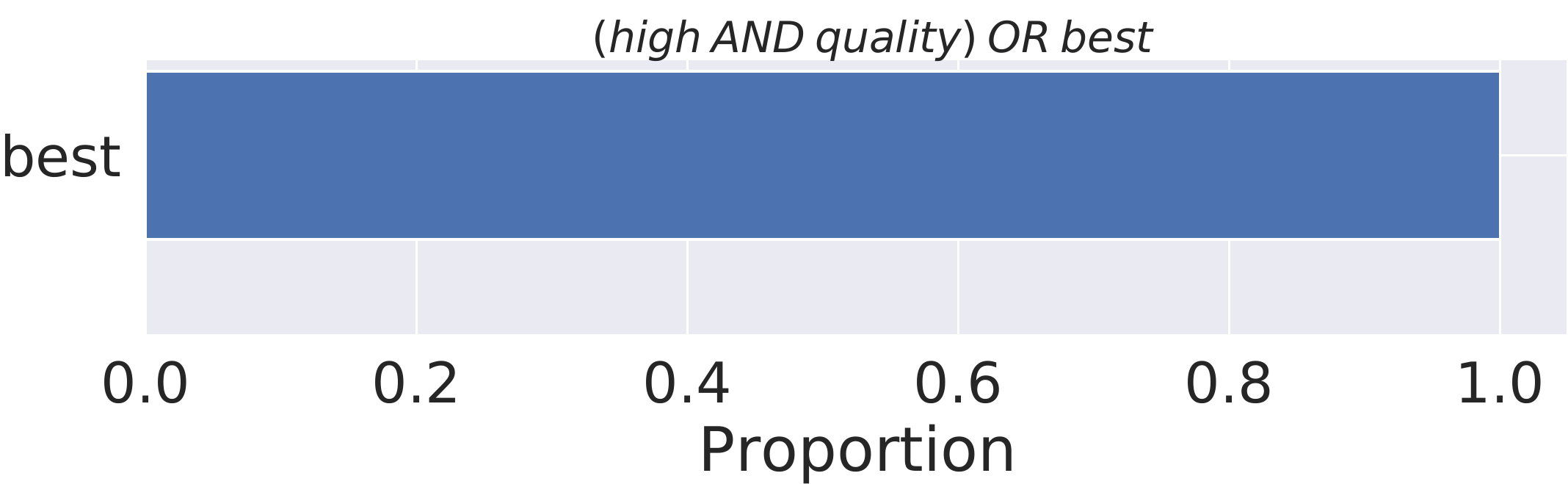}
    \end{minipage}
    \caption{\label{fig:small-decision-tree}Illustration of Proposition~\ref{prop:small-decision-tree}. Count on $100$ runs of Anchors on four different reviews. Classification rules consist of two conditions, shown above each figure.  Note that, for each case, \textbf{both conditions are satisfied} by the example. The shorter is always selected, as predicted by Proposition~\ref{prop:small-decision-tree}.}
\end{figure}

\begin{figure}[h]
\centering
    \begin{minipage}{0.5\textwidth}\centering
        \includegraphics[scale=0.3, right]{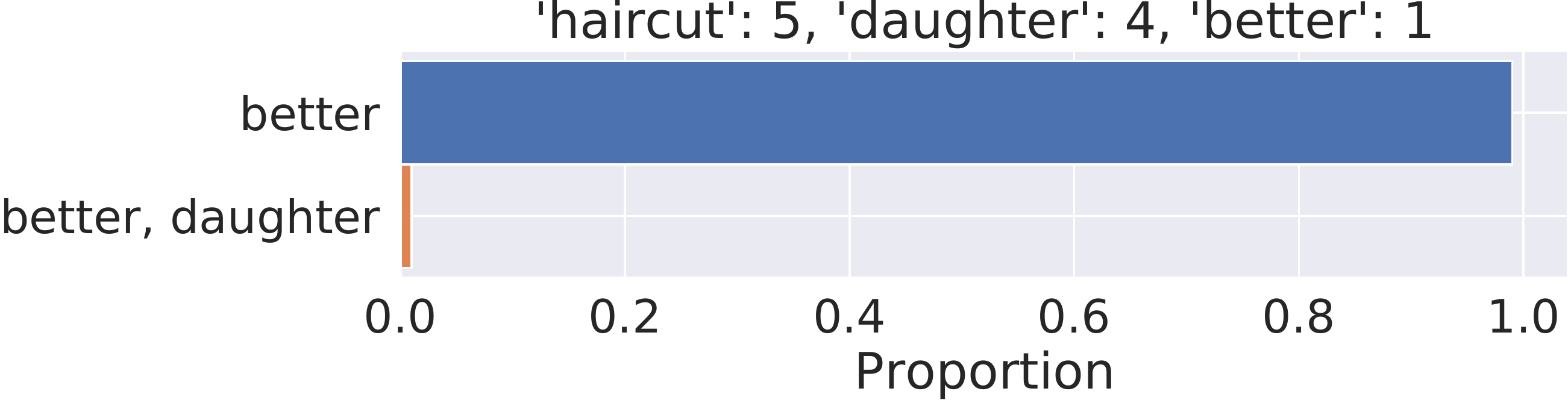}\\
        \vspace{0.3cm}
        \includegraphics[scale=0.3, right]{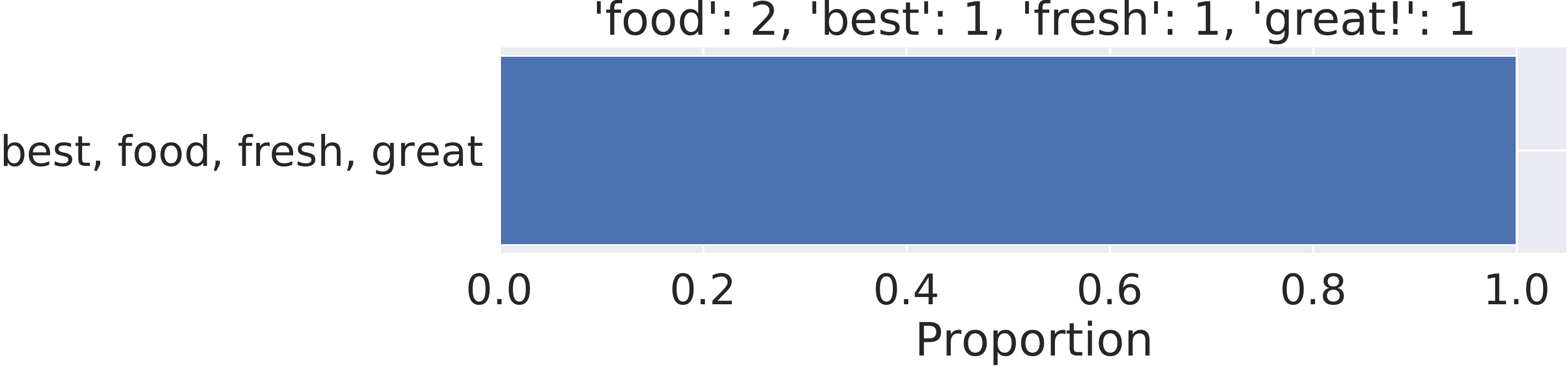}
    \end{minipage}
    \hfill
    \begin{minipage}{0.4\textwidth}\centering
        \includegraphics[scale=0.3, right]{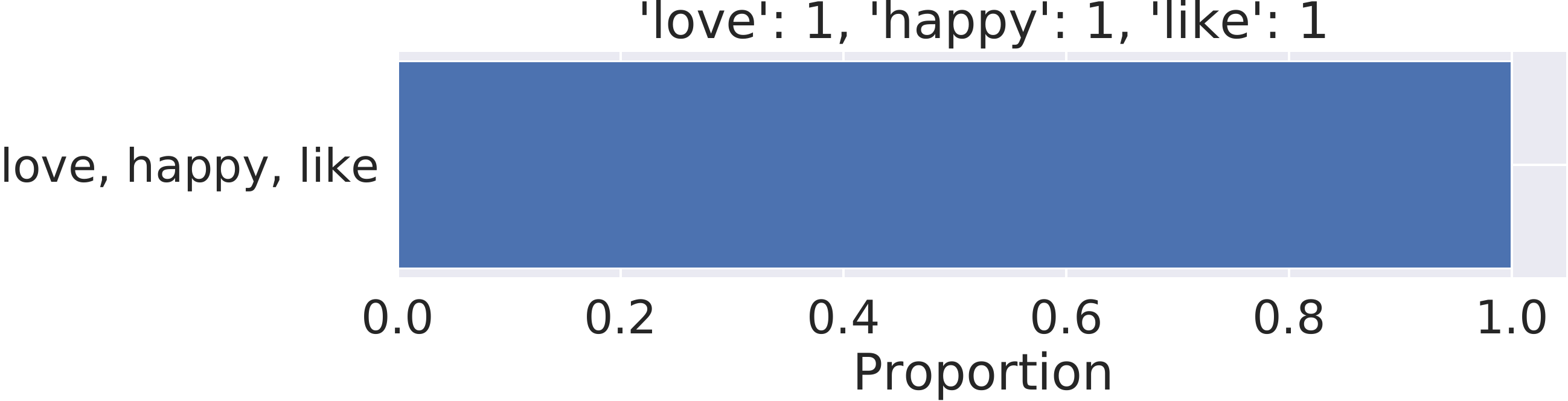}\\
        \vspace{0.3cm}
        \includegraphics[scale=0.3, right]{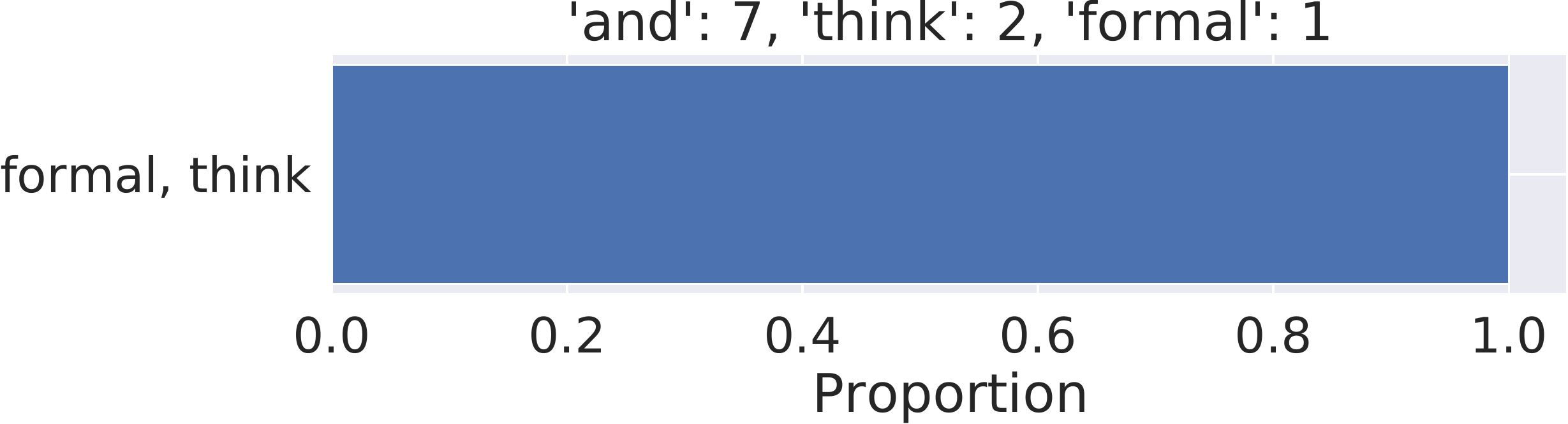}
    \end{minipage}
    \caption{\label{fig:suppl_product}Illustration of Proposition~\ref{prop:product}. Count on $100$ runs of Anchors on four different reviews. The classifier looks for the joint presence of the words reported above each figure (for instance, the model in the upper left panel predicts $1$ if $\doc$ contains ``haircut,'' ``daughter,'' and ``better''). The multiplicities of each word in the document is reported. As predicted by Proposition~\ref{prop:product}, words with multiplicity higher than a given threshold disappear from the explanation. }
\end{figure}

\begin{figure}[h]
    \centering
    \includegraphics[scale=0.4]{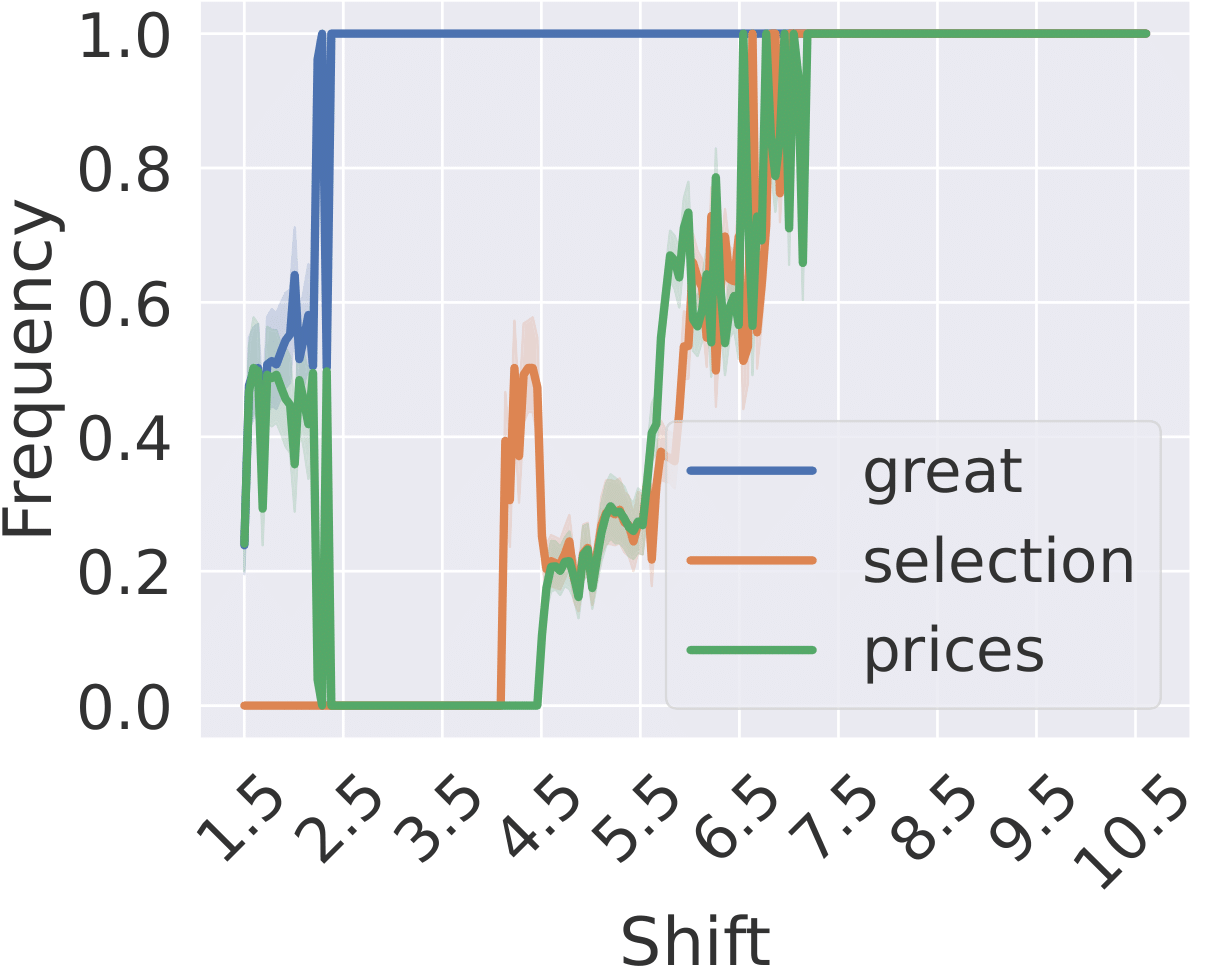}
    \hspace{1cm}
    \includegraphics[scale=0.4]{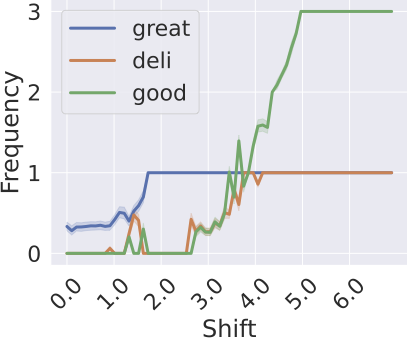} \\
    \vspace{0.5cm}
    \includegraphics[scale=0.4]{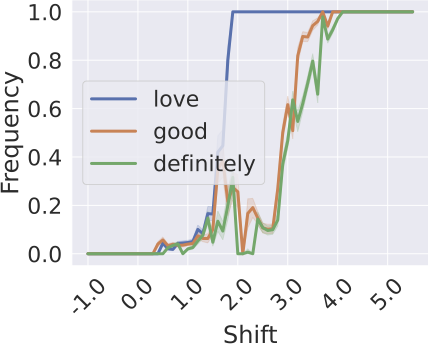}
    \hspace{1cm}
    \includegraphics[scale=0.4]{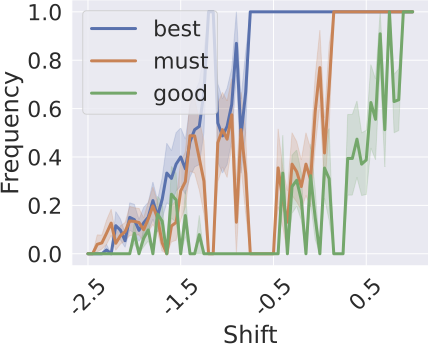}
    \caption{\label{fig:suppl_precision_logistic}Illustration of Proposition~\ref{prop:approx-prec-maximization}. Frequency on $100$ runs of Anchors on four different documents with four different logistic models when a shift $S$ is applied. Legend shows the first three words ordered by $\lambda_j\idf_j$. 
    For example, in the lower-right figure, the shift $S$ increases from $-2.5$ to $1$. As predicted by Proposition~\ref{prop:approx-prec-maximization}, first the words with higher $\lambda_j\idf_j$ are selected. Note that the overlap of some curves is due to similar coefficients for the corresponding words.}
\end{figure}

\begin{figure}[h]
    \centering
    \includegraphics[scale=0.4]{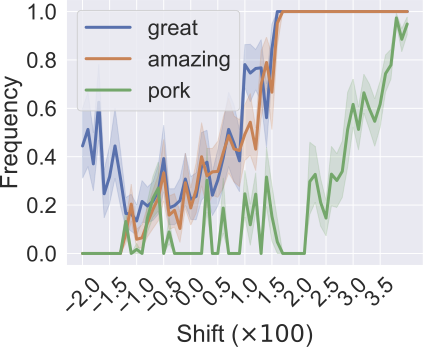}
        \hspace{1cm}
    \includegraphics[scale=0.4]{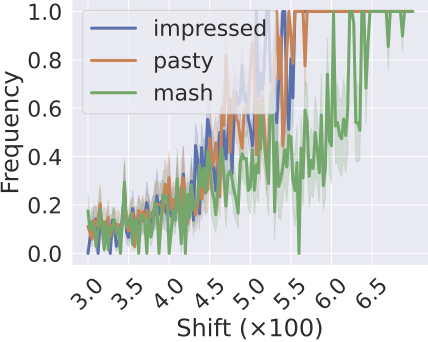}
    \vspace{0.5cm}
    \includegraphics[scale=0.4]{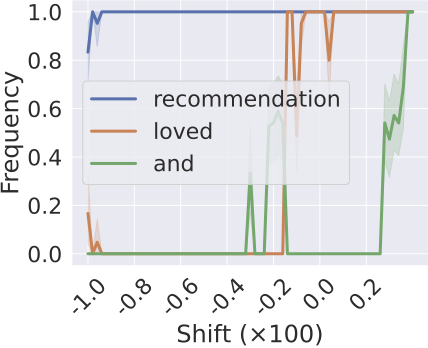}
    \hspace{1cm}
    \includegraphics[scale=0.4]{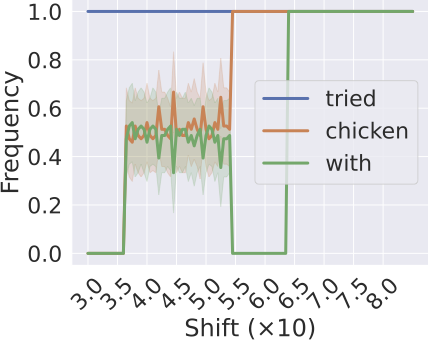}
    \caption{\label{fig:suppl_precision_perceptron}Illustration of Proposition \ref{prop:approx-prec-maximization}. Frequency on $100$ runs of Anchors on four different documents with four different perceptron models when a shift $S$ is applied. Legend shows the first three words ordered by $\lambda_j\idf_j$. Anchors includes new words in order of $\lambda_j\idf_j$, as predicted by Proposition~\ref{prop:approx-prec-maximization}. The overlap of some curves is due to similar coefficients for the corresponding words. For example, in the top-left figure, ``great'' has a coefficient $\lambda_j\idf_j$ equal to $207$, close to the coefficient for ``amazing'', $192$. }
\end{figure}

To further demonstrate this phenomenon, we also conducted the following experiment. 
We first trained a logistic model on three review datasets, achieving accuracies between $80\%$ and $90\%$. 
We then ran Anchors with the default setting $10$ times on positively classified documents. 
For each document, we measure the Jaccard similarity between the anchor $A$ and the first $\length{A}$ words ranked by $\lambda_j\idf_j$.  
In Table \ref{tab:logistic-similarity} we report the average Jaccard index: results validate Proposition~\ref{prop:approx-prec-maximization}. 

%%%%%%%%%%%%%%%%%%%%%%%%%%%%%%%%%%%%%%%%%%%%%%%%%%%%%%%%%%%%%%%%%%%%%%%%%%%

\subsection{Empirical validation of Proposition~\ref{prop:normalized-tf-idf-berry-esseen}: Normalized-TF-IDF, Berry-Esseen} 
\label{sec:check-prop-normalized-tf-idf-berry-esseen}
Figure~\ref{fig:norm-berry-esseen} shows an empirical validation for Proposition~\ref{prop:normalized-tf-idf-berry-esseen} for different document size and for anchors of different sizes. 
\begin{figure}[h]
    \centering
    \includegraphics[scale=0.40]{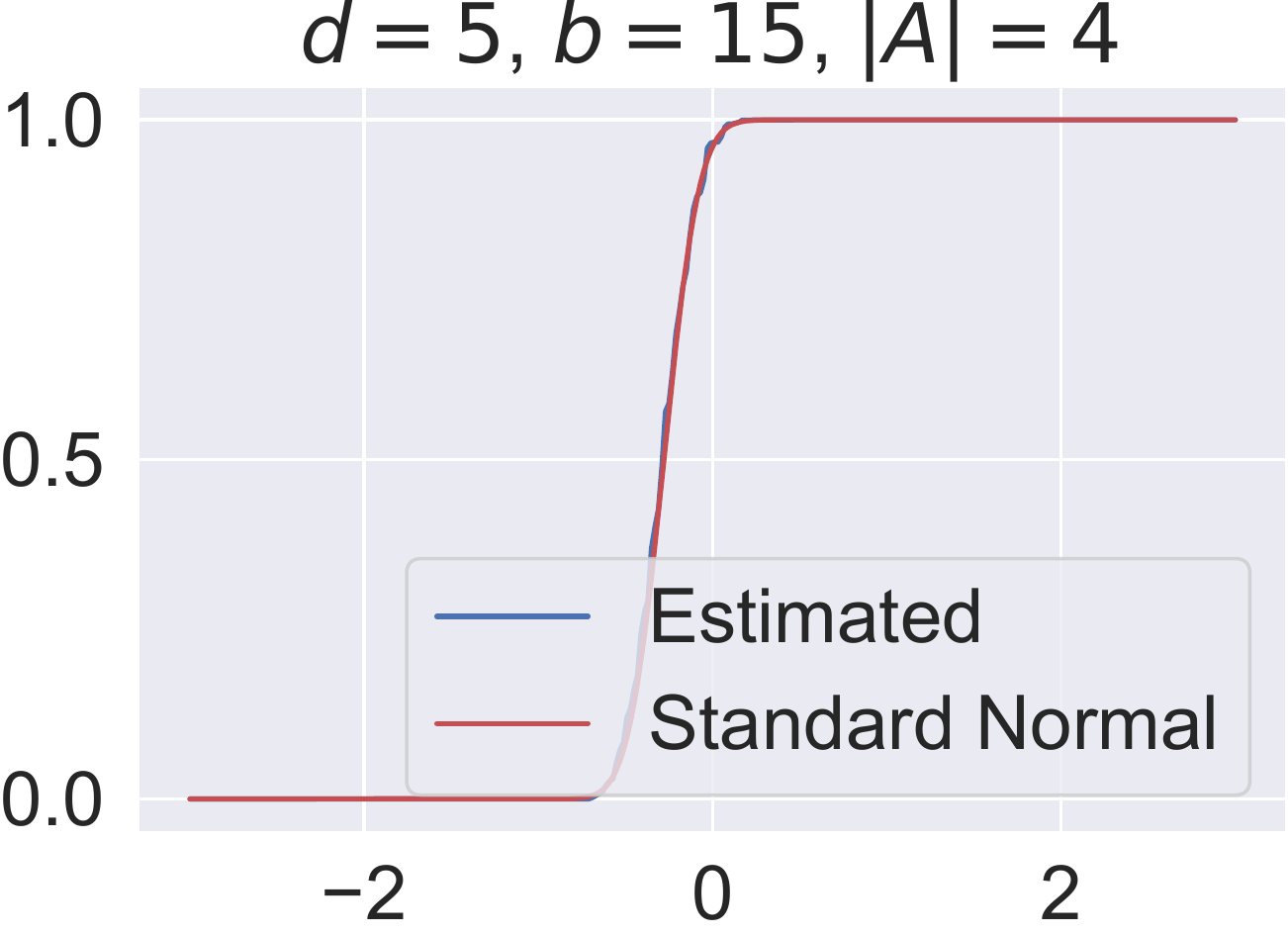}
    \hspace{1cm}
    \includegraphics[scale=0.40]{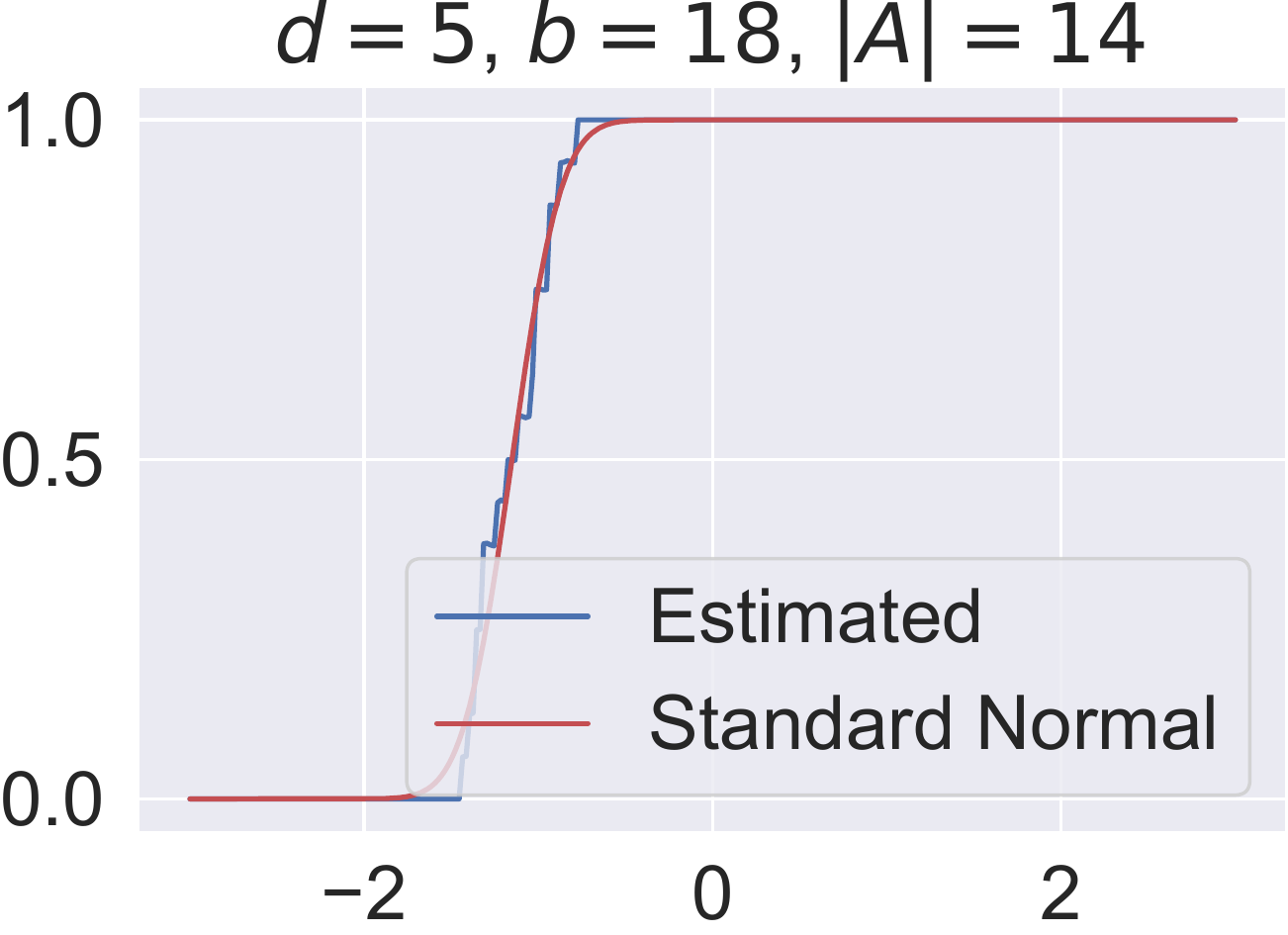} \\
    \vspace{0.5cm}
    \includegraphics[scale=0.40]{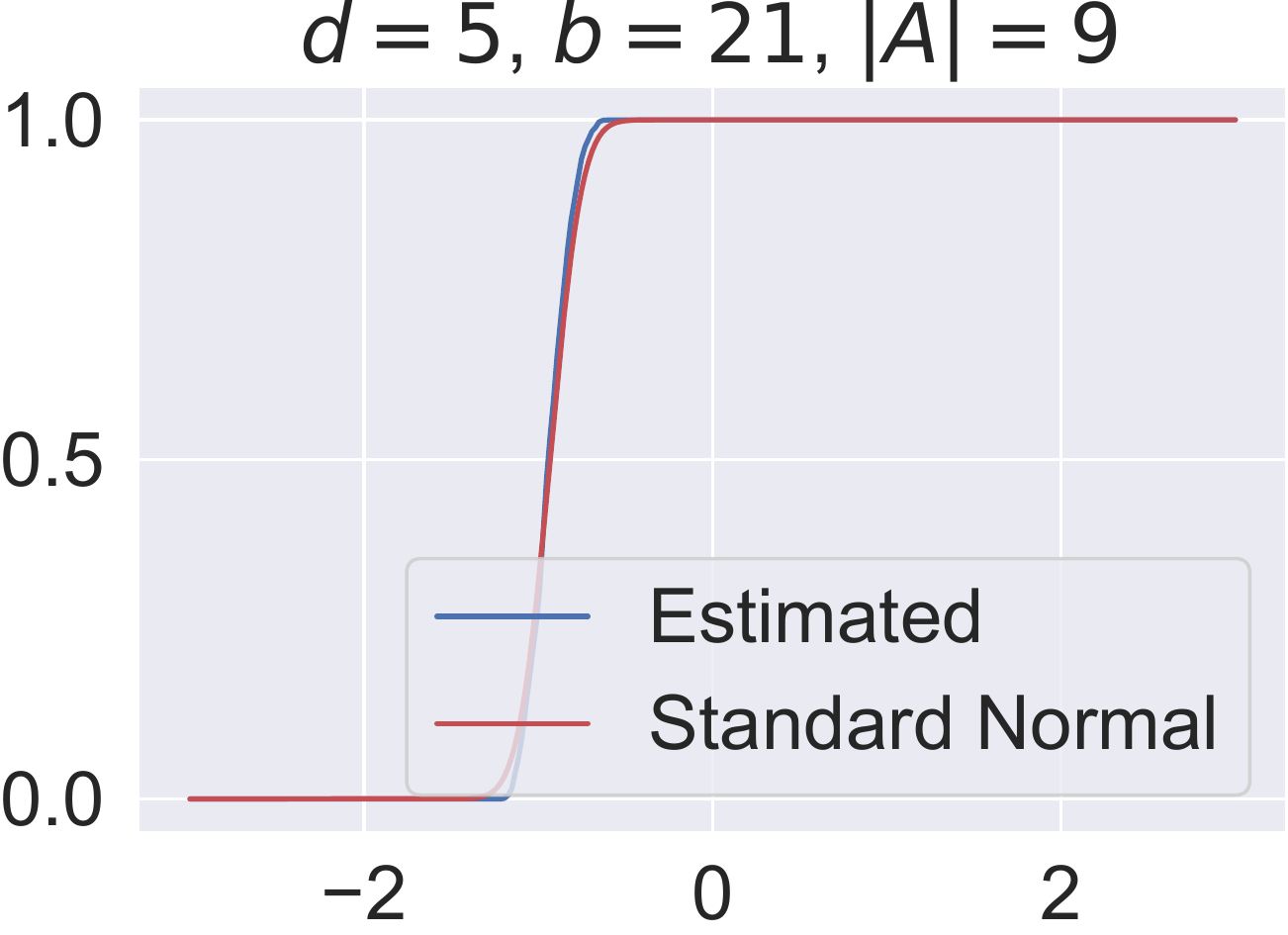}
    \hspace{1cm}
    \includegraphics[scale=0.40]{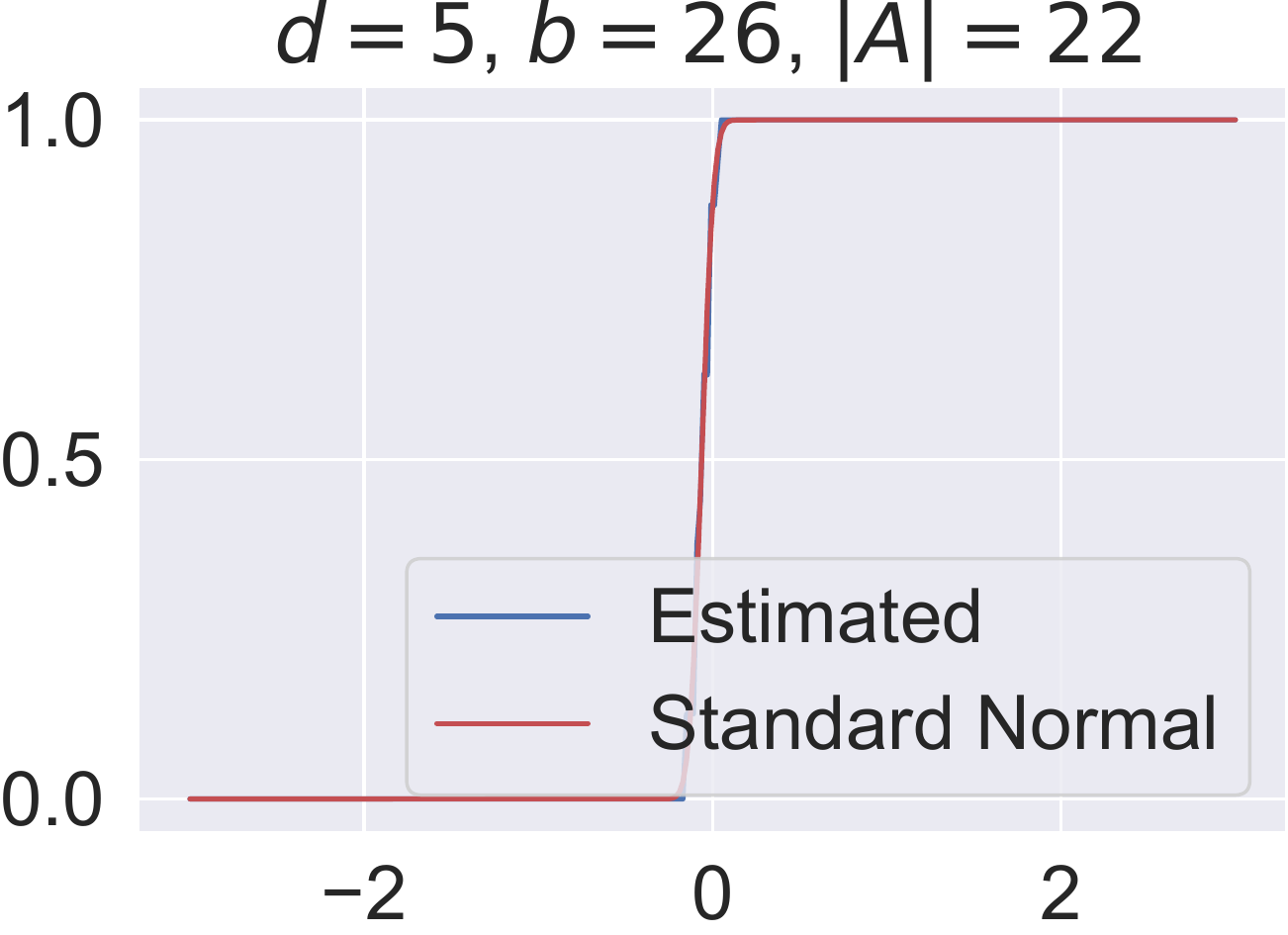} \\
    \vspace{0.5cm}
    \includegraphics[scale=0.40]{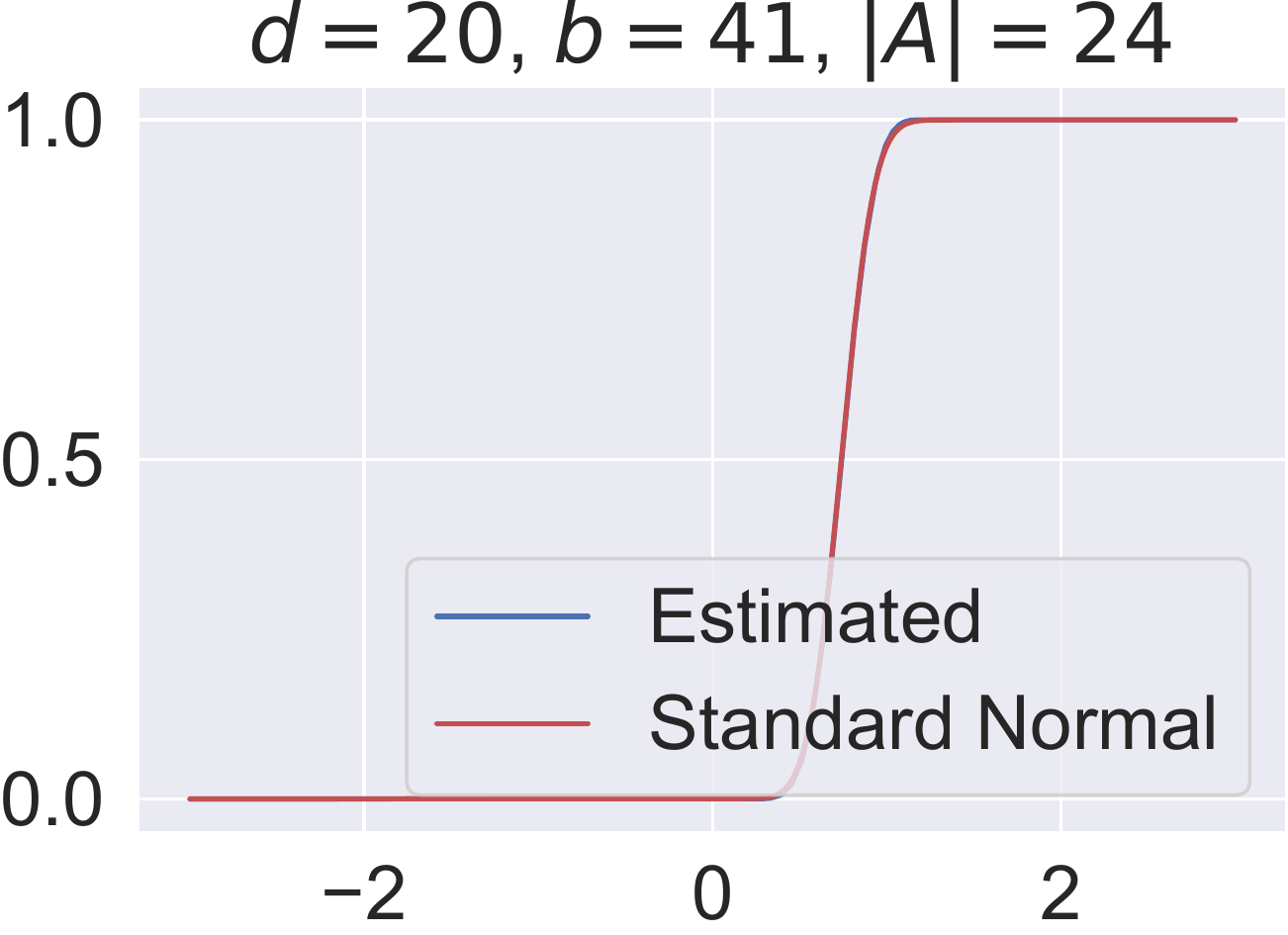}
    \hspace{1cm}
    \includegraphics[scale=0.40]{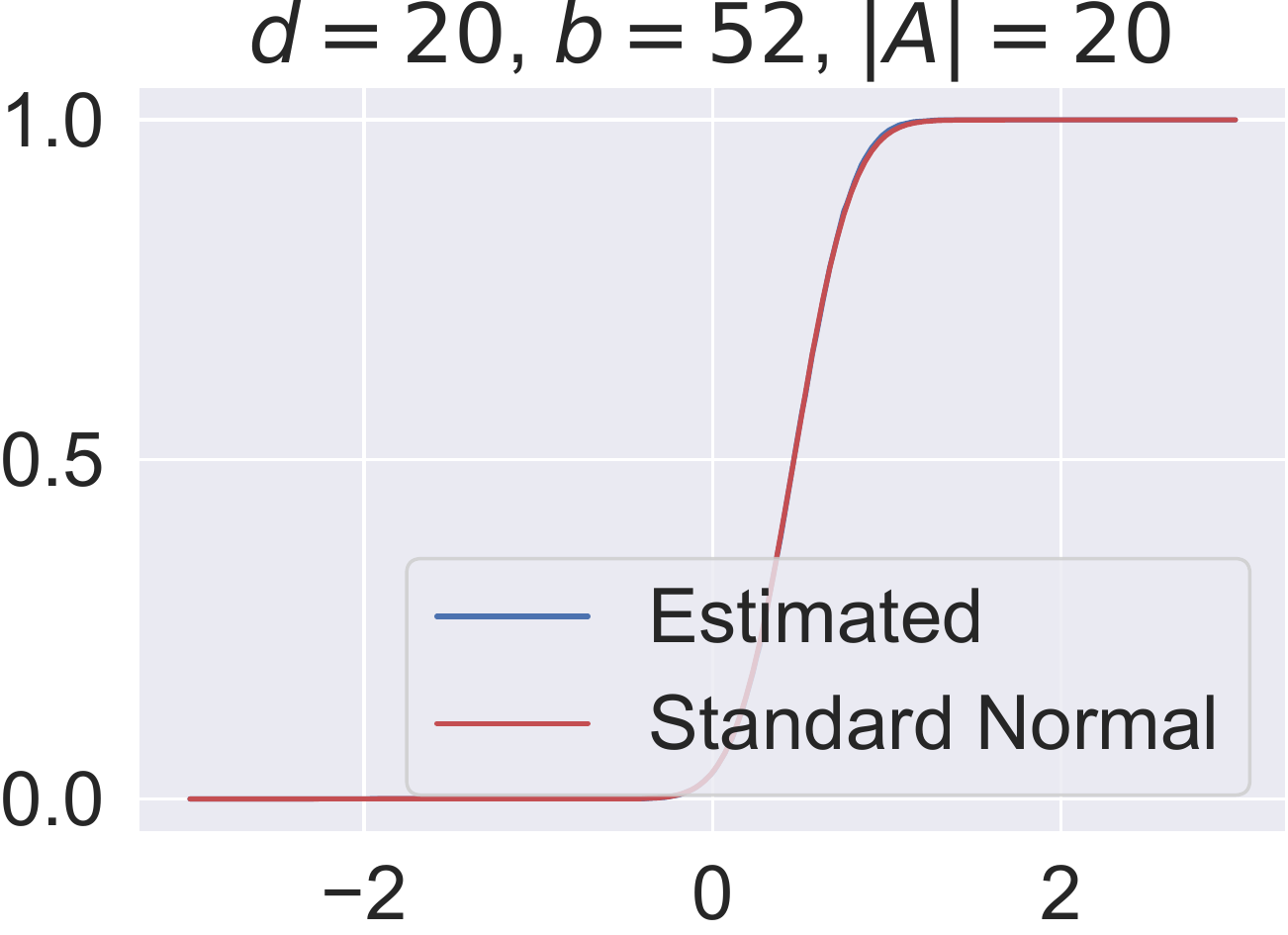} \\
    \vspace{0.5cm}
    \includegraphics[scale=0.40]{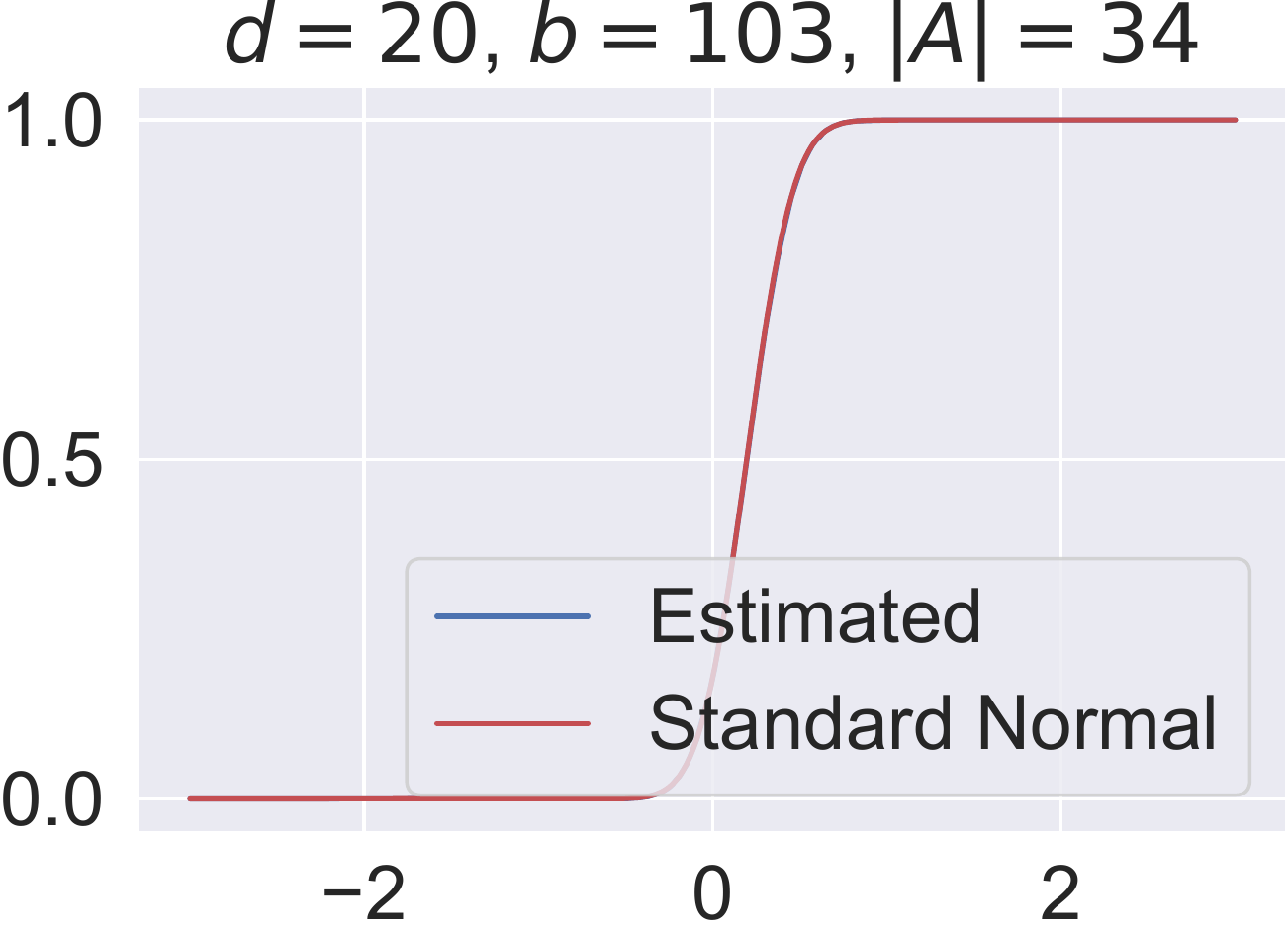}
    \hspace{1cm}
    \includegraphics[scale=0.40]{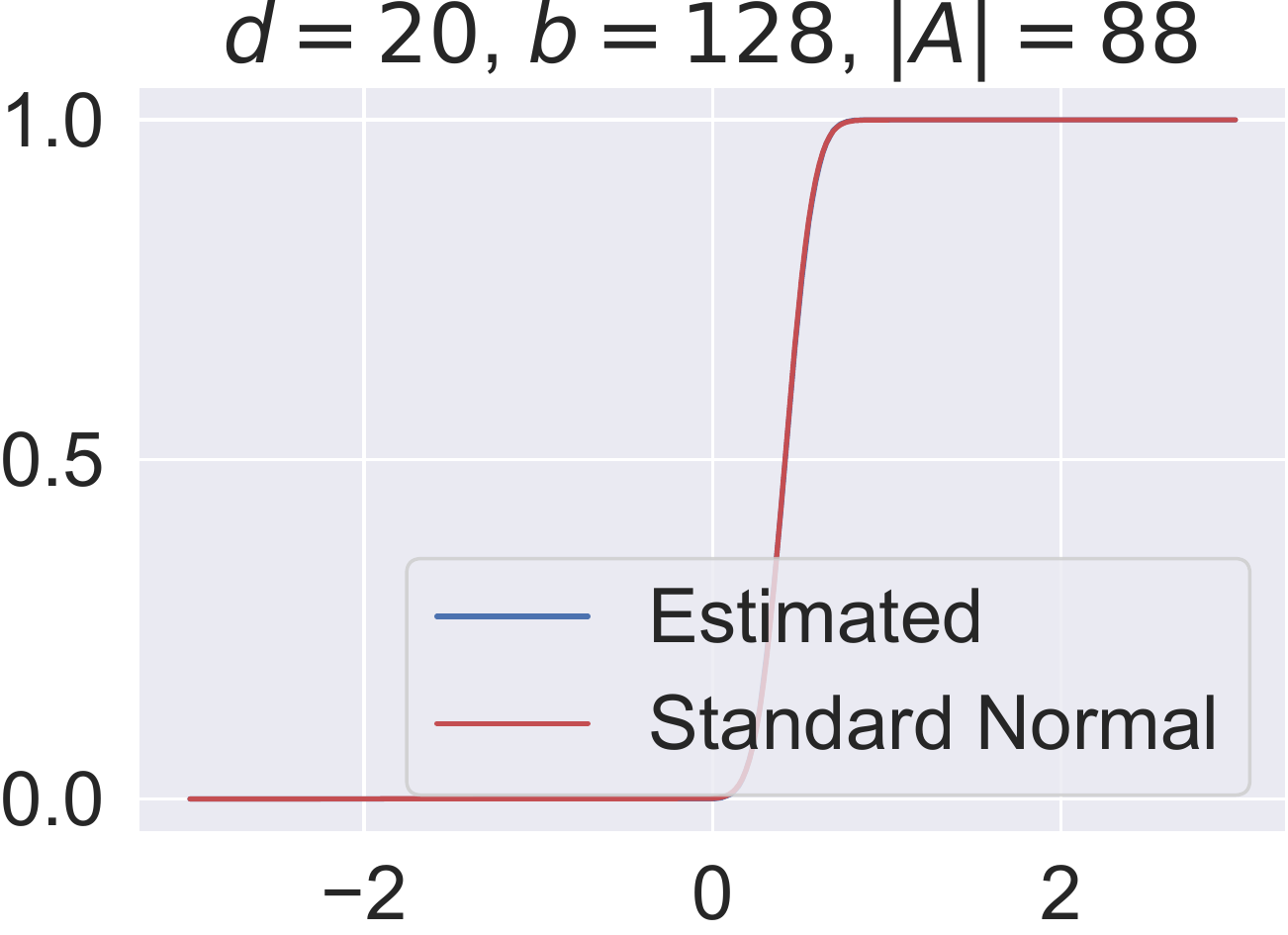}
    \caption{\label{fig:norm-berry-esseen}Illustration of Proposition~\ref{prop:normalized-tf-idf-berry-esseen}. The multiplicities are arbitrary numbers between $1$ and $10$. The $\lambda_j\idf_j$ were drawn according to a Gaussian. Monte-Carlo simulation of the probability in blue $10^5$ simulations). In red, the cumulative distribution function of the $\gaussian{0}{1}$.  Note that Proposition~\ref{prop:normalized-tf-idf-berry-esseen} assumes $\length{A} \leq b/2$: the approximation may be inaccurate when this assumption is not satisfied (right column). }
\end{figure}
%%

%%%%%%%%%%%%%%%%%%%%%%%%%%%%%%%%%%%%%%%%%%%%%%%%%%%%%%%%%%%%%%%%%%%%%%%%%%%%%
\subsection{Additional experiments for Section~\ref{sec:neural-nets}: Anchors on Neural Networks}
\label{sec:additional_nn}
We show in Table \ref{tab:nn-similarity} additional experiments that validate our conjecture expressed in Section~\ref{sec:neural-nets}. 
To this end, we trained, for each dataset (Restaurants, Yelp, and IMDB), three feed-forward neural networks, with $3$, $10$, and $20$ layers, achieving accuracies around $90\%$. 
The code used for model training is available at \url{https://github.com/gianluigilopardo/anchors_text_theory}. 
We then ran Anchors with default settings $10$ times on positively classified documents. 
For each document, we get the gradient of the model with respect to the input: for all $j\in [d]$, $\lambda_j \defeq \frac{\partial g(\tfidf{x})}{\partial \tfidf{x}_j}$. 
We then measure the average Jaccard similarity between the anchor $A$ and the first $\length{A}$ word ranked by $\lambda_j\idf_j$. 
% Table~\ref{tab:supp-nn-similarity} shows the results. 
%
% Note that the experiments of Table~\ref{tab:supp-nn-similarity} were performed on a larger number of samples from the datasets, with respect to those reported in Table~\ref{tab:nn-similarity}. 
%
% \input{supplementary/analysis/nn/nn_tfidf}
%
% \vfill

%%%%%%%%%%%%%%%%%%%%%%%%%%%%%%%%%%%%%%%%%%%%%%%%%%%%%%%%%%%%%%%%%%%%%%%%%%%%%
% Rebuttal

\subsection{BERT replacement}
\label{sec-sup:bert-replacement}
As discussed in Section \ref{sec:vectorizers}, we study the \texttt{UNK}-replacement option even if when replacing words with a fixed token can produce unrealistic samples and lead to out-of-distribution issue. 
Nevertheless, we performed the same experiments of Section \ref{sec:neural-nets} using the BERT-replacement option when a $3$-layers neural network is applied on a sample of $50$ Restaurants reviews. 
Somewhat surprisingly, our message still stands: we reach a Jaccard Similarity of $0.83 \pm 0.2$, similarly to the \texttt{UNK} setting.  %(see Table~\ref{tab:bert}). 
What is more, we notice that such option is $10$ times slower and produces longer anchors. 

\vfill

\end{document}